\documentclass{article}

\PassOptionsToPackage{numbers, compress}{natbib}

\usepackage[final]{neurips_2021}

\usepackage[utf8]{inputenc} 
\usepackage[T1]{fontenc}    
\usepackage{hyperref}       
\usepackage{url}            
\usepackage{booktabs}       
\usepackage{amsfonts}       
\usepackage{nicefrac}       
\usepackage{microtype}      
\usepackage{xcolor}         

\usepackage{amsmath}
\usepackage{amsthm}
\usepackage{wrapfig}
\usepackage{algorithm, algpseudocode}
\usepackage{lipsum}
\usepackage{graphicx}
\usepackage{enumerate}
\usepackage{subcaption}
\newtheorem{theorem}{Theorem}
\newtheorem{theorem_main}{Theorem}
\newtheorem{proposition}{Proposition}
\newtheorem{proposition_main}{Proposition}
\newtheorem{lemma}{Lemma}
\newtheorem{lemma_main}{Lemma}

\newtheorem{definition}{Definition}
\newtheorem{corollary}{Corollary}

\newtheorem{remark}{Remark}
\newtheorem{example}{Example}

\let\oldReturn\Return
\renewcommand{\Return}{\State\oldReturn}

\newcommand{\E}{\mathbb{E}}
\newcommand{\V}{\mathrm{Var}}
\newcommand{\U}{\mathcal{U}}
\newcommand{\Sp}{\mathbb{S}}

\newcommand{\I}{\mathbf{I}}

\title{On the Convergence of Prior-Guided Zeroth-Order Optimization Algorithms}

\author{%
  Shuyu Cheng, \quad Guoqiang Wu, \quad Jun Zhu\thanks{J.Z. is the Corresponding Author. G.W. is now with School of Software, Shandong University.} \\
  Dept. of Comp. Sci. and Tech., BNRist Center, State Key Lab for Intell. Tech. \& Sys., Institute for AI,\\
  Tsinghua-Bosch Joint Center for ML, Tsinghua University, Beijing, 100084, China\\
  Pazhou Lab, Guangzhou, 510330, China \\
  \scriptsize \texttt{chengsy18@mails.tsinghua.edu.cn, guoqiangwu90@gmail.com, dcszj@tsinghua.edu.cn} \\
}

\begin{document}

\maketitle

\begin{abstract}
Zeroth-order (ZO) optimization is widely used to handle challenging tasks, such as query-based black-box adversarial attacks and reinforcement learning.
Various attempts have been made to integrate prior information into the gradient estimation procedure based on finite differences, with promising empirical results. 
However, their convergence properties are not well understood. This paper makes an attempt to fill up this gap by analyzing the convergence of prior-guided ZO algorithms under a greedy descent framework with various gradient estimators. 
We provide a convergence guarantee for the prior-guided random gradient-free (PRGF) algorithms. 
Moreover, to further accelerate over greedy descent methods, we present a new accelerated random search (ARS) algorithm that incorporates prior information, together with a convergence analysis.
Finally, our theoretical results are confirmed by experiments on several numerical benchmarks as well as adversarial attacks. Our code is available at \url{https://github.com/csy530216/pg-zoo}.
\end{abstract}

\section{Introduction}

Zeroth-order (ZO) optimization~\cite{matyas1965random} provides powerful tools to deal with challenging tasks, such as query-based black-box adversarial attacks \cite{chen2017zoo, ilyas2018black}, reinforcement learning \cite{salimans2017evolution, mania2018simple, choromanski2018structured}, meta-learning \cite{andrychowicz2016learning}, and hyperparameter tuning \cite{snoek2012practical}, where the access to gradient information is either not available or too costly. ZO methods only assume an oracle access to \emph{the function value} at any given point, instead of gradients as in first-order methods.
The primary goal is to find a solution with as few queries to the function value oracle as possible. Recently, various ZO methods have been proposed in two main categories. 
One type is to obtain a gradient estimator and plug in some gradient-based methods. \cite{nesterov2017random} analyzes the convergence of such methods with a random gradient-free (RGF) estimator obtained via finite difference along a random direction. 
The other type is directed search, where the update only depends on comparison between function values at different points. Such methods are robust against monotone transform of the objective function \cite{stich2013optimization, golovin2019gradientless, bergou2020stochastic}. However, as they do not directly utilize function values, 
their query complexity is often higher than that of finite difference methods. 
Therefore, we focus on the first type of methods in this paper.
Other methods exist such as CMA-ES \cite{hansen1996adapting} which is potentially better on objective functions with a rugged landscape, but lacks a general convergence guarantee. 

ZO methods are usually less efficient than first-order algorithms, as they 
typically take $O(d)$ queries to reach a given precision, where $d$ is the input dimension (see Table~1 in \cite{golovin2019gradientless}). Specifically, for the methods in \cite{nesterov2017random}, the oracle query count is $O(d)$ times larger than that of their corresponding schemes using gradients. 
This inefficiency stems from random search along uniformly distributed directions. To improve, various attempts have been made to augment random search with (extra) prior information. For instance, \cite{ilyas2018prior, meier2019improving} use a time-dependent prior (i.e., the gradient estimated in the last iteration), 
while \cite{maheswaranathan2019guided, cheng2019improving, brunner2019guessing} use surrogate gradients obtained from other sources.\footnote{Some work \cite{chen2017zoo, ilyas2018prior, tu2019autozoom} restricts the random search to a more effective subspace reflecting the prior knowledge. But, this eliminates the possibility of convergence to the optimal solution.}
Among these methods, \cite{ilyas2018prior, maheswaranathan2019guided, cheng2019improving, meier2019improving} propose objective functions to describe the quality of a gradient estimator for justification or optimizing its hyperparameters;
for example, \cite{meier2019improving} uses a subspace estimator that maximizes its squared cosine similarity with the true gradient as the objective, and finds a better descent direction than both the prior direction and the randomly sampled direction. 
However, all these methods treat gradient estimation and the optimization algorithm separately, and it remains unclear \emph{whether are these gradient estimators and the corresponding prior-exploiting optimization methods theoretically sound?} and \emph{what role does the prior play in accelerating convergence?}

In this paper, we attempt to answer these questions by establishing a formal connection between convergence rates and the quality of gradient estimates. Further, we develop a more efficient ZO algorithm with prior inspired by a theoretical analysis. First, we present a greedy descent framework of ZO methods and provide its convergence rate under smooth convex optimization, which is positively related to the squared cosine similarity between the gradient estimate and true gradient. As shown by \cite{meier2019improving} and our results, given some finite-difference queries, the optimal estimator maximizing the squared cosine similarity is the projection of true gradient on the subspace spanned by those queried directions. In the case with prior information, a natural such estimator is in the same form as the one in \cite{meier2019improving}. We call it PRGF estimator and analyze its convergence rate.\footnote{Note that the estimator is different from the P-RGF estimator in \cite{cheng2019improving}.}
Our results show that no matter what the prior is, the convergence rate of PRGF is at least the same as that of the RGF baseline~\cite{nesterov2017random, kozak2021stochastic}, and it could be significantly improved if given a useful prior.\footnote{In this paper, RGF and PRGF could refer to either a gradient estimator or the greedy descent algorithm with the corresponding estimator, depending on the context.} Such results shed light on 
exploring prior information in ZO optimization 
to accelerate convergence.

Then, as a concrete example, we apply the analysis to 
History-PRGF~\cite{meier2019improving}, which uses historical information (i.e., gradient estimate in the last iteration) as the prior.
\cite{meier2019improving} presents an analysis on linear functions, yet still lacking a convergence rate. 
We analyze on general $L$-smooth functions, and find that when the learning rate is smaller than the optimal value $\nicefrac{1}{L}$ in smooth optimization, the expected squared cosine similarity could converge to a larger value, which compensates for the slowdown of convergence brought by the inappropriate choice of learning rate.
We also show that History-PRGF admits a convergence rate independent of learning rate as long as it is in a fairly wide range. 

Finally, to further accelerate greedy descent methods, we present Prior-Guided ARS (PARS), a new variant of Accelerated Random Search (ARS)~\cite{nesterov2017random} to utilize prior information. 
Technically, PARS is a non-trivial extension of ARS, as 
directly replacing the gradient estimator in ARS by the PRGF estimator would lose the convergence analysis. Thus, we present necessary extensions to ARS, and show that PARS has a convergence rate no worse than that of ARS and admits potential acceleration given a good prior. 
In particular, when the prior is chosen as the historical information, the resulting History-PARS is robust to learning rate in experiments. 
To our knowledge, History-PARS is the first ARS-based method that is empirically robust while retaining the convergence rate as ARS.
Our experiments on numerical benchmarks and adversarial attacks confirm the theoretical results.

\section{Setups}
\label{sec:2}
\paragraph{Assumptions on the problem class}
We consider unconstrained optimization, where the objective function $f: \mathbb{R}^d\to \mathbb{R}$ is convex and $L$-smooth for $L\geq 0$. Optionally, we require $f$ to be $\tau$-strongly convex for $\tau>0$. We leave definitions of these concepts to Appendix~\ref{sec:A1}.

\paragraph{Directional derivative oracle}
In ZO optimization, we follow the finite difference approach, which makes more use of the queried function values than direct search and 
provides better gradient approximations than alternatives \cite{berahas2021theoretical}. In particular, we consider the forward difference method:
\begin{align}
\label{eq:forward-difference}
    g_\mu(v;x):=\frac{f(x+\mu v)-f(x)}{\mu}\approx \nabla f(x)^\top v,
\end{align}
where $v$ is a vector with unit $\ell_2$ norm $\|v\|=1$ and $\mu$ is a small positive step. As long as the objective function is smooth, the error between the finite difference and the directional derivative could be uniformly bounded, as shown in the following proposition (see Appendix~\ref{sec:A2} for its proof).
\begin{proposition}
\label{prop:finite-error}
If $f$ is $L$-smooth, then for any $(x,v)$ with $\|v\|=1$, $|g_\mu(v;x)-\nabla f(x)^\top v|\leq \frac{1}{2}L\mu$.
\end{proposition}
\vspace{-.1cm}
Thus, in smooth optimization, the error brought by finite differences to the convergence bound can be analyzed in a principled way, and its impact tends to zero as $\mu\to 0$. We also choose $\mu$ as small as possible in practice. Hence, in the following analysis we directly assume the $\textbf{directional derivative oracle}$: suppose that we can obtain $\nabla f(x)^\top v$ for any $(x,v)$ in which $\|v\|=1$ with one query. 

\begin{algorithm}[t]
\caption{Greedy descent framework}
\label{alg:greedy-descent}
\begin{algorithmic}[1]
\Require $L$-smooth convex function $f$; initialization $x_0$; upper bound $\hat{L}$ ($\hat{L} \geq L$); iteration number $T$.
\Ensure $x_T$ as the approximate minimizer of $f$.
\For {$t = 0$ to $T-1$}
\State Let $v_t$ be a random vector s.t. $\|v_t\|=1$; 
\State $x_{t+1}\leftarrow x_t - \frac{1}{\hat{L}} g_t$, where $g_t\leftarrow \nabla f(x_t)^\top v_t \cdot v_t$; \label{lne:3-alg1} \label{lne:alg}
\EndFor
\Return $x_T$.
\end{algorithmic}
\end{algorithm}
\vspace{-.2cm}

\section{Greedy descent framework and PRGF algorithm}
\label{sec:greedy-descent}
\label{sec:3}
We now introduce a greedy descent framework in ZO optimization which can be implemented with various gradient estimators. We first provide a general analysis, followed by a concrete example.

\vspace{-.1cm}
\subsection{The greedy descent framework and general analysis}
In first-order smooth convex optimization, a sufficient single-step decrease of the objective can guarantee convergence (see Chapter~3.2 in \cite{bubeck2014convex}). Inspired by this fact, we design the update in an iteration to greedily seek for maximum decrease. Suppose we are currently at $x$, and want to update along the direction $v$. Without loss of generality, assume $\|v\|=1$ and $\nabla f(x)^\top v>0$. To choose a suitable step size $\eta$ that minimizes $f(x-\eta v)$, we note that
\begin{align}
\label{eq:upper-bound}
    f(x-\eta v)\leq f(x)-\eta\nabla f(x)^\top v+\frac{1}{2}L\eta^2:=F(\eta)
\end{align}
by smoothness of $f$. For the r.h.s, we have $F(\eta)= f(x)-\frac{(\nabla f(x)^\top v)^2}{2L}$ when $\eta=\frac{\nabla f(x)^\top v}{L}$, which minimizes $F(\eta)$. Thus, choosing such $\eta$ could lead to a largest guaranteed decrease of $f(x-\eta v)$ from $f(x)$. In practice, the value of $L$ is often unknown, but we can verify that as long as $0<\eta\leq\frac{\nabla f(x)^\top v}{L}$, then $f(x-\eta v)\leq F(\eta)< f(x)$, 
i.e., we can guarantee decrease of the objective (regardless of the direction of $v$ if $\nabla f(x)^\top v>0$). Based on the above discussion, we further allow $v$ to be random
and present the greedy descent framework in Algorithm~\ref{alg:greedy-descent}.

\begin{remark}
If $v_t\sim \mathcal{U}(\mathbb{S}^{d-1})$, i.e. $v_t$ is uniformly sampled from $\Sp^{d-1}$ (the unit sphere in $\mathbb{R}^d$), then Algorithm~\ref{alg:greedy-descent} is similar to the simple random search in \cite{nesterov2017random} except that $v_t$ is sampled from a Gaussian distribution there.
\end{remark}
\begin{remark}
\label{rem:dependence-history}
In general, $v_t$ could depend on the history (i.e., the randomness before sampling $v_t$). For example, $v_t$ can be biased towards a vector $p_t$ that corresponds to prior information, and $p_t$ depends on the history since it depends on $x_t$ or the optimization trajectory. 

Theoretically speaking, $v_t$ is sampled from the conditional probability distribution $\Pr(\cdot|\mathcal{F}_{t-1})$ where $\mathcal{F}_{t-1}$ is a sub $\sigma$-algebra modelling the historical information. $\mathcal{F}_{t-1}$ is important in our theoretical analysis since it tells how to perform conditional expectation given the history.

We always require that $\mathcal{F}_{t-1}$ includes all the randomness before iteration $t$ to ensure that Lemma~\ref{lem:single-progress} (and thus Theorems~\ref{thm:smooth} and \ref{thm:strong}) and Theorem~\ref{thm:nag} hold. For further theoretical analysis of various implementations of the framework, $\mathcal{F}_{t-1}$ remains to be specified by possibly also including some randomness in iteration $t$ (see e.g. Example~\ref{exp:new-prgf} as an implementation of Algorithm~\ref{alg:greedy-descent}).\footnote{In mathematical language, if $\mathcal{F}_{t-1}$ includes (and only includes) the randomness brought by random vectors $\{x_1,x_2,\ldots,x_n\}$, then $\mathcal{F}_{t-1}$ is the $\sigma$-algebra generated by $\{x_1,x_2,\ldots,x_n\}$: $\mathcal{F}_{t-1}$ is the smallest $\sigma$-algebra s.t. $x_i$ is $\mathcal{F}_{t-1}$-measurable for all $1\leq i\leq n$.}
\end{remark}
By Remark~\ref{rem:dependence-history}, we introduce $\E_t[\cdot]:=\E[\cdot|\mathcal{F}_{t-1}]$ to denote the conditional expectation given the history. 
In Algorithm~\ref{alg:greedy-descent}, under a suitable choice of $\mathcal{F}_{t-1}$, $\E_t[\cdot]$ roughly means only taking expectation w.r.t.~$v_t$.
We let $\overline{v} :=\nicefrac{v}{\|v\|}$ denote the $\ell_2$ normalization of vector $v$. We let $x^*$ denote one of the minimizers of $f$ (we assume such minimizer exists), and $\delta_t:=f(x_t)-f(x^*)$.
Thanks to the descent property in Algorithm~\ref{alg:greedy-descent}, we have the following lemma on single step progress. 
\begin{lemma}[Proof in Appendix~\ref{sec:B1}]
\label{lem:single-progress}
    Let $C_t:=\left(\overline{\nabla f(x_t)}^\top v_t\right)^2$ and  $L':=\frac{L}{1-(1-\frac{L}{\hat{L}})^2}$, then in Algorithm~\ref{alg:greedy-descent}, we have
    \begin{align}
    \label{eq:single-progress}
        \E_t[\delta_{t+1}] \leq \delta_t -\frac{\E_t[C_t]}{2L'}\|\nabla f(x_t)\|^2.
    \end{align}
\end{lemma}
We note that $L'=\nicefrac{\hat{L}}{(2-\nicefrac{L}{\hat{L}})}$, so $\frac{\hat{L}}{2}\leq L'\leq \hat{L}$ and $L'\geq L$.

To obtain a bound on $\E[\delta_T]$, one of the classical proofs (see e.g., Theorem~1 in \cite{nesterov2012efficiency}) requires us to take expectation on both sides of Eq.~\eqref{eq:single-progress}. 
Allowing the distribution of $v_t$ to be dependent on the history leads to additional technical difficulty: in the r.h.s of Eq.~\eqref{eq:single-progress}
, $\E_t[C_t]$ becomes a random variable that is not independent of $\|\nabla f(x_t)\|^2$. Thus, we cannot simplify the term $\E[\E_t[C_t]\|\nabla f(x_t)\|^2]$ if we take expectation.\footnote{It is also the difficulty we faced in our very preliminary attempts of the theoretical analysis when $f$ is not convex, and we leave its solution or workaround in the non-convex case as future work.} By using other 
techniques, we obtain the following main theorems on convergence rate. 
\begin{theorem}[Algorithm~\ref{alg:greedy-descent}, smooth and convex; proof in Appendix~\ref{sec:B1}]
\label{thm:smooth}
Let $R:=\max_{x: f(x)\leq f(x_0)}\|x-x^*\|$ and suppose $R<\infty$. Then, 
in Algorithm~\ref{alg:greedy-descent}, we have
\begin{align}
\label{eq:smooth}
    \E[\delta_T]\leq \frac{2L' R^2\sum_{t=0}^{T-1} \E\left[\frac{1}{\E_t[C_t]}\right]}{T(T+1)}.
\end{align}
\end{theorem}
\begin{theorem}[Algorithm~\ref{alg:greedy-descent}, smooth and strongly convex; proof in Appendix~\ref{sec:B1}]
\label{thm:strong}
If $f$ is also $\tau$-strongly convex, then we have
\begin{align}
\label{eq:strong}
    \E\left[\frac{\delta_T}{\exp\left(-\frac{\tau}{L'}\sum_{t=0}^{T-1}\E_t[C_t]\right)}\right] \leq \delta_0.
\end{align}
\end{theorem}
\begin{remark} 
In our results, the convergence rate depends on $\E_t[C_t]$ in a more complicated way than only depending on $\E[C_t]=\E[\E_t[C_t]]$. For concrete cases, one may need to study the concentration properties of $\E_t[C_t]$ besides its expectation, as in the proofs of Theorems~\ref{thm:prgf-smooth} and \ref{thm:prgf-strong} when we analyze History-PRGF, a special implementation in the greedy descent framework.

In the strongly convex case, currently we cannot directly obtain a final bound of $\E[\delta_T]$. However, the form of Eq.~\eqref{eq:strong} is still useful since it roughly tells us that $\delta_T$ converges as the denominator $\exp\left(-\frac{\tau}{L'}\sum_{t=0}^{T-1}\E_t[C_t]\right)$. For History-PRGF, we will turn this intuition into a rigorous theorem (Theorem~\ref{thm:prgf-strong}) since in that case we can prove that the denominator has a nice concentration property.
\end{remark}
\begin{remark}
\label{rem:lower-bound-convergence}
If we have a lower bound of $\E_t[C_t]$, e.g. $\E_t[C_t]\geq a>0$, then we directly obtain that $\E[\delta_T]\leq \frac{2L' R^2}{a(T+1)}$ for Theorem~\ref{thm:smooth}, and $\E[\delta_T]\leq \delta_0 \exp(-\frac{\tau}{L'}aT)$ for Theorem~\ref{thm:strong}. These could recover the ideal convergence rate for RGF estimator and the worst-case convergence rate for PRGF estimator, as explained in Examples~\ref{exp:generalized-rgf} and \ref{exp:new-prgf}.
\end{remark}

From Theorems~\ref{thm:smooth} and~\ref{thm:strong}, we see that a larger value of $C_t$ would lead to a better bound. To find a good choice of $v_t$ in Algorithm~\ref{alg:greedy-descent}, it becomes natural to discuss the following problem. Suppose in an iteration in Algorithm~\ref{alg:greedy-descent}, we query the directional derivative oracle at $x_t$ along $q$ directions $\{u_i\}_{i=1}^q$ (maybe randomly chosen) and obtain the values of $\{\nabla f(x_t)^\top u_i\}_{i=1}^q$. We could use this information to construct a vector $v_t$. 
What is the $v_t$ that maximizes $C_t$ s.t. $\|v_t\|=1$?
To answer this question, we give the following proposition based on Proposition~1 in \cite{meier2019improving} and additional justification. 
\begin{proposition}[Optimality of subspace estimator; proof in Appendix~\ref{sec:B2}]
\label{prop:optimality}
In one iteration of Algorithm~\ref{alg:greedy-descent}, if we have queried $\{\nabla f(x_t)^\top u_i\}_{i=1}^q$, then the optimal $v_t$ maximizing $C_t$ s.t. $\|v_t\|=1$ should be in the following form: $v_t=\overline{\nabla f(x_t)_A}$, where $A:=\operatorname{span}\{u_1,u_2,\ldots,u_q\}$ and $\nabla f(x_t)_A$ denotes the projection of $\nabla f(x_t)$ onto $A$.
\end{proposition}
Note that in Line 3 of Algorithm~\ref{alg:greedy-descent}, we have $g_t=\nabla f(x_t)^\top \overline{\nabla f(x_t)_A}\cdot \overline{\nabla f(x_t)_A}=\nabla f(x_t)_A$. Therefore, the gradient estimator $g_t$ is equivalent to the projection of the gradient to the subspace $A$, which justifies its name of subspace estimator. Next we discuss some special cases of the subspace estimator. We leave detailed derivation in following examples to Appendix~\ref{sec:B3}.
\begin{example}[RGF]
\label{exp:generalized-rgf}
$u_i\sim \mathcal{U}(\mathbb{S}^{d-1})$ for $i=1,2,\ldots, q$. Without loss of generality, we assume they are orthonormal (e.g., via Gram-Schmidt orthogonalization).\footnote{ The computational complexity of Gram-Schmidt orthogonalization over $q$ vectors in $\mathbb{R}^d$ is $O(q^2 d)$. Therefore, with a moderate value of $q$ (e.g. $q\in[10,20]$ in our experiments), its cost is usually much smaller than that brought by $O(q)$ function evaluations used to approximate the directional derivatives. We note that when using a numerical computing framework, for orthogonalization one could also adopt an efficient implementation, by calling the QR decomposition procedure such as \texttt{torch.linalg.qr} in PyTorch.} 
The corresponding estimator $g_t=\sum_{i=1}^q \nabla f(x_t)^\top u_i \cdot u_i$ ($v_t=\overline{g_t}$). When $q=1$, the estimator is similar to the random gradient-free oracle in \cite{nesterov2017random}. With $q\geq 1$, it is essentially the same as the stochastic subspace estimator with columns from Haar-distributed random orthogonal matrix \cite{kozak2021stochastic} and similar to the orthogonal ES estimator \cite{choromanski2018structured}. In theoretical analysis, we let $\mathcal{F}_{t-1}$ only include the randomness before iteration $t$, and then we can prove that $\E_t[C_t]=\frac{q}{d}$. By Theorems~\ref{thm:smooth} and~\ref{thm:strong}, the convergence rate is $\E[\delta_T]\leq \frac{2L' R^2 \frac{d}{q}}{T+1}$ for smooth convex case, and $\E[\delta_T]\leq \delta_0 \exp(-\frac{\tau}{L'}\frac{q}{d}T)$ for smooth and strongly convex case. The bound is the same as that in \cite{kozak2021stochastic}. Since the query complexity in each iteration is proportional to $q$, the bound for total query complexity is indeed independent of $q$.
\end{example}
\begin{example}[PRGF]
\label{exp:new-prgf}
With slight notation abuse, we assume the subspace in Proposition~\ref{prop:optimality} is spanned by $\{p_1,\ldots,p_k,u_1,\ldots,u_q\}$, so each iteration takes $q+k$ queries. Let $p_1, \cdots, p_k$ be $k$ non-zero vectors corresponding to the prior message (e.g. the historical update, or the gradient of a surrogate model), and $u_i\sim \mathcal{U}(\mathbb{S}^{d-1})$ for $i=1,2,\ldots, q$. We note that intuitively we cannot construct a better subspace since the only extra information we know is the $k$ priors. In our analysis, we assume $k=1$ for convenience, and we change the original notation $p_1$ to $p_t$ to explicitly show the dependence of $p_t$ on the history.
We note that $p_t$ could also depend on extra randomness in iteration $t$ (see e.g. the specification of $p_t$ in Appendix~\ref{sec:D11}). For convenience of theoretical analysis, we require that $p_t$ is determined before sampling $\{u_1,u_2,\ldots,u_q\}$, and let $\mathcal{F}_{t-1}$ also include the extra randomness of $p_t$ in iteration $t$ (not including the randomness of $\{u_1,u_2,\ldots,u_q\}$) besides the randomness before iteration $t$. Then $p_t$ is always $\mathcal{F}_{t-1}$-measurable, i.e. determined by the history.
Without loss of generality, we assume $\{p_t,u_1,\ldots,u_q\}$ are orthonormal (e.g., via Gram-Schmidt orthogonalization). The corresponding estimator $g_t=\nabla f(x_t)^\top p_t \cdot p_t+\sum_{i=1}^q \nabla f(x_t)^\top u_i \cdot u_i$ ($v_t=\overline{g_t}$), which is similar to the estimator in \cite{meier2019improving}. By \cite{meier2019improving} (the expected drift of $X_t^2$ in its Theorem~1), we have
\begin{lemma}[Proof in Appendix~\ref{sec:B34}]
\label{lem:expected-drift}
In Algorithm~\ref{alg:greedy-descent} with PRGF estimator,
$    \E_t[C_t]=D_t+\frac{q}{d-1}(1-D_t)$
where $D_t:=\left(\overline{\nabla f(x_t)}^\top p_t\right)^2$.
\end{lemma}
Hence $\E_t[C_t]\geq \frac{q}{d}$ holds. By Remark~\ref{rem:lower-bound-convergence}, PRGF admits a guaranteed convergence rate of RGF and is potentially better given a good prior (if $D_t$ is large), but it costs an additional query per iteration. This shows soundness of the PRGF algorithm. For further theoretical analysis, we need to bound $D_t$. This could be done when using the historical prior introduced in Section~\ref{sec:history-prgf} (see Lemma~\ref{lem:decrease}). Bounding $D_t$ is usually challenging when a general prior is adopted, but if the prior is an approximate gradient (such case appears in \cite{maheswaranathan2019guided}), it may be possible. We leave related investigation as future work. 
\end{example}


\subsection{Analysis on the PRGF algorithm with the historical prior}
\label{sec:history-prgf}
We apply the above analysis to a concrete example of the History-PRGF estimator~\cite{meier2019improving}, which considers the historical prior in the PRGF estimator.  In this case, Lemma~\ref{lem:single-progress}, Theorem~\ref{thm:smooth} and Theorem~\ref{thm:strong} will manifest themselves by clearly stating the convergence rate which is robust to the learning rate.

Specifically, History-PRGF considers the historical prior as follows: we choose $p_t=\overline{g_{t-1}}$, i.e., we let the prior be the direction of the previous gradient estimate.\footnote{One can also utilize multiple historical priors (e.g. the last $k$ updates with $k>1$, as proposed and experimented in \cite{meier2019improving}), but here we only analyze the $k=1$ case.} This is equivalent to letting $p_t=v_{t-1}$. Thus, in Algorithm~\ref{alg:greedy-descent}, $v_t=\overline{\nabla f(x_t)_A}=\overline{\nabla f(x_t)^\top v_{t-1}\cdot v_{t-1}+\sum_{i=1}^q\nabla f(x_t)^\top u_i\cdot u_i}$. In this form we require $\{v_{t-1},u_1,\ldots,u_q\}$ to be orthonormal, so we first determine $v_{t-1}$, and then sample $\{u_i\}_{i=1}^q$ in $A_\perp$, the $(d-1)$-dimensional subspace of $\mathbb{R}^d$ perpendicular to $v_{t-1}$, and then do Gram-Schmidt orthonormalization on $\{u_i\}_{i=1}^q$.

To study the convergence rate, we first study evolution of $C_t$ under a general $L$-smooth function. This extends the analysis on linear functions (corresponding to $L=0$) in \cite{meier2019improving}. Under the framework of Algorithm~\ref{alg:greedy-descent}, intuitively, the change of the gradient should be smaller when the objective function is very smooth ($L$ is small) or the learning rate is small ($\hat{L}$ is large). Since we care about the cosine similarity between the gradient and the prior, we prove the following lemma:
\begin{lemma}[Proof in Appendix~\ref{sec:B4}]
\label{lem:decrease}
In History-PRGF ($p_t=v_{t-1}$), we have 
$
    D_t\geq \left(1-\nicefrac{L}{\hat{L}}\right)^2 C_{t-1}.
$
\end{lemma}
When $\hat{L}=L$, i.e., using the optimal learning rate, Lemma~\ref{lem:decrease} does not provide a useful bound, since an optimal learning rate in smooth optimization could find an approximate minimizer along the update direction, so the update direction may be useless in next iteration. In this case, the historical prior does not provide acceleration. Hence, Lemma~\ref{lem:decrease} explains the empirical findings in \cite{meier2019improving} that past directions can be less useful when the learning rate is larger. However, in practice we often use a conservative learning rate, for the following reasons: 1) We usually do not know $L$, and the cost of tuning learning rate could be large; 2) Even if $L$ is known, it only provides a global bound, so a fixed learning rate could be too conservative in the smoother local regions. In scenarios where $\hat{L}$ is too conservative ($\hat{L}>L$), History-PRGF could bring more acceleration over RGF.

By Lemma~\ref{lem:decrease}, we can assume $D_t=\left(1-\nicefrac{L}{\hat{L}}\right)^2 C_{t-1}$ to obtain a lower bound of quantities about $C_t$. Meanwhile, since $D_t$ means quality of the prior, the construction in Example~\ref{exp:new-prgf} tells us relationship between $C_t$ and $D_t$. Then we have full knowledge of the evolution of $C_t$, and thus $\E_t[C_t]$. In Appendix~\ref{sec:B4}, we discuss about evolution of $\E[C_t]$ and show that $\E[C_t]\to O(\frac{q}{d}\frac{L'}{L})$ if $\frac{q}{d}\leq \frac{L}{\hat{L}}\leq 1$. Therefore, by Lemma~\ref{lem:single-progress}, assuming $\E_t[C_t]$ concentrates well around $\E[C_t]$ and hence $\frac{\E_t[C_t]}{L'}\approx \frac{q}{dL}$, PRGF could recover the single step progress with optimal learning rate ($\hat{L}=L$), since Eq.~\eqref{eq:single-progress} only depends on $\frac{\E_t[C_t]}{L'}$ which is constant w.r.t. $L'$ now. 
While the above discussion is informal, based on Theorems~\ref{thm:smooth} and \ref{thm:strong}, we prove following theorems which show that convergence rate of History-PRGF is robust to choice of learning rate.

\begin{theorem}[History-PRGF, smooth and convex; proof in Appendix~\ref{sec:B51}]
\label{thm:prgf-smooth}
In the setting of Theorem~\ref{thm:smooth}, assuming $d\geq 4$, $\frac{q}{d-1}\leq \frac{L}{\hat{L}} \leq 1$ and $T> \left\lceil\frac{d}{q}\right\rceil$ ($\lceil\cdot\rceil$ denotes the ceiling function), we have
\begin{align}
\label{eq:prgf-smooth}
    \E[\delta_T]\leq \left(\frac{32}{q}+2\right)\frac{2L\frac{d}{q} R^2}{T-\left\lceil\frac{d}{q}\right\rceil+1}.
\end{align} 
\end{theorem}
\textit{Sketch of the proof}.
The idea of the proof of Theorem~\ref{thm:prgf-smooth} is to show that for the random variable $\E_t[C_t]$, its standard deviation $\sqrt{\V[\E_t[C_t]]}$ is small relative to its expectation $\E[\E_t[C_t]]=\E[C_t]$. By Chebyshev's inequality, we can bound $\E\left[\frac{1}{\E_t[C_t]}\right]$ in Theorem~\ref{thm:smooth} with $\frac{1}{\E[\E_t[C_t]]}=\frac{1}{\E[C_t]}$. In the actual proof we replace $C_t$ that appears above with a lower bound $E_t$.

\begin{theorem}[History-PRGF, smooth and strongly convex; proof in Appendix~\ref{sec:B53}]
\label{thm:prgf-strong}
Under the same conditions as in Theorem~\ref{thm:strong}, then assuming $d\geq 4$, $\frac{q}{d-1}\leq \frac{L}{\hat{L}} \leq 1$, $\frac{q}{d}\leq 0.2\frac{L}{\tau}$, and $T\geq 5\frac{d}{q}$, we have
\begin{align}
\label{eq:prgf-strong}
    \E[\delta_T]\leq 2\exp\left(-0.1\frac{q}{d}\frac{\tau}{L}T\right)\delta_0.
\end{align}
\end{theorem}
This result seems somewhat surprising since Theorem~\ref{thm:strong} does not directly give a bound of $\E[\delta_T]$.

\textit{Sketch of the proof}.
The goal is to show that the denominator in the l.h.s of Eq.~\eqref{eq:strong} in Theorem~\ref{thm:strong}, $\exp(-\frac{\tau}{\tilde{L}}\sum_{t=0}^{T-1}\E_t[C_t])$, concentrates very well. Indeed, the probability that $\exp(-\frac{\tau}{\tilde{L}}\sum_{t=0}^{T-1}\E_t[C_t])$ is larger than $\exp(-0.1\frac{q}{d}\frac{\tau}{L}T)$ is very small so that its influence can be bounded by another $\exp(-0.1\frac{q}{d}\frac{\tau}{L}T)\delta_0$, leading to the coefficient $2$ in Eq.~\eqref{eq:prgf-strong}. In our actual analysis we replace $C_t$ that appears above with a lower bound $E_t$. 
\begin{remark}
\label{rem:robustness}
As stated in Example~\ref{exp:new-prgf}, using RGF with the optimal learning rate, we have $\E[\delta_T]\leq \frac{2L\frac{d}{q}R^2}{T+1}$ for smooth and convex case, and $\E[\delta_T]\leq \exp\left(-\frac{q}{d}\frac{\tau}{L}T\right)\delta_0$ for smooth and strongly convex case. Therefore, History-PRGF with a suboptimal learning rate $\frac{1}{\hat{L}}$ under the condition $\frac{q}{d-1}\frac{1}{L}\leq \frac{1}{\hat{L}}\leq \frac{1}{L}$ could reach similar convergence rate to RGF with optimal learning rate (up to constant factors), which indicates that History-PRGF is more robust to learning rate than RGF.
\end{remark}
\begin{remark}
\label{rem:loose}
We note that the constants in the bounds are loose and have a large potential to be improved in future work, and empirically the convergence rate of History-PRGF is often not worse than RGF using the optimal learning rate (see Fig.~\ref{fig:history}).
\end{remark}

As a sidenote, we discuss how to set $q$ in History-PRGF. The iteration complexity given by Theorems~\ref{thm:prgf-smooth} and \ref{thm:prgf-strong} is proportional to $\frac{1}{q}$ if we ignore the constants such as $\frac{32}{q}+2$ in Eq.~\eqref{eq:prgf-smooth} by Remark~\ref{rem:loose}. Meanwhile, we recall that each iteration of PRGF requires $q+1$ queries to the directional derivative oracle, so the total query complexity is roughly proportional to $\frac{q+1}{q}$. Hence, a very small $q$ (e.g. $q=1$) is suboptimal. On the other hand, Theorems~\ref{thm:prgf-smooth} and \ref{thm:prgf-strong} require $\frac{1}{\hat{L}}\in \left[\frac{q}{d-1}\frac{1}{L}, \frac{1}{L}\right]$, so to enable robustness of History-PRGF to a wider range of the choice of learning rate, $q$ should not be too large. In summary, it is desirable to set $q$ to a moderate value.

Note that 
if we adopt line search in Algorithm~\ref{alg:greedy-descent}, then 
one can adapt the learning rate in a huge range and reach the convergence guarantee with the optimal learning rate under weak assumptions. 
Nevertheless, it is still an intriguing fact that History-PRGF could perform similarly to methods adapting the learning rate, while its mechanism is very different. Meanwhile, History-PRGF is easier to be implemented and parallelized compared with methods like line search, since its implementation is the same as that of the RGF baseline except that it records and uses the historical prior. 

\section{Extension of ARS framework and PARS algorithm}
\label{sec:4}
To further accelerate greedy descent methods, we extend our analysis to a new variant of Accelerated Random Search (ARS)~\cite{nesterov2017random} by incorporating prior information, under the smooth and convex setting.\footnote{
The procedure of ARS requires knowledge of the strong convexity parameter $\tau$ ($\tau$ can be $0$), but for clarity we only discuss the case $\tau=0$ here (i.e., we do not consider strong convexity), and leave the strongly convex case to Appendix~\ref{sec:C6}.}

By delving into the proof in \cite{nesterov2017random}, we present our extension to ARS in Algorithm~\ref{alg:nag-extended}, state its convergence guarantee in Theorem~\ref{thm:nag} and explain its design in the proof sketch in Appendix~\ref{sec:C1}.

\begin{algorithm}[!htbp]
\small
\caption{Extended accelerated random search framework\protect\footnotemark}
\label{alg:nag-extended}
\begin{algorithmic}[1]
\Require $L$-smooth convex function $f$; initialization $x_0$; $\hat{L} \geq L$; iteration number $T$; $\gamma_0>0$.
\Ensure $x_T$ as the approximate minimizer of $f$.
\State $m_0\leftarrow x_0$;
\For {$t = 0$ to $T-1$}
\State Find a $\theta_t>0$ such that $\theta_t\leq \frac{\E_t\left[\left(\nabla f(y_t)^\top v_t\right)^2\right]}{\hat{L}\cdot\E_t[\|g_2(y_t)\|^2]}$ where $y_t$, $v_t$ and $g_2(y_t)$ are defined in following steps: \label{lne:3-alg2}
\State Step 1: $y_t\leftarrow(1-\alpha_t)x_t+\alpha_t m_t$, where $\alpha_t$ is a positive root of the equation $\alpha_t^2=\theta_t (1-\alpha_t)\gamma_t$; $\gamma_{t+1}\leftarrow (1-\alpha_t)\gamma_t$;
\State Step 2: Let $v_t$ be a random vector s.t. $\|v_t\|=1$; $g_1(y_t)\leftarrow \nabla f(y_t)^\top v_t \cdot v_t$;
\State Step 3: Let $g_2(y_t)$ be an unbiased estimator of $\nabla f(y_t)$, i.e., $\E_t[g_2(y_t)]=\nabla f(y_t)$;
\State $x_{t+1}\leftarrow y_t - \frac{1}{\hat{L}} g_1(y_t)$, $m_{t+1}\leftarrow m_t - \frac{\theta_t}{\alpha_t} g_2(y_t)$;
\EndFor
\Return $x_T$.
\end{algorithmic}
\end{algorithm}
\begin{theorem}[Proof in Appendix~\ref{sec:C1}]
\label{thm:nag}
In Algorithm~\ref{alg:nag-extended}, if $\theta_t$ is $\mathcal{F}_{t-1}$-measurable
(see Appendix~\ref{sec:C1} for more explanation), 
then we have
\begin{align}
\label{eq:theorem-nag}
    \E\left[(f(x_T)-f(x^*))\left(1+\frac{\sqrt{\gamma_0}}{2}\sum_{t=0}^{T-1}\sqrt{\theta_t}\right)^2\right]\leq f(x_0)-f(x^*)+\frac{\gamma_0}{2}\|x_0-x^*\|^2.
\end{align}
\end{theorem}
\footnotetext{Keys in this extension are: 1) We require $\E_t[g_2(y_t)]=\nabla f(y_t)$. Thus, $g_2(y_t)$ could not be the PRGF estimator as it is biased towards the prior; 2) To accelerate convergence, we need to find an appropriate $\theta_t$ since it appears in Eq.~\eqref{eq:theorem-nag}. If we set $\theta_t$ to the value in ARS baseline, then no potential acceleration is guaranteed.}
\begin{remark}
\label{rem:ars-convergence}
If we let $g_1(y_t)$ be the RGF estimator in Example~\ref{exp:generalized-rgf} and let $g_2(y_t)=\nicefrac{d}{q}\cdot g_1(y_t)$, we can show that $\E[g_2(y_t)]=\nabla f(y_t)$ and $\theta_t$ could be chosen as $\frac{q^2}{\hat{L}d^2}$ since $\frac{\E_t\left[\left(\nabla f(y_t)^\top v_t\right)^2\right]}{\hat{L}\cdot\E_t[\|g_2(y_t)\|^2]}=\frac{q^2}{\hat{L}d^2}$. Then roughly, the convergence rate $\propto q$, so the total query complexity is independent of $q$. When $q=1$, ARS baseline is recovered. For convenience we call the algorithm ARS regardless of the value of $q$.
\end{remark}
\begin{remark}
\label{rem:ars-lower}
If we have a uniform constant lower bound $\theta>0$ such that $\forall t, \theta_t\geq \theta$, then we have
\begin{align}
    \E\left[f(x_T) - f(x^*)\right] \leq \left(1+\frac{\sqrt{\gamma_0}}{2}T\sqrt{\theta}\right)^{-2}\left(f(x_0)-f(x^*)+\frac{\gamma_0}{2}\|x_0-x^*\|^2\right).
\end{align}
\end{remark}

Next we present Prior-Guided ARS (PARS) by specifying the choice of $g_1(y_t)$ and $g_2(y_t)$ in Algorithm~\ref{alg:nag-extended} when prior information $p_t\in\mathbb{R}^d$ is available. Since we want to maximize the value of $\theta_t$, regarding $g_1(y_t)$ we want to maximize $\E_t\left[\left(\nabla f(y_t)^\top v_t\right)^2\right]$. By Proposition~\ref{prop:optimality} and Example~\ref{exp:new-prgf}, it is natural to let $g_1(y_t)$ be the PRGF estimator for $\nabla f(y_t)$. Then by Lemma~\ref{lem:expected-drift}, we have $\E_t\left[\left(\nabla f(y_t)^\top v_t\right)^2\right]=\|\nabla f(y_t)\|^2(D_t+\frac{q}{d-1}(1-D_t))$, where $D_t:=\left(\overline{\nabla f(y_t)}^\top p_t\right)^2$. The remaining problem is to construct $g_2(y_t)$, an unbiased estimator of $\nabla f(y_t)$ ($\E_t[g_2(y_t)]=\nabla f(y_t)$), and make $\E_t[\|g_2(y_t)\|^2]$ as small as possible. We leave the construction of $g_2(y_t)$ in Appendix~\ref{sec:C2}. Finally, we calculate the following expression which appears in Line~\ref{lne:3-alg2} of Algorithm~\ref{alg:nag-extended} to complete the description of PARS:
\begin{align}
    \frac{\E_t\left[\left(\nabla f(y_t)^\top v_t\right)^2\right]}{\hat{L}\cdot\E_t[\|g_2(y_t)\|^2]} = \frac{D_t+\frac{q}{d-1}(1-D_t)}{\hat{L}\left(D_t+\frac{d-1}{q}(1-D_t)\right)}.
\end{align}
Since $D_t\geq 0$, the right-hand side is larger than $\nicefrac{q^2}{\hat{L}d^2}$ (by Remark~\ref{rem:ars-convergence}, this value corresponds to the value of $\theta_t$ in ARS), so by Remark~\ref{rem:ars-lower} PARS is guaranteed a convergence rate of ARS.

In implementation of PARS, we note that there remain two problems to solve. The first is that $D_t$ is not accessible through one oracle query, since $D_t=(\overline{\nabla f(y_t)}^\top p_t)^2=\left(\nicefrac{\nabla f(y_t)^\top p_t}{\|\nabla f(y_t)\|}\right)^2$, and $\|\nabla f(y_t)\|$ requires estimation. Fortunately, the queries used to construct $g_1(y_t)$ and $g_2(y_t)$ can also be used to estimate $D_t$. With a moderate value of $q$, we can prove that considering the error brought by estimation of $D_t$, a modified version of PARS is guaranteed to converge with high probability. We leave related discussion to Appendix~\ref{sec:C3}. The second is that Line~\ref{lne:3-alg2} of Algorithm~\ref{alg:nag-extended} has a subtlety that $y_t$ depends on $\theta_t$, so we cannot directly determine an optimal $\theta_t$ satisfying $\theta_t\leq \frac{\E_t\left[\left(\nabla f(y_t)^\top v_t\right)^2\right]}{\hat{L}\cdot\E_t[\|g_2(y_t)\|^2]}$. Theoretically, we can guess a conservative estimate of $\theta_t$ and verify this inequality, but in practice we adopt a more aggressive strategy to find an approximate solution of $\theta_t$. We leave the actual implementation, named PARS-Impl, in Appendix~\ref{sec:C4}.

In PARS, if we adopt the historical prior as in Section~\ref{sec:history-prgf}, i.e., letting $p_t$ be the previous PRGF gradient estimator $\overline{g_1(y_{t-1})}$, then we arrive at a novel algorithm named History-PARS. Here, we note that unlike the case in History-PRGF, it is more difficult to derive the evolution of $\theta_t$ theoretically, so we currently cannot prove theorems corresponding to Theorem~\ref{thm:prgf-smooth}. However, History-PARS can be guaranteed the convergence rate of ARS, which is desirable since if we adopt line search in ARS to reach robustness against learning rate (e.g. in \cite{stich2013optimization}), currently there is no convergence guarantee. We present the actual implementation History-PARS-Impl in Appendix~\ref{sec:C5} and empirically verify that History-PARS-Impl is robust to learning rate in Section~\ref{sec:experiments}. 
\section{Experiments}
\label{sec:5}
\label{sec:experiments}
\subsection{Numerical benchmarks}
\label{sec:5.1}
\label{sec:experiments_test}

We first experiment on several closed-form test functions to support our theoretical claims. We leave more details of experimental settings to Appendix~\ref{sec:D1}.

First, we present experimental results when a general useful prior is provided. The prior-guided methods include PRGF, PARS (refers to PARS-Impl) and PARS-Naive (simply replacing the RGF estimator in ARS with the PRGF estimator). We adopt the setting in Section~4.1 of \cite{maheswaranathan2019guided} in which the prior is a biased version of the true gradient. 
Our test functions are as follows: 1) $f_1$ is the ``worst-case smooth convex function'' used to construct the lower bound complexity of first-order optimization, as in \cite{nesterov2017random};
2) $f_2$ is a simple smooth and strongly convex function with a worst-case initialization:
$f_2(x)=\frac{1}{d}\sum_{i=1}^d \left(i\cdot (x^{(i)})^2\right)$, where $x_0^{(1)}=d, x_0^{(i)}=0 \text{ for }i\geq 2$; and 3) $f_3$ is the Rosenbrock function ($f_8$ in \cite{hansen2009real}) which is a well-known non-convex function used to test the performance of optimization problems.
For $f_1$ and $f_2$, we set $\hat{L}$ to ground truth value $L$; for $f_3$, we search $\hat{L}$ for best performance for each algorithm. 
We set $d=256$ for all test functions and set $q$ such that each iteration of these algorithms costs $11$ queries\footnote{That is, for prior-guided algorithms we set $q=10$, and for other algorithms (RGF and ARS) we set $q=11$.} to the directional derivative oracle.\footnote{The directional derivative is approximated by finite differences. In PARS, $2$ additional queries to the directional derivative oracle per iteration are required to find $\theta_t$ (see Appendix~\ref{sec:C4}).} We plot the experimental results in Fig.~\ref{fig:biased}, where the horizontal axis represents the number of iterations divided by $\lfloor\nicefrac{d}{11}\rfloor$, and the vertical axis represents $\log_{10}\frac{f(x_{\mathrm{current}})-f(x^*)}{f(x_0)-f(x^*)}$. Methods without using the prior information are shown with dashed lines. We also plot the 95\% confidence interval in the colored region.
\begin{figure*}[t]
\centering
\begin{subfigure}{0.25\textwidth}
\centering
\includegraphics[width=0.95\linewidth]{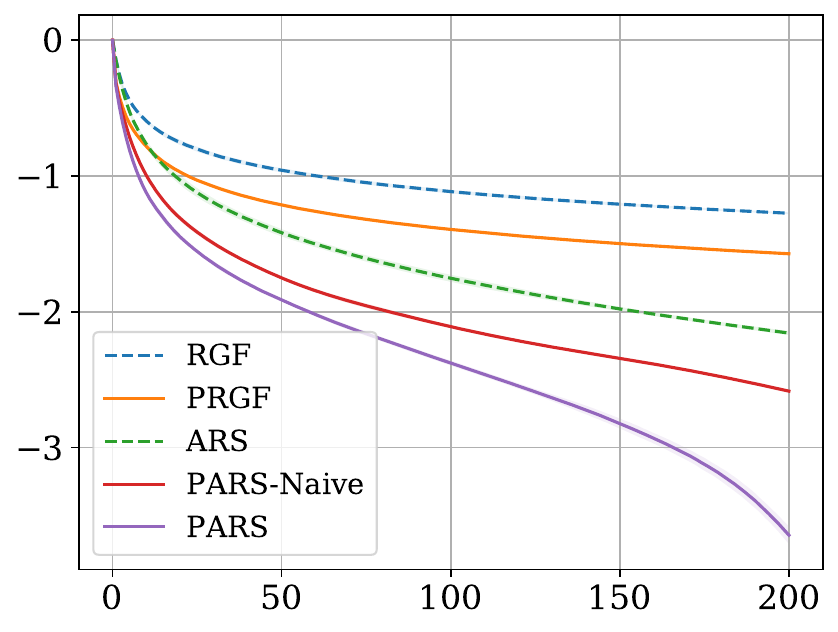}
\caption{$f_1$}
\end{subfigure}
\begin{subfigure}{0.25\textwidth}
\centering
\includegraphics[width=0.95\linewidth]{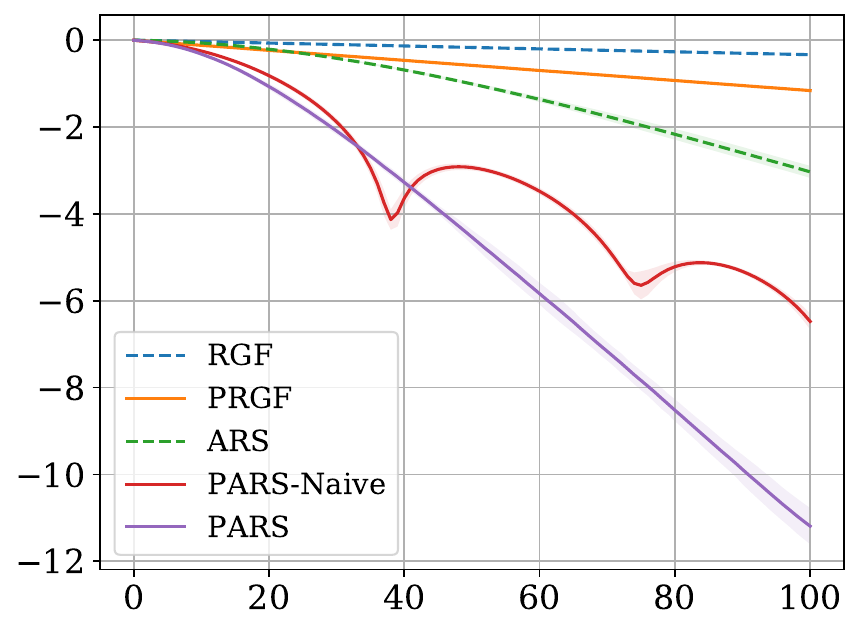}
\caption{$f_2$}
\end{subfigure}
\begin{subfigure}{0.25\textwidth}
\centering
\includegraphics[width=0.95\linewidth]{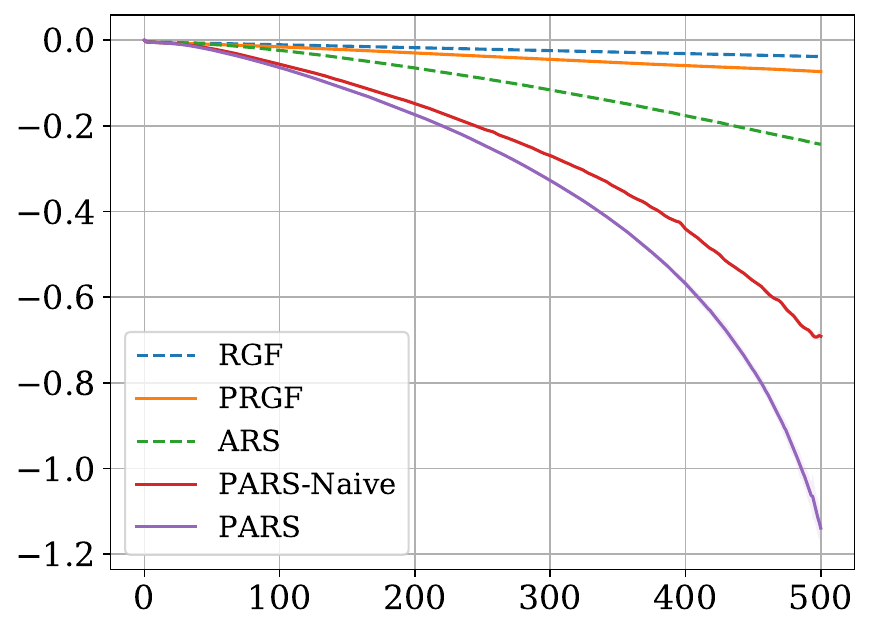}
\caption{$f_3$}
\end{subfigure}
\centering
\small
\caption{Experimental results using biased gradient as the prior (best viewed in color).}
\label{fig:biased}
\end{figure*}
The results show that for these functions (which have ill-conditioned Hessians), ARS-based methods perform better than the methods based on greedy descent. Importantly, the utilization of the prior could significantly accelerate convergence for both greedy descent and ARS. We note that the performance of our proposed PARS algorithm is better than PARS-Naive which naively replaces the gradient estimator in the original ARS with the PRGF estimator, demonstrating the value of our algorithm design with convergence analysis.

Next, we verify the properties of History-PRGF and History-PARS, i.e., the historical-prior-guided algorithms. In this part we set $d=500$.
We first verify that they are robust against learning rate on $f_1$ and $f_2$, and plot the results in Fig.~\ref{fig:history}(a)(b).\footnote{In Fig.~\ref{fig:history}, the setting of ARS-based methods are different from that in Fig.~\ref{fig:biased} as explained in Appendix~\ref{sec:D1}, which leads to many differences of the ARS curves between Fig.~\ref{fig:biased} and Fig.~\ref{fig:history}.} In the legend, for example, `RGF' means RGF using the optimal learning rate ($\hat{L}=L$), and `RGF-0.02' means that the learning rate is set to $0.02$ times of the optimal one ($\hat{L}=50L$). We note that for PRGF and PARS, $q=10$, so $\frac{q}{d}=0.02$. From Fig.~\ref{fig:history}(a)(b), we see that: 1) when using the optimal learning rate, the performance of prior-guided algorithms is not worse than that of its corresponding baseline; and 2) the performance of prior-guided algorithms under the sub-optimal learning rate such that $\frac{q}{d}\leq \frac{L}{\hat{L}}\leq 1$ is at least comparable to that of its corresponding baseline with optimal learning rate. However, for baseline algorithms (RGF and ARS), the convergence rate significantly degrades if a smaller learning rate is set. In summary, we verify our claims that History-PRGF and History-PARS are robust to learning rate if $\frac{q}{d}\leq \frac{L}{\hat{L}}\leq 1$. Moreover, we show that they can provide acceleration over baselines with optimal learning rate on functions with varying local smoothness. We design a new test function as follows:
\begin{align}
    f_4(x)=\begin{cases}
        \frac{1}{2}r^2, & r\leq 1\\
        r-\frac{1}{2}, & r > 1
    \end{cases}, r=\sqrt{f_2(x)},\text{ where }x_0^{(1)}=5\sqrt{d}, x_0^{(i)}=0 \text{ for }i\geq 2.
\end{align}
We note that $f_4$ in regions far away from the origin is more smooth than in the region near the origin, and the global smoothness parameter is determined by the worst-case situation (the region near the origin). Therefore, baseline methods using an optimal learning rate could also manifest sub-optimal performance. Fig.~\ref{fig:history}(c) shows the results. We can see that when utilizing the historical prior, the algorithm could show behaviors of adapting to the local smoothness, thus accelerating convergence when the learning rate is locally too conservative.

\begin{figure*}[!htbp]
\centering
\begin{subfigure}{0.25\textwidth}
\centering
\includegraphics[width=0.95\linewidth]{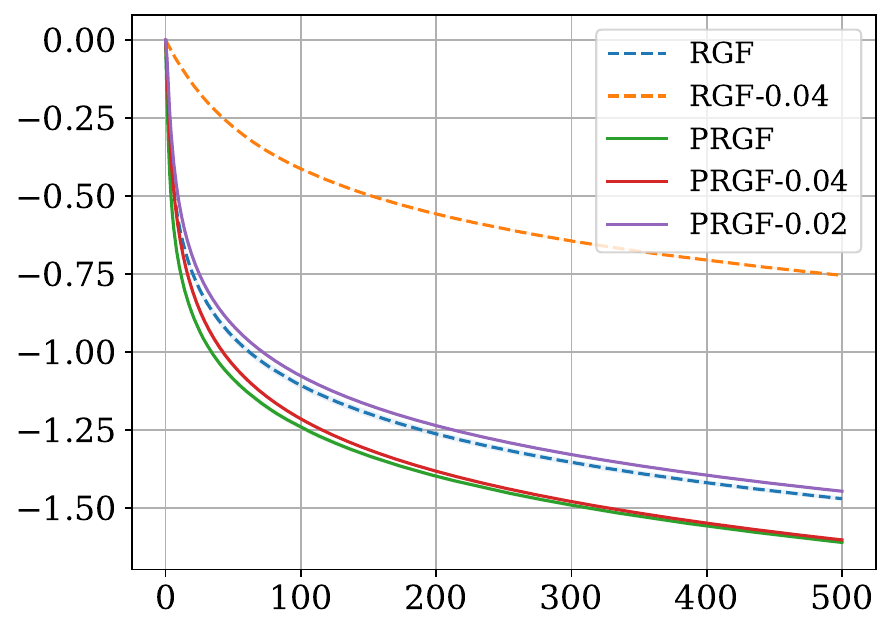}
\end{subfigure}
\begin{subfigure}{0.25\textwidth}
\centering
\includegraphics[width=0.95\linewidth]{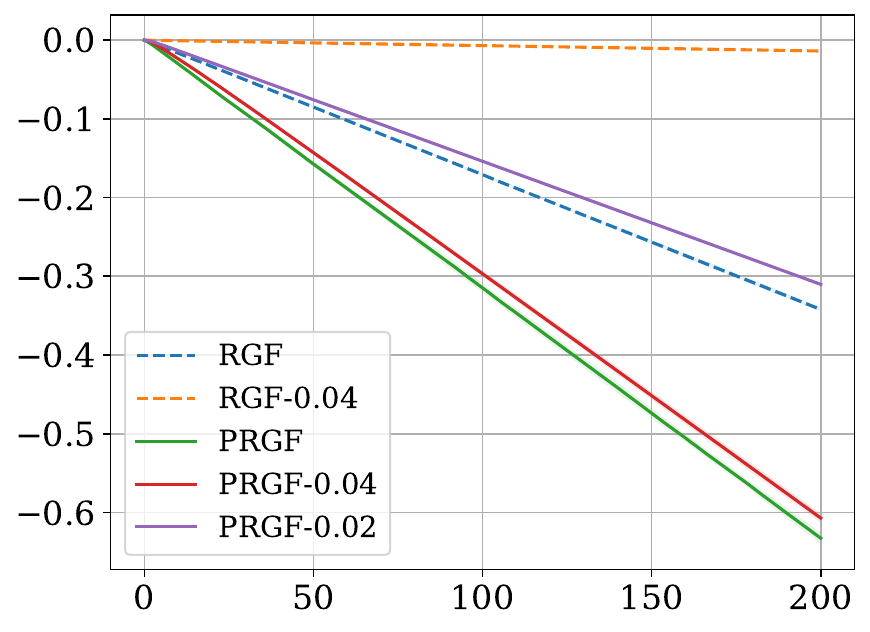}
\end{subfigure}
\begin{subfigure}{0.25\textwidth}
\centering
\includegraphics[width=0.95\linewidth]{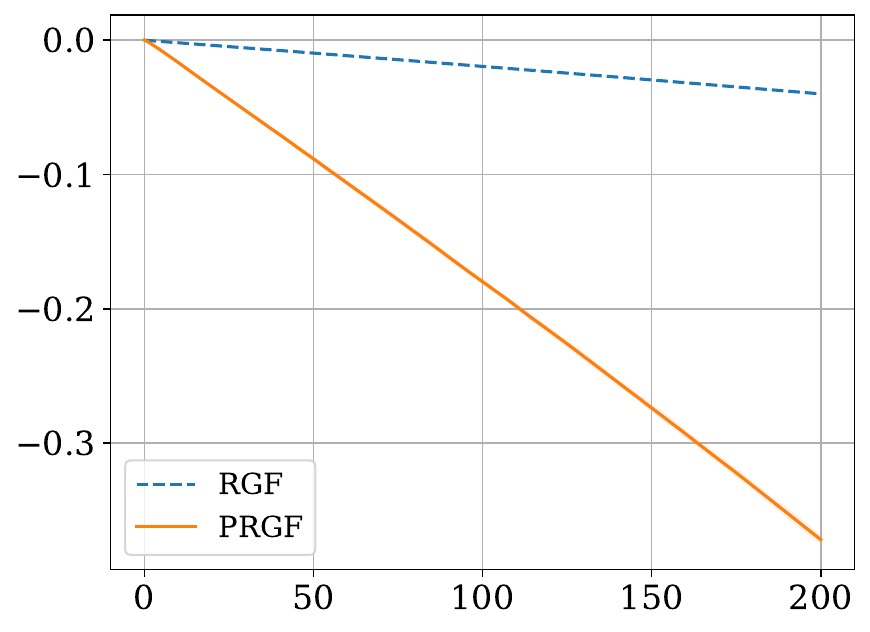}
\end{subfigure}
\centering
\begin{subfigure}{0.25\textwidth}
\centering
\includegraphics[width=0.95\linewidth]{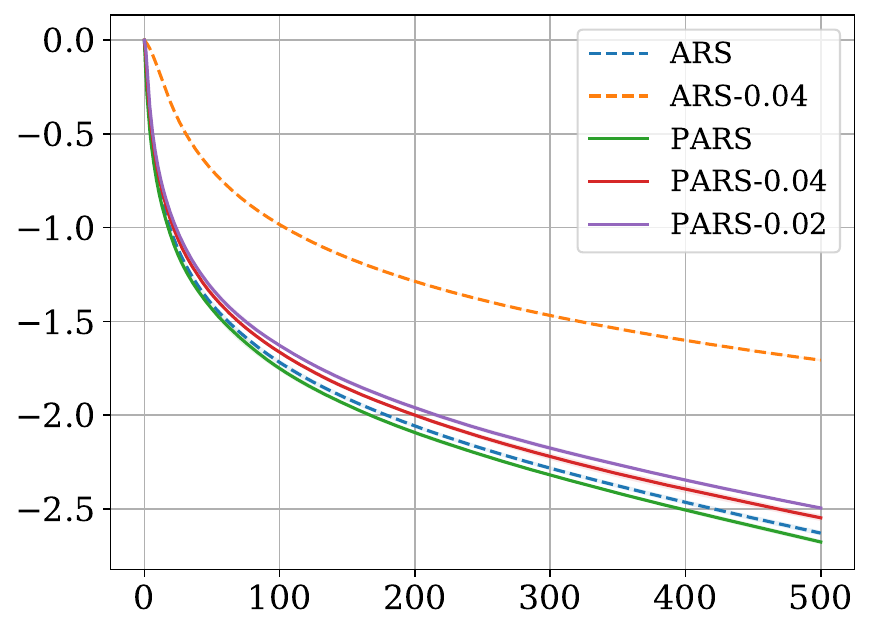}
\caption{$f_1$}
\end{subfigure}
\begin{subfigure}{0.25\textwidth}
\centering
\includegraphics[width=0.95\linewidth]{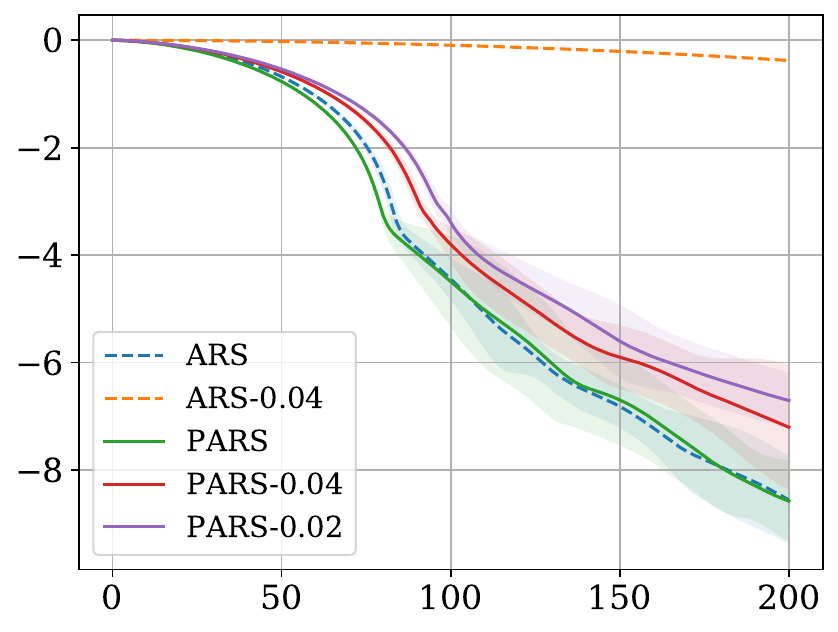}
\caption{$f_2$}
\end{subfigure}
\begin{subfigure}{0.25\textwidth}
\centering
\includegraphics[width=0.95\linewidth]{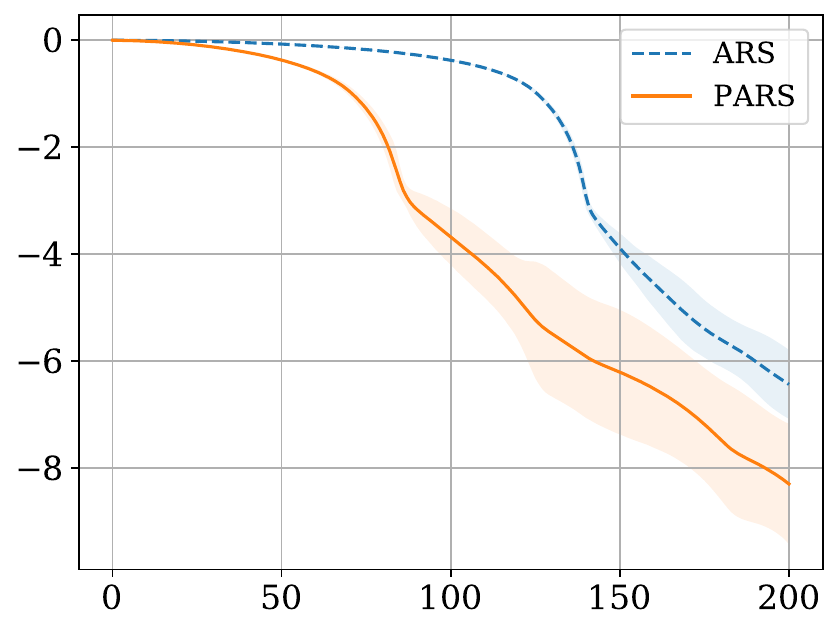}
\caption{$f_4$}
\end{subfigure}
\centering
\small
\caption{Experimental results using the historical prior.}
\label{fig:history}
\vspace{-.2cm}
\end{figure*}
\subsection{Black-box adversarial attacks}
\label{sec:5.2}

In this section, we perform ZO optimization on real-world problems. We conduct score-based black-box targeted adversarial attacks on 500 images from MNIST and leave more details of experimental settings to Appendix~\ref{sec:D2}. In view of optimization, this corresponds to performing constrained maximization over $\{f_i\}_{i=1}^{500}$ respectively, where $f_i$ denotes the loss function to maximize in attacks. For each image $i$, we record the number of queries of $f_i$ used in optimization until the attack succeeds (when using the C\&W loss function \cite{carlini2017towards}, this means $f_i> 0$). For each optimization method, we report the median query number over these images (smaller is better) in Table~\ref{tab:mnist_attack}. The subscript of the method name indicates the learning rate $\nicefrac{1}{\hat{L}}$. For all methods we set $q$ to $20$.
Since \cite{cheng2019improving} has shown that using PRGF estimator with transfer-based prior significantly outperforms using RGF estimator in adversarial attacks, for prior-guided algorithms here we only include the historical prior case.

\begin{table}[htbp]
\vspace{-.2cm}
\centering
\caption{Attack results on MNIST.}
\begin{tabular}{lc|lc}
\toprule
\sc Method    & \sc Median Query & \sc Method    & \sc Median Query \\ \midrule
$\text{RGF}_{0.2}$   & 777   & $\text{ARS}_{0.2}$   & 735        \\
$\text{RGF}_{0.1}$   & 1596  & $\text{ARS}_{0.1}$   & 1386        \\
$\text{History-PRGF}_{0.2}$  & 484   & $\text{History-PARS}_{0.2}$   & 484        \\
$\text{History-PRGF}_{0.1}$  & 572   & $\text{History-PARS}_{0.1}$   & 550        \\
$\text{History-PRGF}_{0.05}$ & 704   & $\text{History-PARS}_{0.05}$   & 726        \\ \bottomrule
\end{tabular}
\label{tab:mnist_attack}
\end{table}

We found that in this task, ARS-based methods perform comparably to RGF-based ones. 
This could be because 1) the numbers of iterations until success of attacks are too small to show the advantage of ARS; 2) currently ARS is not guaranteed to converge faster than RGF under non-convex problems.
We leave more evaluation of ARS-based methods in adversarial attacks and further improvement of their performance as future work.
Experimental results show that History-PRGF is more robust to learning rate than RGF. However, a small learning rate could still lead to its deteriorated performance due to non-smoothness of the objective function. The same statement holds for ARS-based algorithms.
\section{Conclusion and discussion}
\label{sec:conclusion}
In this paper, we present a convergence analysis on existing prior-guided ZO optimization algorithms including PRGF and History-PRGF. We further propose a novel prior-guided ARS algorithm with convergence guarantee. Experimental results confirm our theoretical analysis.

Our limitations lie in: 1) we adopt a directional derivative oracle in our analysis, so the error on the convergence bound brought by finite-difference approximation has not been clearly stated; and 2) our implementation of PARS in practice requires an approximate solution of $\theta_t$, and the accuracy and influence of this approximation is not well studied yet. We leave these as future work. Other future work includes extension of the theoretical analysis to non-convex cases, and more empirical studies in various application tasks. 

\section*{Acknowledgements}
This work was supported by the National Key Research and Development Program of China (No. 2020AAA0104304), NSFC Projects (Nos. 61620106010, 62061136001, 61621136008, 62076147, U19B2034, U19A2081, U1811461), Beijing NSF Project (No. JQ19016), Beijing Academy of Artificial Intelligence (BAAI), Tsinghua-Huawei Joint Research Program, a grant from Tsinghua Institute for Guo Qiang, Tiangong Institute for Intelligent Computing, and the NVIDIA NVAIL Program with GPU/DGX Acceleration.





{\small
\bibliographystyle{plainnat}
\bibliography{refs_all}}

\begin{thebibliography}{33}
\providecommand{\natexlab}[1]{#1}
\providecommand{\url}[1]{\texttt{#1}}
\expandafter\ifx\csname urlstyle\endcsname\relax
  \providecommand{\doi}[1]{doi: #1}\else
  \providecommand{\doi}{doi: \begingroup \urlstyle{rm}\Url}\fi

\bibitem[Andrychowicz et~al.(2016)Andrychowicz, Denil, Gomez, Hoffman, Pfau,
  Schaul, Shillingford, and De~Freitas]{andrychowicz2016learning}
Marcin Andrychowicz, Misha Denil, Sergio Gomez, Matthew~W Hoffman, David Pfau,
  Tom Schaul, Brendan Shillingford, and Nando De~Freitas.
\newblock Learning to learn by gradient descent by gradient descent.
\newblock \emph{arXiv preprint arXiv:1606.04474}, 2016.

\bibitem[Bansal and Gupta(2019)]{bansal2019potential}
Nikhil Bansal and Anupam Gupta.
\newblock Potential-function proofs for gradient methods.
\newblock \emph{Theory of Computing}, 15\penalty0 (1):\penalty0 1--32, 2019.

\bibitem[Berahas et~al.(2021)Berahas, Cao, Choromanski, and
  Scheinberg]{berahas2021theoretical}
Albert~S Berahas, Liyuan Cao, Krzysztof Choromanski, and Katya Scheinberg.
\newblock A theoretical and empirical comparison of gradient approximations in
  derivative-free optimization.
\newblock \emph{Foundations of Computational Mathematics}, pages 1--54, 2021.

\bibitem[Bergou et~al.(2020)Bergou, Gorbunov, and
  Richtarik]{bergou2020stochastic}
El~Houcine Bergou, Eduard Gorbunov, and Peter Richtarik.
\newblock Stochastic three points method for unconstrained smooth minimization.
\newblock \emph{SIAM Journal on Optimization}, 30\penalty0 (4):\penalty0
  2726--2749, 2020.

\bibitem[Brunner et~al.(2019)Brunner, Diehl, Le, and
  Knoll]{brunner2019guessing}
Thomas Brunner, Frederik Diehl, Michael~Truong Le, and Alois Knoll.
\newblock Guessing smart: Biased sampling for efficient black-box adversarial
  attacks.
\newblock In \emph{Proceedings of the IEEE/CVF International Conference on
  Computer Vision}, pages 4958--4966, 2019.

\bibitem[Bubeck(2014)]{bubeck2014convex}
S{\'e}bastien Bubeck.
\newblock Convex optimization: Algorithms and complexity.
\newblock \emph{arXiv preprint arXiv:1405.4980}, 2014.

\bibitem[Carlini and Wagner(2017)]{carlini2017towards}
Nicholas Carlini and David Wagner.
\newblock Towards evaluating the robustness of neural networks.
\newblock In \emph{2017 ieee symposium on security and privacy (sp)}, pages
  39--57. IEEE, 2017.

\bibitem[Chen et~al.(2017)Chen, Zhang, Sharma, Yi, and Hsieh]{chen2017zoo}
Pin-Yu Chen, Huan Zhang, Yash Sharma, Jinfeng Yi, and Cho-Jui Hsieh.
\newblock Zoo: Zeroth order optimization based black-box attacks to deep neural
  networks without training substitute models.
\newblock In \emph{Proceedings of the 10th ACM workshop on artificial
  intelligence and security}, pages 15--26, 2017.

\bibitem[Cheng et~al.(2019)Cheng, Dong, Pang, Su, and Zhu]{cheng2019improving}
Shuyu Cheng, Yinpeng Dong, Tianyu Pang, Hang Su, and Jun Zhu.
\newblock Improving black-box adversarial attacks with a transfer-based prior.
\newblock \emph{arXiv preprint arXiv:1906.06919}, 2019.

\bibitem[Choromanski et~al.(2018)Choromanski, Rowland, Sindhwani, Turner, and
  Weller]{choromanski2018structured}
Krzysztof Choromanski, Mark Rowland, Vikas Sindhwani, Richard Turner, and
  Adrian Weller.
\newblock Structured evolution with compact architectures for scalable policy
  optimization.
\newblock In \emph{International Conference on Machine Learning}, pages
  970--978. PMLR, 2018.

\bibitem[Golovin et~al.(2019)Golovin, Karro, Kochanski, Lee, Song, and
  Zhang]{golovin2019gradientless}
Daniel Golovin, John Karro, Greg Kochanski, Chansoo Lee, Xingyou Song, and
  Qiuyi Zhang.
\newblock Gradientless descent: High-dimensional zeroth-order optimization.
\newblock \emph{arXiv preprint arXiv:1911.06317}, 2019.

\bibitem[Hansen and Ostermeier(1996)]{hansen1996adapting}
Nikolaus Hansen and Andreas Ostermeier.
\newblock Adapting arbitrary normal mutation distributions in evolution
  strategies: The covariance matrix adaptation.
\newblock In \emph{Proceedings of IEEE international conference on evolutionary
  computation}, pages 312--317. IEEE, 1996.

\bibitem[Hansen et~al.(2009)Hansen, Finck, Ros, and Auger]{hansen2009real}
Nikolaus Hansen, Steffen Finck, Raymond Ros, and Anne Auger.
\newblock \emph{Real-parameter black-box optimization benchmarking 2009:
  Noiseless functions definitions}.
\newblock PhD thesis, INRIA, 2009.

\bibitem[Heijmans(1999)]{heijmans1999does}
Risto Heijmans.
\newblock When does the expectation of a ratio equal the ratio of expectations?
\newblock \emph{Statistical Papers}, 40\penalty0 (1):\penalty0 107--115, 1999.

\bibitem[Ilyas et~al.(2018{\natexlab{a}})Ilyas, Engstrom, Athalye, and
  Lin]{ilyas2018black}
Andrew Ilyas, Logan Engstrom, Anish Athalye, and Jessy Lin.
\newblock Black-box adversarial attacks with limited queries and information.
\newblock In \emph{International Conference on Machine Learning}, pages
  2137--2146. PMLR, 2018{\natexlab{a}}.

\bibitem[Ilyas et~al.(2018{\natexlab{b}})Ilyas, Engstrom, and
  Madry]{ilyas2018prior}
Andrew Ilyas, Logan Engstrom, and Aleksander Madry.
\newblock Prior convictions: Black-box adversarial attacks with bandits and
  priors.
\newblock \emph{arXiv preprint arXiv:1807.07978}, 2018{\natexlab{b}}.

\bibitem[Karimi et~al.(2016)Karimi, Nutini, and Schmidt]{karimi2016linear}
Hamed Karimi, Julie Nutini, and Mark Schmidt.
\newblock Linear convergence of gradient and proximal-gradient methods under
  the polyak-{\l}ojasiewicz condition.
\newblock In \emph{Joint European Conference on Machine Learning and Knowledge
  Discovery in Databases}, pages 795--811. Springer, 2016.

\bibitem[Kozak et~al.(2021)Kozak, Becker, Doostan, and
  Tenorio]{kozak2021stochastic}
David Kozak, Stephen Becker, Alireza Doostan, and Luis Tenorio.
\newblock A stochastic subspace approach to gradient-free optimization in high
  dimensions.
\newblock \emph{Computational Optimization and Applications}, 79\penalty0
  (2):\penalty0 339--368, 2021.

\bibitem[Lojasiewicz(1963)]{lojasiewicz1963topological}
Stanislaw Lojasiewicz.
\newblock A topological property of real analytic subsets.
\newblock \emph{Coll. du CNRS, Les {\'e}quations aux d{\'e}riv{\'e}es
  partielles}, 117:\penalty0 87--89, 1963.

\bibitem[Maheswaranathan et~al.(2019)Maheswaranathan, Metz, Tucker, Choi, and
  Sohl-Dickstein]{maheswaranathan2019guided}
Niru Maheswaranathan, Luke Metz, George Tucker, Dami Choi, and Jascha
  Sohl-Dickstein.
\newblock Guided evolutionary strategies: Augmenting random search with
  surrogate gradients.
\newblock In \emph{International Conference on Machine Learning}, pages
  4264--4273. PMLR, 2019.

\bibitem[Mania et~al.(2018)Mania, Guy, and Recht]{mania2018simple}
Horia Mania, Aurelia Guy, and Benjamin Recht.
\newblock Simple random search of static linear policies is competitive for
  reinforcement learning.
\newblock In \emph{Proceedings of the 32nd International Conference on Neural
  Information Processing Systems}, pages 1805--1814, 2018.

\bibitem[Matyas(1965)]{matyas1965random}
J~Matyas.
\newblock Random optimization.
\newblock \emph{Automation and Remote control}, 26\penalty0 (2):\penalty0
  246--253, 1965.

\bibitem[Meier et~al.(2019)Meier, Mujika, Gauy, and Steger]{meier2019improving}
Florian Meier, Asier Mujika, Marcelo~Matheus Gauy, and Angelika Steger.
\newblock Improving gradient estimation in evolutionary strategies with past
  descent directions.
\newblock \emph{arXiv preprint arXiv:1910.05268}, 2019.

\bibitem[Mortici(2010)]{mortici2010new}
Cristinel Mortici.
\newblock New approximation formulas for evaluating the ratio of gamma
  functions.
\newblock \emph{Mathematical and Computer Modelling}, 52\penalty0
  (1-2):\penalty0 425--433, 2010.

\bibitem[Nesterov(2012)]{nesterov2012efficiency}
Yu~Nesterov.
\newblock Efficiency of coordinate descent methods on huge-scale optimization
  problems.
\newblock \emph{SIAM Journal on Optimization}, 22\penalty0 (2):\penalty0
  341--362, 2012.

\bibitem[Nesterov and Spokoiny(2017)]{nesterov2017random}
Yurii Nesterov and Vladimir Spokoiny.
\newblock Random gradient-free minimization of convex functions.
\newblock \emph{Foundations of Computational Mathematics}, 17\penalty0
  (2):\penalty0 527--566, 2017.

\bibitem[O’donoghue and Candes(2015)]{o2015adaptive}
Brendan O’donoghue and Emmanuel Candes.
\newblock Adaptive restart for accelerated gradient schemes.
\newblock \emph{Foundations of computational mathematics}, 15\penalty0
  (3):\penalty0 715--732, 2015.

\bibitem[Polyak(1963)]{polyak1963gradient}
Boris~Teodorovich Polyak.
\newblock Gradient methods for minimizing functionals.
\newblock \emph{Zhurnal Vychislitel'noi Matematiki i Matematicheskoi Fiziki},
  3\penalty0 (4):\penalty0 643--653, 1963.

\bibitem[Salimans et~al.(2017)Salimans, Ho, Chen, Sidor, and
  Sutskever]{salimans2017evolution}
Tim Salimans, Jonathan Ho, Xi~Chen, Szymon Sidor, and Ilya Sutskever.
\newblock Evolution strategies as a scalable alternative to reinforcement
  learning.
\newblock \emph{arXiv preprint arXiv:1703.03864}, 2017.

\bibitem[Snoek et~al.(2012)Snoek, Larochelle, and Adams]{snoek2012practical}
Jasper Snoek, Hugo Larochelle, and Ryan~P Adams.
\newblock Practical bayesian optimization of machine learning algorithms.
\newblock \emph{arXiv preprint arXiv:1206.2944}, 2012.

\bibitem[Stich et~al.(2013)Stich, Muller, and Gartner]{stich2013optimization}
Sebastian~U Stich, Christian~L Muller, and Bernd Gartner.
\newblock Optimization of convex functions with random pursuit.
\newblock \emph{SIAM Journal on Optimization}, 23\penalty0 (2):\penalty0
  1284--1309, 2013.

\bibitem[Tu et~al.(2019)Tu, Ting, Chen, Liu, Zhang, Yi, Hsieh, and
  Cheng]{tu2019autozoom}
Chun-Chen Tu, Paishun Ting, Pin-Yu Chen, Sijia Liu, Huan Zhang, Jinfeng Yi,
  Cho-Jui Hsieh, and Shin-Ming Cheng.
\newblock Autozoom: Autoencoder-based zeroth order optimization method for
  attacking black-box neural networks.
\newblock In \emph{Proceedings of the AAAI Conference on Artificial
  Intelligence}, volume~33, pages 742--749, 2019.

\bibitem[Vershynin(2018)]{vershynin2018high}
Roman Vershynin.
\newblock \emph{High-dimensional probability: An introduction with applications
  in data science}, volume~47.
\newblock Cambridge university press, 2018.

\end{thebibliography}

\clearpage
\appendix

\section{Supplemental materials for Section~\ref{sec:2}}
\label{sec:A}

\subsection{Definitions in convex optimization}
\label{sec:A1}
\begin{definition}[Convexity]
A differentiable function $f$ is convex if for every $x, y\in \mathbb{R}^d$,
\begin{align*}
    f(y)\geq f(x)+\nabla f(x)^\top (y - x).
\end{align*}
\end{definition}
\begin{definition}[Smoothness]
A differentiable function $f$ is $L$-smooth for some positive constant $L$ if its gradient is $L$-Lipschitz; namely, for every $x, y\in \mathbb{R}^d$, we have
    \begin{align*}
        \| \nabla f(x) - \nabla f(y) \| \leq L \| x - y \|.
    \end{align*}
\end{definition}
\begin{corollary}
If $f$ is $L$-smooth, then for every $x, y \in \mathbb R^d$,
\begin{align}
\label{eq:cor-smooth}
|f(y)-f(x)-\nabla f(x)^\top (y-x)|\leq \frac{1}{2}L\|y-x\|^2.
\end{align}
\end{corollary}
\begin{proof}
See Lemma~3.4 in \cite{bubeck2014convex}.
\end{proof}
\begin{definition}[Strong convexity] 
A differentiable function $f$ is $\tau$-strongly convex for some positive constant $\tau$, if for all $x, y\in \mathbb{R}^d$,
\begin{align*}
    f(y)\geq f(x)+\nabla f(x)^\top (y-x)+\frac{\tau}{2}\|y-x\|^2.
\end{align*}
\end{definition}

\subsection{Proof of Proposition~\ref{prop:finite-error}}
\label{sec:A2}
\begin{proposition_main}
If $f$ is $L$-smooth, then for any $(x,v)$ with $\|v\|=1$, $|g_\mu(v;x)-\nabla f(x)^\top v|\leq \frac{1}{2}L\mu$.
\end{proposition_main}
\begin{proof}
In Eq.~\eqref{eq:cor-smooth}, setting $y=x+\mu v$ and dividing both sides by $\mu$, we complete the proof.
\end{proof}

\section{Supplemental materials for Section~\ref{sec:3}}
\label{sec:B}
\subsection{Proofs of Lemma~\ref{lem:single-progress}, Theorem~\ref{thm:smooth} and Theorem~\ref{thm:strong}}
\label{sec:B1}
\begin{lemma_main}
\label{lem_main:single-progress}
    Let $C_t:=\left(\overline{\nabla f(x_t)}^\top v_t\right)^2$ and  $L':=\frac{L}{1-\left (1-\frac{L}{\hat{L}}\right)^2}$, then in Algorithm~\ref{alg:greedy-descent}, we have
    \begin{align}
    \label{eq_main:single-progress}
        \E_t[\delta_{t+1}] \leq \delta_t -\frac{\E_t[C_t]}{2L'}\|\nabla f(x_t)\|^2.
    \end{align}
\end{lemma_main}
\begin{proof}
We have
\begin{align}
    \delta_{t+1} - \delta_t &= f(x_{t+1}) - f(x_t) \\
    &= f\left(x_t - \frac{1}{\hat{L}}\nabla f(x_t)^\top v_t \cdot v_t\right) - f(x_t) \\
    &\leq -\frac{1}{\hat{L}} \left(\nabla f(x_t)^\top v_t\right)^2 + \frac{1}{2}L \cdot \left(\frac{1}{\hat{L}} \nabla f(x_t)^\top v_t\right)^2 \\
    &= -\frac{1}{2L'}\left(\nabla f(x_t)^\top v_t\right)^2.
\end{align}
Hence,
\begin{align}
\label{eq:value_diff}
    \E_t[\delta_{t+1}]-\delta_t \leq -\frac{1}{2L'}\E_t\left[\left(\nabla f(x_t)^\top v_t\right)^2\right] = -\frac{\E_t[C_t]}{2L'}\|\nabla f(x_t)\|^2,
\end{align}
where the last equality holds because $x_t$ is $\mathcal{F}_{t-1}$-measurable since we require $\mathcal{F}_{t-1}$ to include all the randomness before iteration $t$ in Remark~\ref{rem:dependence-history} (so $\|\nabla f(x_t)\|^2$ is also $\mathcal{F}_{t-1}$-measurable).
\end{proof}
\begin{remark}
The proof actually does not require $f$ to be convex. It only requires $f$ to be $L$-smooth.
\end{remark}
\begin{remark}
\label{rem:decreasing}
From the proof we see that $\delta_{t+1}-\delta_t\leq 0$, so $f(x_{t+1})\leq f(x_t)$. Hence, the sequence $\{f(x_t)\}_{t\geq 0}$ is non-increasing in Algorithm~\ref{alg:greedy-descent}.
\end{remark}
\begin{theorem_main}[Algorithm~\ref{alg:greedy-descent}, smooth and convex]
\label{thm_main:smooth}
Let $R:=\max_{x: f(x)\leq f(x_0)}\|x-x^*\|$ and suppose $R<\infty$. Then, 
in Algorithm~\ref{alg:greedy-descent}, we have
\begin{align}
\label{eq_main:smooth}
    \E[\delta_T]\leq \frac{2L' R^2\sum_{t=0}^{T-1} \E\left[\frac{1}{\E_t[C_t]}\right]}{T(T+1)}.
\end{align}
\end{theorem_main}
\begin{proof}
Since $f$ is convex, we have
\begin{align}
    \delta_t = f(x_t)-f(x^*)\leq \nabla f(x_t)^\top (x_t - x^*) \leq \|\nabla f(x_t)\|\cdot\|x_t-x^*\| \leq R\|\nabla f(x_t)\|,
\end{align}
where the last inequality follows from the definition of $R$ and the fact that $f(x_t)\leq f(x_0)$ (since $\delta_{t+1}\leq \delta_t$ for all $t$). The following proof is adapted from the proof of Theorem 3.2 in~\cite{bansal2019potential}. Define $\Phi_t:=t(t+1)\delta_t$. By Lemma~\ref{lem:single-progress}, we have 
\begin{align}
    \E_t[\Phi_{t+1}] - \Phi_t &= (t+1)(t+2)\E_t[\delta_{t+1}] - t(t+1)\delta_t \\
    &= (t+1)(t+2)(\E_t[\delta_{t+1}] - \delta_t) + 2(t+1)\delta_t \\
    &\leq -(t+1)(t+2)\frac{\E_t[C_t]}{2L'}\|\nabla f(x_t)\|^2 + 2(t+1)R\|\nabla f(x_t)\|  \\
    &\leq \frac{(2(t+1)R)^2}{4(t+1)(t+2)\frac{\E_t[C_t]}{2L'}}\label{app:proof-lem-quad-sup} \\
    &=\frac{2L'(t+1)R^2}{(t+2)\E_t[C_t]} \\
    &\leq\frac{2L' R^2}{\E_t[C_t]},
\end{align}
where Eq.~\eqref{app:proof-lem-quad-sup} follows from the fact that $-at^2 + bt \leq \frac{b^2}{4a}$ for $a > 0$.
Hence
\begin{align}
    \E[\Phi_{t+1}] - \E[\Phi_t] = \E[\E_t[\Phi_{t+1}] - \Phi_t] \leq 2L' R^2\E\left[\frac{1}{\E_t[C_t]}\right].
\end{align}
Since $\Phi_0=0$, we have $\E[\Phi_T]\leq 2L' R^2 \sum_{t=0}^{T-1} \E\left[\frac{1}{\E_t[C_t]}\right]$. Therefore,
\begin{align}
    \E[\delta_T] = \frac{\E[\Phi_T]}{T(T+1)} \leq \frac{2L' R^2 \sum_{t=0}^{T-1} \E\left[\frac{1}{\E_t[C_t]}\right]}{T(T+1)}.
\end{align}
\end{proof}
\begin{remark}
\label{rem:initial-random}
By inspecting the proof, we note that Theorem~\ref{thm_main:smooth} still holds if we replace the fixed initialization $x_0$ in Algorithm~\ref{alg:greedy-descent} with a random initialization $x_0'$ for which $f(x_0')\leq f(x_0)$ always holds.
We formally summarize this in the following proposition. This proposition will be useful in the proof of Theorem~\ref{thm:prgf-smooth}.
\begin{proposition}
\label{prop:smooth-random}
Let $x_{\mathrm{fix}}$ be a fixed vector, $R:=\max_{x: f(x)\leq f(x_{\mathrm{fix}})}\|x-x^*\|$ and suppose $R<\infty$. Then, in Algorithm~\ref{alg:greedy-descent}, using a random initialization $x_0$, if $f(x_0)\leq f(x_{\mathrm{fix}})$ always hold, we have
\begin{align}
\label{eq:smooth-random}
    \E[\delta_T]\leq \frac{2L' R^2\sum_{t=0}^{T-1} \E\left[\frac{1}{\E_t[C_t]}\right]}{T(T+1)}.
\end{align}
\end{proposition}
\begin{proof}
By Remark~\ref{rem:decreasing}, $f(x_t)\leq f(x_0)$. We note that $\|x_t-x^*\|\leq R$ since $f(x_t)\leq f(x_0)\leq f(x_{\mathrm{fix}})$. The remaining proof is the same as the proof of Theorem~\ref{thm_main:smooth}.
\end{proof}
\end{remark}
Next we state the proof regarding the convergence guarantee of Algorithm~\ref{alg:greedy-descent} under smooth and strongly convex case.
\begin{theorem_main}[Algorithm~\ref{alg:greedy-descent}, smooth and strongly convex]
\label{thm_main:strong}
In Algorithm~\ref{alg:greedy-descent}, if we further assume that $f$ is $\tau$-strongly convex, then we have
\begin{align}
\label{eq_main:strong}
    \E\left[\frac{\delta_T}{\exp\left(-\frac{\tau}{L'}\sum_{t=0}^{T-1}\E_t[C_t]\right)}\right] \leq \delta_0.
\end{align}
\end{theorem_main}
\begin{proof}
Since $f$ is $\tau$-strongly convex, we have
\begin{align}
    \delta_t &= f(x_t)-f(x^*)\leq \nabla f(x_t)^\top (x_t - x^*)-\frac{\tau}{2}\|x_t-x^*\|^2 \\
    &\leq \|\nabla f(x_t)\|\cdot\|x_t-x^*\|-\frac{\tau}{2}\|x_t-x^*\|^2 \\
    &\leq \frac{\|\nabla f(x_t)\|^2}{2\tau}.
\end{align}
Therefore we have
\begin{align}
    \label{eq:bound_norm_strong}
    \|\nabla f(x_t)\|^2 \geq 2\tau\delta_t.
\end{align}
By Lemma~\ref{lem:single-progress} and Eq.~\eqref{eq:bound_norm_strong} we have
\begin{align}
    \E_t[\delta_{t+1}] \leq \delta_t-\frac{\E_t[C_t]\tau}{L'}\delta_t = \left(1-\frac{\tau}{L'}\E_t[C_t]\right)\delta_t.
\end{align}
Let $\alpha_t:=\frac{\tau}{L'}\E_t[C_t]$, then $\E_t[\delta_{t+1}]\leq(1-\alpha_t)\delta_t$. We have 
\begin{align*}
    \delta_0&=\E[\delta_0]\geq \E\left[\frac{1}{1-\alpha_0}\E_0[\delta_1]\right]=\E\left[\E_0\left[\frac{\delta_1}{1-\alpha_0}\right]\right]=\E\left[\frac{\delta_1}{1-\alpha_0}\right] \\
    &\geq \E\left[\frac{\E_1[\delta_2]}{(1-\alpha_0)(1-\alpha_1)}\right]=\E\left[\E_1\left[\frac{\delta_2}{(1-\alpha_0)(1-\alpha_1)}\right]\right]=\E\left[\frac{\delta_2}{(1-\alpha_0)(1-\alpha_1)}\right] \\
    &\geq \ldots \\
    &\geq \E\left[\frac{\delta_T}{\prod_{t=0}^{T-1}(1-\alpha_t)}\right].
\end{align*}

Since $\exp(-x)\geq 1-x \geq 0$ when $0\leq x \leq 1$, the proof is completed.
\end{proof}
\begin{remark}
Indeed, the proof does not require $f$ to be strongly convex or convex. It only requires the Polyak-\L{}ojasiewicz condition (Eq.~\eqref{eq:bound_norm_strong}) which is weaker than strong convexity \cite{polyak1963gradient, lojasiewicz1963topological, karimi2016linear}.
\end{remark}
\subsection{Proof of Proposition~\ref{prop:optimality}}
\label{sec:B2}
We note that in Algorithm~\ref{alg:greedy-descent}, $\|v_t\|=1$. If $v_t\in A$, then we have the following lemma by Proposition~1 in \cite{meier2019improving}.
\begin{lemma}
\label{lem:subspace}
Let $u_1,u_2,\ldots,u_q$ be $q$ fixed vectors in $\mathbb{R}^d$ and $A:=\operatorname{span}\{u_1,u_2,\ldots,u_q\}$ be the subspace spanned by $u_1,u_2,\ldots,u_q$. Let $\nabla f(x_t)_A$ denote the projection of $\nabla f(x_t)$ onto $A$, then $\overline{\nabla f(x_t)_A}=\operatorname{argmax}_{v_t\in A, \|v_t\|=1}C_t$.
\end{lemma}
We further note that $\nabla f(x_t)_A$ could be calculated with the values of $\{\nabla f(x_t)^\top u_i\}_{i=1}^q$:
\begin{lemma}
\label{lem:computable}
Let $A:=\operatorname{span}\{u_1,u_2,\ldots,u_q\}$ be the subspace spanned by $u_1,u_2,\ldots,u_q$, and suppose $\{u_1,u_2,\ldots,u_q\}$ is linearly independent (if they are not, then we choose a subset of these vectors which is linearly independent). Then $\nabla f(x_t)_A=\sum_{i=1}^q a_i u_i$, where $a:=(a_1,a_2,\cdots,a_q)^\top$ is given by $a=G^{-1}b$, where $G$ is a $q\times q$ matrix in which $G_{ij}=u_i^\top u_j$, $b$ is a $q$-dimensional vector in which $b_i=\nabla f(x_t)^\top u_i$.
\end{lemma}
\begin{proof}
Since $\nabla f(x_t)_A\in \operatorname{span}\{u_1,u_2,\ldots,u_q\}$, there exists $a\in \mathbb{R}^q$ such that $\nabla f(x_t)_A=\sum_{i=1}^q a_i u_i$. Since $\nabla f(x_t)_A$ is the projection of $\nabla f(x_t)$ onto $A$ and $u_1,u_2,\ldots,u_q\in A$, $\nabla f(x_t)_A^\top u_i=\nabla f(x_t)^\top u_i$ holds for any $i$. Therefore, $Ga=b$. Since $\{u_1,u_2,\ldots,u_q\}$ is linearly independent and $G$ is corresponding Gram matrix, $G$ is invertible. Hence $a=G^{-1}b$.
\end{proof}
Therefore, if we suppose $v_t\in A$, then the optimal $v_t$ is given by $\overline{\nabla f(x_t)_A}$, which could be calculated from $\{\nabla f(x_t)^\top u_i\}_{i=1}^q$. Now we are ready to prove Proposition~\ref{prop:optimality} through an additional justification.
\begin{proposition_main}[Optimality of subspace estimator]
\label{prop_main:optimality}
In one iteration of Algorithm~\ref{alg:greedy-descent}, if we have queried $\{\nabla f(x_t)^\top u_i\}_{i=1}^q$, then the optimal $v_t$ maximizing $C_t$ s.t. $\|v_t\|=1$ should be in the following form: $v_t=\overline{\nabla f(x_t)_A}$, where $A:=\operatorname{span}\{u_1,u_2,\ldots,u_q\}$.
\end{proposition_main}
\begin{proof}
It remains to justify the assumption that $v_t\in A$. We note that in Line~\ref{lne:3-alg1} of Algorithm~\ref{alg:greedy-descent}, generally it requires 1 additional call to query the value of $\nabla f(x_t)^\top v_t$, but if $v_t\in A$, then we can always save this query by calculating $\nabla f(x_t)^\top v_t$ with the values of $\{\nabla f(x_t)^\top u_i\}_{i=1}^q$, since if $v_t\in A$, then we can write $v_t$ in the form $v_t=\sum_{i=1}^q a_i u_i$, and hence $\nabla f(x_t)^\top v_t=\sum_{i=1}^q a_i \nabla f(x_t)^\top u_i$. Now suppose we finally sample a $v_t\notin A$. Then this additional query of $\nabla f(x_t)^\top v_t$ is necessary. Now we could let $A':=\operatorname{span}\{u_1,u_2,\ldots,u_q,v_t\}$ and calculate $v_t'=\overline{\nabla f(x_t)_{A'}}$. Obviously, $(\overline{\nabla f(x_t)}^\top v_t')^2\geq (\overline{\nabla f(x_t)}^\top v_t)^2$, suggesting that $v_t'$ is better than $v_t$. Therefore, without loss of generality we can always assume $v_t\in A$, and by Lemma~\ref{lem:subspace} the proof is complete.
\end{proof}
\subsection{Details regarding RGF and PRGF estimators}
\label{sec:B3}
\subsubsection{Construction of RGF estimator}
In Example~1, we mentioned that the RGF estimator is given by $v_t=\overline{\nabla f(x_t)}_A$ where $A=\operatorname{span}\{u_1,u_2,\ldots,u_q\}$ ($q>0$) and $\forall i, u_i\sim \U(\Sp_{d-1})$ ($u\sim \U(\Sp_{d-1})$ means that $u$ is sampled uniformly from the $(d-1)$-dimensional unit sphere, as a normalized $d$-dimensional random vector), and $u_1,u_2,\ldots,u_q$ are sampled independently. Now we present the detailed expression of $v_t$ by explicitly orthogonalizing $\{u_1,u_2,\ldots,u_q\}$:
\begin{align*}
    u_1&\sim\U(\Sp_{d-1}); \\
    u_2&=\overline{(\I-u_1 u_1^\top) \xi_2}, \xi_2\sim \U(\Sp_{d-1}); \\
    u_3&=\overline{(\I-u_1 u_1^\top-u_2 u_2^\top) \xi_3}, \xi_3\sim \U(\Sp_{d-1}); \\
     &... \\
    u_q&=\overline{\left(\I-\sum_{i=1}^{q-1} u_i u_i^\top\right)\xi_q}, \xi_q\sim \U(\Sp_{d-1}).
\end{align*}
Then we let $v_t=\overline{\sum_{i=1}^q \nabla f(x_t)^\top u_i \cdot u_i}$. Since $g_t=\nabla f(x_t)^\top \overline{\nabla f(x_t)_A}\cdot \overline{\nabla f(x_t)_A}=\nabla f(x_t)_A$, we have $g_t=\sum_{i=1}^q \nabla f(x_t)^\top u_i \cdot u_i$. Therefore, when using the RGF estimator, each iteration in Algorithm~\ref{alg:greedy-descent} costs $q$ queries to the directional derivative oracle.

\subsubsection{Properties of RGF estimator}
\label{sec:rgf}
In this section we show that for RGF estimator with $q$ queries, $\E_t[C_t]=\frac{q}{d}$. We first state a simple proposition here.
\begin{proposition}
\label{prop:square}
If $v_t=\overline{\sum_{i=1}^q \nabla f(x_t)^\top u_i \cdot u_i}$ and $u_1, u_2, \ldots, u_q$ are orthonormal, then
\begin{align}
    \left(\overline{\nabla f(x_t)}^\top v_t\right)^2 = \sum_{i=1}^q\left(\overline{\nabla f(x_t)}^\top u_i\right)^2.
\end{align}
\end{proposition}
\begin{proof}
Since $\nabla f(x_t)^\top v_t\cdot v_t:=g_t=\sum_{i=1}^q \nabla f(x_t)^\top u_i \cdot u_i$, we have $\overline{\nabla f(x_t)}^\top v_t\cdot v_t=\sum_{i=1}^q \overline{\nabla f(x_t)}^\top u_i \cdot u_i$. Taking inner product with $\overline{\nabla f(x_t)}$ to both sides, we obtain the result.
\end{proof}
By Proposition~\ref{prop:square},
\begin{align}
\label{eq:rgf-Et}
    \E_t[C_t]=\E_t\left[\left(\overline{\nabla f(x_t)}^\top v_t\right)^2\right]=\sum_{i=1}^q\E_t\left[\left(\overline{\nabla f(x_t)}^\top u_i\right)^2\right]=\sum_{i=1}^q \left(\overline{\nabla f(x_t)}^\top\E_t[u_i u_i^\top]\overline{\nabla f(x_t)}\right).
\end{align}

In RGF, $u_i$ is independent of the history, so in this section we directly write $\E[u_i u_i^\top]$ instead of $\E_t[u_i u_i^\top]$.

For $i=1$, since $u_1\sim \U(\Sp_{d-1})$, we have $\E[u_1 u_1^\top]=\frac{\I}{d}$. (Explanation: the distribution of $u_1$ is symmetric, hence $\E[u_1 u_1^\top]$ should be something like $a\I$; since $\mathrm{Tr}(\E[u_1 u_1^\top])=\E[u_1^\top u_1]=1$, $a=\nicefrac{1}{\mathrm{Tr}(\I)}=\nicefrac{1}{d}$.)

For $i=2$, we have $\E[u_2 u_2^\top| u_1]=\frac{\I-u_1 u_1^\top}{d-1}$. (See Section A.2 in~\cite{cheng2019improving} for the proof.) Therefore, $\E[u_2 u_2^\top]=\E[\E[u_2 u_2^\top|u_1]]=\frac{\I-\E[u_1 u_1^\top]}{d-1}=\frac{\I}{d}$.

Then by induction, we have that $\forall 1\leq i\leq q$, $\E[u_i u_i^\top] = \frac{\I}{d}$. Hence by Eq.~\eqref{eq:rgf-Et}, $\E_t[C_t]=\frac{q}{d}$.


\subsubsection{Construction of PRGF estimator}
\label{sec:construction-prgf}

In Example~\ref{exp:new-prgf}, we mentioned that the PRGF estimator is given by $v_t=\overline{\nabla f(x_t)}_A$ where $A=\operatorname{span}\{p_t,u_1,u_2,\ldots,u_q\}$ ($q>0$), where $p_t$ is a vector corresponding to the prior message which is available at the beginning of iteration $t$, and $\forall i, u_i\sim \U(\Sp_{d-1})$ ($u_1,u_2,\ldots,u_q$ are sampled independently). Now we present the detailed expression of $v_t$ by explicitly orthogonalizing $\{p_t,u_1,u_2,\ldots,u_q\}$. We note that here we leave $p_t$ unchanged (we only normalize it, i.e. $p_t\leftarrow \frac{p_t}{\|p_t\|}$ if $\|p_t\|\neq 1$) and make $\{u_1,u_2,\ldots,u_q\}$ orthogonal to $p_t$. 
Specifically, given a positive integer $q\leq d-1$,
\begin{align*}
    u_1&=\overline{(\I-p_t p_t^\top) \xi_1}, \xi_1\sim \U(\Sp_{d-1}); \\
    u_2&=\overline{(\I-p_t p_t^\top-u_1 u_1^\top) \xi_2}, \xi_2\sim \U(\Sp_{d-1}); \\
    u_3&=\overline{(\I-p_t p_t^\top-u_1 u_1^\top-u_2 u_2^\top) \xi_3}, \xi_3\sim \U(\Sp_{d-1}); \\
     &... \\
    u_q&=\overline{\left(\I-p_t p_t^\top-\sum_{i=1}^{q-1} u_i u_i^\top\right)\xi_q}, \xi_q\sim \U(\Sp_{d-1}).
\end{align*}
Then we let $v_t=\overline{\nabla f(x_t)^\top p_t \cdot p_t+\sum_{i=1}^q \nabla f(x_t)^\top u_i \cdot u_i}$. Since $g_t=\nabla f(x_t)^\top \overline{\nabla f(x_t)_A}\cdot \overline{\nabla f(x_t)_A}=\nabla f(x_t)_A$, we have $g_t=\nabla f(x_t)^\top p_t \cdot p_t+\sum_{i=1}^q \nabla f(x_t)^\top u_i \cdot u_i$. Therefore, when using the PRGF estimator, each iteration in Algorithm~\ref{alg:greedy-descent} costs $q+1$ queries to the directional derivative oracle.

\subsubsection{Properties of PRGF estimator}
\label{sec:B34}
Here we prove Lemma~\ref{lem:expected-drift} in the main article (its proof appears in \cite{meier2019improving}; we prove it in our language here), but for later use we give a more useful formula here, which can derive Lemma~\ref{lem:expected-drift}. Let $D_t := \left(\overline{\nabla f(x_t)}^\top p_t\right)^2$. We have
\begin{proposition}
\label{prop:1}
For $t\geq 1$,
\begin{align}
\label{eq:CtDt}
    C_t = D_t + (1-D_t)\xi_t^2,
\end{align}
where $\xi_t^2:=\sum_{i=1}^q \xi_{ti}^2$, $\xi_{ti}:=\overline{\nabla f(x_t)_H}^\top u_i$\footnote{Note that in different iterations, $\{u_i\}$ are different. Hence here we explicitly show this dependency on $t$ in the subscript of $\xi$.} in which $e_H:=e-p_t p_t^\top e$ denotes the projection of the vector $e$ onto the $(d-1)$-dimensional subspace $H$, of which $p_t$ is a normal vector.
\end{proposition}
\begin{proof}
By Proposition~\ref{prop:square}, we have
\begin{align}
\label{eq:prop1-1}
    C_t=\left(\overline{\nabla f(x_t)}^\top v_t\right)^2=D_t+\sum_{i=1}^q \left(\overline{\nabla f(x_t)}^\top u_i\right)^2.
\end{align}
By the definition of $u_1, u_2, \ldots, u_q$, they are in the subspace $H$. Therefore
\begin{align}
\label{eq:prop1-2}
    \left(\overline{\nabla f(x_t)}^\top u_i\right)^2 = \left(\overline{\nabla f(x_t)}_H^\top u_i\right)^2 = \|\overline{\nabla f(x_t)}_H\|^2 \left(\overline{\nabla f(x_t)_H}^\top u_i\right)^2=(1-D_t)\left(\overline{\nabla f(x_t)_H}^\top u_i\right)^2.
\end{align}
By Eq.~\eqref{eq:prop1-1} and Eq.~\eqref{eq:prop1-2}, the proposition is proved.
\end{proof}
Next we state $\E_t[\xi_t^2]$, the conditional expectation of $\xi_t^2$ given the history $\mathcal{F}_{t-1}$. We can also derive it in the similar way as in Section~\ref{sec:rgf}, but for later use let us describe the distribution of $\xi_t^2$ in a more convenient way. We note that the conditional distribution of $u_i$ is the uniform distribution from the unit sphere in the $(d-1)$-dimensional subspace $H$. Since $\xi_{ti}:=\overline{\nabla f(x_t)_H}^\top u_i$, $\xi_{ti}$ is indeed the inner product between one fixed unit vector and one uniformly random sampled unit vector in $H$. Indeed, $\xi_t^2$ is equal to $\left\|\left(\overline{\nabla f(x_t)_H}\right)_{A'}\right\|^2$ where $A':=\operatorname{span}(u_1,u_2,\cdots,u_q)$ is a random $q$-dimensional subspace of $H$. Therefore, $\xi_t^2$ is equal to the squared norm of the projection of a fixed unit vector in $H$ to a random $q$-dimensional subspace of $H$. By the discussion in the proof of Lemma~5.3.2 in \cite{vershynin2018high}, we can view a random projection acting on a
fixed vector as a fixed projection acting on a random vector. 
Therefore, we state the following proposition.
\begin{proposition}
\label{prop:equivalence}
The conditional distribution of $\xi_t^2$ given $\mathcal{F}_{t-1}$ is the same as the distribution of $\sum_{i=1}^q z_i^2$, where $(z_1,z_2,\ldots,z_{d-1})^\top \sim \U(\Sp^{d-2})$, where $\Sp^{d-2}$ is the unit sphere in $\mathbb{R}^{d-1}$.
\end{proposition}
Then it is straightforward to prove the following proposition.
\begin{proposition}
\label{prop:expectation}
$\E_t[\xi_t^2]=\frac{q}{d-1}$.
\end{proposition}
\begin{proof}
By symmetry, $\E[z_i^2]=\E[z_j^2]$ $\forall i, j$. Since $\E[\sum_{i=1}^{d-1} z_i^2]=1$, $\E[z_i^2]=\frac{1}{d-1}$. Hence by Proposition~\ref{prop:equivalence}, $\E_t[\xi_t^2]=\E[\sum_{i=1}^q z_i^2]=\frac{q}{d-1}$.
\end{proof}
Now we reach Lemma~\ref{lem:expected-drift}.
\begin{lemma_main}
\label{lem_main:expected-drift}
In Algorithm~\ref{alg:greedy-descent} with PRGF estimator,
\begin{align}
\label{eq:expected-drift}
    \E_t[C_t]=D_t+\frac{q}{d-1}(1-D_t),
\end{align}
where $D_t:=\left(\overline{\nabla f(x_t)}^\top p_t\right)^2$.
\end{lemma_main}
\begin{proof}
Since $D_t$ is $\mathcal{F}_{t-1}$-measurable, by Proposition~\ref{prop:1} and Proposition~\ref{prop:expectation}, we have
\begin{align*}
    \E_t[C_t]=D_t+(1-D_t)\E_t[\xi_t^2]=D_t+\frac{q}{d-1}(1-D_t).
\end{align*}
\end{proof}

Finally, we note that Proposition~\ref{prop:equivalence} implies that $\xi_t^2$ is independent of the history (indeed, for all $i$, $\xi_{ti}^2$ is independent of the history). For convenience, in the following, when we need the conditional expectation (given some historical information) of quantities only related to $\xi_t^2$, we could directly write the expectation without conditioning. For example, we directly write $\E[\xi_t^2]$ instead of $\E_t[\xi_t^2]$, and write $\V[\xi_t^2]$ instead of the conditional variance $\V_t[\xi_t^2]$.

\subsection{Proof of Lemma~\ref{lem:decrease} and evolution of $\E[C_t]$}
\label{sec:B4}
In this section, we discuss the key properties of History-PRGF before presenting the theorems in Section~\ref{sec:theorems-history-prgf}. First we mention that while in History-PRGF we choose the prior $p_t$ to be $v_{t-1}$, we can choose $p_0$ as any fixed normalized vector. We first present a lemma which is useful for the proof of Lemma~\ref{lem:decrease}.

\begin{lemma}[Proof in Section~\ref{sec:proof_lem5_6}]
\label{lem:2}
Let $a$, $b$ and $c$ be vectors in $\mathbb{R}^d$, $\|a\|=\|c\|=1$, $B:=\{b:\|b-a\|\leq k\cdot a^\top c\}$, $0\leq k\leq 1$, $a^\top c\geq 0$. Then $\min_{b\in B} \overline{b}^\top c\geq \min_{b\in B} b^\top c = (1-k) a^\top c$.
\end{lemma}

\begin{lemma_main}
\label{lem_main:decrease}
In History-PRGF ($p_t=v_{t-1}$), we have
\begin{align}
    D_t\geq \left(1-\frac{L}{\hat{L}}\right)^2 C_{t-1}.
\end{align}
\end{lemma_main}
\begin{proof}

In History-PRGF $p_t=v_{t-1}$, so by the definitions of $D_t$ and $C_t$ we are going to prove
\begin{align}
    \left(\overline{\nabla f(x_t)} ^\top v_{t-1}\right)^2\geq \left(1-\frac{L}{\hat{L}}\right)^2 \left(\overline{\nabla f(x_{t-1})}^\top v_{t-1}\right)^2.
\end{align}

Without loss of generality, assume $\nabla f(x_{t-1})^\top v_{t-1}\geq 0$. Since $f$ is $L$-smooth, we have
\begin{align}
    \|\nabla f(x_t)-\nabla f(x_{t-1})\| \leq L \|x_t-x_{t-1}\| = \frac{L}{\hat{L}} \nabla f(x_{t-1})^\top v_{t-1},
\end{align}
which is equivalent to
\begin{align}
    \left\|\frac{\nabla f(x_t)}{\|\nabla f(x_{t-1})\|}-\overline{\nabla f(x_{t-1})}\right\| \leq \frac{L}{\hat{L}} \overline{\nabla f(x_{t-1})}^\top v_{t-1}.
\end{align}
Let $a=\overline{\nabla f(x_{t-1})}$, $b=\frac{\nabla f(x_t)}{\|\nabla f(x_{t-1})\|}$, $c=v_{t-1}$. By Lemma~\ref{lem:2} we have
\begin{align}
    \overline{\nabla f(x_t)} ^\top v_{t-1} \geq \left(1-\frac{L}{\hat{L}}\right) \overline{\nabla f(x_{t-1})}^\top v_{t-1}.
\end{align}
By the definition of $v_t$, the right-hand side is non-negative. Taking square on both sides, the proof is completed.
\end{proof}
When considering the lower bound related to $C_t$, we can replace the inequality with equality in Lemma~\ref{lem:decrease}. Therefore, by Proposition~\ref{prop:1} and Lemma~\ref{lem:decrease}, we now have full knowledge of evolution of $C_t$. We summarize the above discussion in the following proposition. We define $a':=\left(1-\frac{L}{\hat{L}}\right)^2$ in the following.
\begin{proposition}
Let $a':=\left(1-\frac{L}{\hat{L}}\right)^2$. Then in History-PRGF, we have
\begin{align}
\label{eq:iter}
    C_t\geq a' C_{t-1} + (1-a' C_{t-1})\xi_t^2.
\end{align}
\begin{proof}
By Proposition~\ref{prop:1}, $C_t=(1-\xi_t^2)D_t+\xi_t^2$. By Lemma~\ref{lem:decrease}, $D_t\geq a'C_{t-1}$. Since $\xi_t^2\leq 1$, we obtain the result.
\end{proof}
\end{proposition}

As an appetizer, we discuss the evolution of $\E[C_t]$ here using Lemma~\ref{lem:expected-drift} and Lemma~\ref{lem:decrease} in the following proposition.
\begin{proposition}
\label{prop:bound-expectation}
Suppose $\frac{q}{d-1}= k\frac{L}{\hat{L}}$ ($k>0$), then in History-PRGF, $\E[C_t]\geq (1-e^{-n})\frac{2}{2+k}\frac{q}{d-1}\frac{1}{1-a'}$ for $t\geq n\frac{d-1}{q}$.
\end{proposition}
\begin{proof}
By Eq.~\eqref{eq:iter}, we have
\begin{align}
    \E[C_t]&=\E[\E_t[C_t]]\geq\E[a'C_{t-1}+(1-a'C_{t-1})\E_t[\xi_t^2]] \\
    &=\E[a'C_{t-1}+(1-a'C_{t-1})\E[\xi_t^2]]\\
    &= \left(1-\frac{q}{d-1}\right)a' \E[C_{t-1}] + \frac{q}{d-1}.
\end{align}

Letting $a:=a'(1-\frac{q}{d-1})$, $b:=\frac{q}{d-1}$, then $\E[C_t]\geq a \E[C_{t-1}]+b$ and $0\leq a<1$. We have $\E[C_t]-\frac{b}{1-a} \geq a(\E[C_{t-1}]-\frac{b}{1-a}) \geq a^2(\E[C_{t-2}]-\frac{b}{1-a})\geq \ldots \geq a^t(\E[C_0]-\frac{b}{1-a})$, hence $\E[C_t]\geq\frac{b}{1-a} - a^t (\frac{b}{1-a} - \E[C_0])\geq(1-a^t)\frac{b}{1-a}$.

Since $1-a=1-(1-\frac{q}{d-1})(1-\frac{L}{\hat{L}})^2= 1-(1-k\frac{L}{\hat{L}})(1-\frac{L}{\hat{L}})^2$, noting that
\begin{align}
    \frac{1-(1-\frac{L}{\hat{L}})^2}{1-(1-k\frac{L}{\hat{L}})(1-\frac{L}{\hat{L}})^2}=\frac{\frac{L}{\hat{L}}+\frac{L}{\hat{L}}(1-\frac{L}{\hat{L}})}{\frac{L}{\hat{L}}+\frac{L}{\hat{L}}(1-\frac{L}{\hat{L}})+k\frac{L}{\hat{L}}(1-\frac{L}{\hat{L}})^2}\geq \frac{2}{2+k},
\end{align}
we have $\frac{1-a'}{1-a}\geq \frac{2}{2+k}$. Meanwhile, $a\leq 1-\frac{q}{d-1}$. Therefore, if $t\geq n\frac{d-1}{q}$, we have
\begin{align}
    a^t\leq (1-\frac{q}{d-1})^{n\frac{d-1}{q}}\leq \exp(-\frac{q}{d-1})^{n\frac{d-1}{q}}=e^{-n}.
\end{align}
Since $\frac{1-a'}{1-a}\geq \frac{2}{2+k}$ and $a^t\leq e^{-n}$, we have
\begin{align}
    \E[C_t]&\geq (1-a^t)\frac{b}{1-a}\geq \frac{2}{2+k}(1-e^{-n})\frac{1}{1-a'}\frac{q}{d-1}.
\end{align}
\end{proof}
\begin{corollary}
In History-PRGF, $\liminf_{t\to\infty}\E[C_t]\geq \frac{2}{2+k}\frac{q}{d-1}\frac{1}{1-a'}$.
\end{corollary}

Recalling that $L':=\frac{L}{1-a'}$, the propositions above tell us that $\E[C_t]$ tends to $O\left(\frac{q}{d}\frac{L'}{L}\right)$ in a fast rate, as long as $k$ is small, e.g. when $\frac{q}{d}\leq \frac{L}{\hat{L}}$ (which means that the chosen learning rate $\frac{1}{\hat{L}}$ is not too small compared with the optimal learning rate $\frac{1}{L}$: $\frac{1}{\hat{L}}\geq \frac{q}{d}\frac{1}{L}$). If $\E[C_t]\approx \frac{q}{d}\frac{L'}{L}$, then $\frac{\E[C_t]}{L'}$ is not dependent on $L'$ (and thus independent of $\hat{L}$). By Lemma~\ref{lem:single-progress}, Theorem~\ref{thm:smooth} and Theorem~\ref{thm:strong}, this roughly means that the convergence rate is robust to the choice of $\hat{L}$, i.e. robust to the choice of learning rate. Specifically, History-PRGF with $\hat{L}>L$ (but $\hat{L}$ is not too large) could roughly recover the performance of RGF with $\hat{L}=L$, since $\frac{\E[C_t]}{L'}\approx \frac{\frac{q}{d}}{L}$ where $\frac{q}{d}$ is the value of $\E_t[C_t]$ when using the RGF estimator.

\subsubsection{Proof of Lemma~\ref{lem:2}}
\label{sec:proof_lem5_6}
In this section, we first give a lemma for the proof of Lemma~\ref{lem:2}.

\begin{lemma}
\label{lem:1}
Let $a$ and $b$ be vectors in $\mathbb{R}^d$, $\|a\|=1$, $\|b\|\geq 1$. Then $\|\overline{b}-a\|\leq \|b-a\|$.
\end{lemma}

\begin{proof}
    \begin{align}
        \|b-a\|^2-\|\overline{b}-a\|^2&=\|b-\overline{b}\|^2 + 2(b-\overline{b})^\top (\overline{b}-a) \\
        &\geq 2(b-\overline{b})^\top (\overline{b}-a) \\
        &= 2(\|b\|-1)\overline{b}^\top (\overline{b}-a) \\
        &=2(\|b\|-1)(1-\overline{b}^\top a) \\
        &\geq 0.
    \end{align}
\end{proof}
    
Then, the detailed proof of \textbf{Lemma~\ref{lem:2}} is as follows.
\begin{proof}
    $\forall b\in B$, $b^\top c = a^\top c - (a-b)^\top c \geq a^\top c - \|a-b\|\|c\|\geq (1-k)a^\top c$, and both equality holds when $b=a-k\cdot a^\top c \cdot c$.
    
    \textbf{Case 1: $\|b\|\geq 1$}\ \ \ \ By Lemma~\ref{lem:1} we have $\|\overline{b}-a\|\leq \|b-a\|$, hence if $b\in B$, then $\overline{b}\in B$, so when $\|b\|\geq 1$ we have $\overline{b}^\top c\geq \min_{b\in B} b^\top c$.
    
    \textbf{Case 2: $\|b\|< 1$}\ \ \ \ $\forall b\in B$, if $\|b\|\leq 1$, then $\overline{b}^\top c = \frac{b^\top c}{\|b\|} \geq b^\top c\geq \min_{b\in B} b^\top c$.
    
    The proof of the lemma is completed.
\end{proof}

\subsection{Proofs of Theorem~\ref{thm:prgf-smooth} and Theorem~\ref{thm:prgf-strong}}
\label{sec:B5}
\label{sec:theorems-history-prgf}
\subsubsection{Proof of Theorem~\ref{thm:prgf-smooth}}
\label{sec:B51}
\label{sec:theorems3-proof}

As mentioned above, we define $a':=\left(1-\frac{L}{\hat{L}}\right)^2$ to be used in the proofs. In the analysis, we first try to replace the inequality in Lemma~\ref{lem:decrease} with equality. To do that, similar to Eq.~\eqref{eq:iter}, we define $\{E_t\}_{t=0}^{T-1}$ as follows: $E_0=0$, and
\begin{align}
\label{eq:iterE-smooth}
    E_t=a' E_{t-1} + (1-a' E_{t-1})\xi_t^2,
\end{align}
where $\xi_t^2$ is defined in Proposition~\ref{prop:1}.

First, we give the following lemmas, which is useful for the proof of Theorem~\ref{thm:prgf-smooth}.

\begin{lemma}[Upper-bounded variance; proof in Section~\ref{sec:lemmas_for_thm3}]
\label{lem:variance-bound-smooth}
If $d\geq 4$, then $\forall t$, $\V[\E_t[E_t]]\leq \frac{1}{1-(a')^2}\frac{2q}{(d-1)^2}$.
\end{lemma}

\begin{lemma}[Lower-bounded expectation; proof in Section~\ref{sec:lemmas_for_thm3}]
\label{lem:expectation-bound-smooth}
If $\frac{q}{d-1}\leq \frac{L}{\hat{L}}$ and $t\geq \frac{d-1}{q}$, then
\begin{align}
    \E[E_t]\geq\frac{1}{2}\frac{1}{1-a'}\frac{q}{d-1}.
\end{align}
\end{lemma}

\begin{lemma}[Proof in Section~\ref{sec:lemmas_for_thm3}]
\label{lem:chebyshev}
If a random variable $X\geq B>0$ satisfies that $\E[X]\geq\mu B$, $\V[X]\leq(\sigma B)^2$, then
\begin{align}
    \E\left[\frac{1}{X}\right]\leq\frac{1}{\mu B}\left(\frac{4\sigma^2}{\mu}+2\right).
\end{align}
\end{lemma}

Then, we provide the proof of Theorem~\ref{thm:prgf-smooth} in the following.
\begin{theorem_main}[History-PRGF, smooth and convex]
\label{thm_main:prgf-smooth}
In the setting of Theorem~\ref{thm:smooth}, when using the History-PRGF estimator, assuming $d\geq 4$, $\frac{q}{d-1}\leq \frac{L}{\hat{L}} \leq 1$ and $T> \left\lceil\frac{d}{q}\right\rceil$ ($\lceil\cdot\rceil$ denotes the ceiling function), we have
\begin{align}
    \E[f(x_T)] - f(x^*)\leq \left(\frac{32}{q}+2\right)\frac{2L\frac{d}{q} R^2}{T-\left\lceil\frac{d}{q}\right\rceil+1}.
\end{align}
\end{theorem_main}
\begin{proof}
Since $E_0=0\leq C_0$, and if $E_{t-1}\leq C_{t-1}$, then
\begin{align}
    E_t&= a' E_{t-1} + (1-a' E_{t-1})\xi_t^2 \\
    &\leq a' C_{t-1} + (1-a' C_{t-1})\xi_t^2 \\
    &\leq C_t,
\end{align}
in which the first inequality is because $\xi_t^2\leq 1$ and the second inequality is due to Eq.~\eqref{eq:iter}. Therefore by mathematical induction we have that $\forall t, E_t\leq C_t$.

Next, if $d\geq 4$, $\frac{q}{d-1}\leq\frac{L}{\hat{L}}$ and $t\geq \frac{d-1}{q}$, by Lemma~\ref{lem:variance-bound-smooth} and Lemma~\ref{lem:expectation-bound-smooth}, if we set $B=\frac{q}{d-1}$, then $\E[\E_t[E_t]]=\E[E_t]\geq \frac{1}{2}\frac{1}{1-a'}B$, and $\V[\E_t[E_t]]\leq \frac{2}{q}\frac{1}{1-(a')^2}B^2$. Meanwhile, if $t\geq 1$, then $\E_t[E_t]=a'(1-\frac{q}{d-1})E_{t-1}+\frac{q}{d-1}\geq B$. Therefore, by Lemma~\ref{lem:chebyshev} we have
\begin{align}
    \E\left[\frac{1}{\E_t[E_t]}\right]&\leq \frac{1}{\frac{1}{2}\frac{1}{1-a'}\frac{q}{d-1}}\left(\frac{4\frac{2}{q}\frac{1}{1-(a')^2}}{\frac{1}{2}\frac{1}{1-a'}}+2\right) \\
    &=\frac{d-1}{q}(1-a')\left(\frac{32}{q}\frac{1-a'}{1-(a')^2}+2\right) \\
    &\leq\frac{d}{q}(1-a')\left(\frac{32}{q}+2\right) \\
    &=\frac{d}{q}\frac{L}{L'}\left(\frac{32}{q}+2\right).
\end{align}
Since $E_t\leq C_t$, $\E_t[E_t]\leq \E_t[C_t]$. Let $s:=\left\lceil\frac{d}{q}\right\rceil$, then $\forall t\geq s$, $\E\left[\frac{1}{\E_t[C_t]}\right]\leq \E\left[\frac{1}{\E_t[E_t]}\right]\leq \frac{d}{q}\frac{L}{L'}\left(\frac{32}{q}+2\right)$.
Now imagine that we run History-PRGF algorithm with $x_s$ as the random initialization in Algorithm~\ref{alg:greedy-descent}, and set $p_0$ to $v_{s-1}$. Then quantities in iteration $t$ (e.g. $x_t$, $v_t$, $C_t$) in the imaginary setting have the same distribution as quantities in iteration $t+s$ (e.g. $x_{t+s}$, $v_{t+s}$, $C_{t+s}$) in the original algorithm (indeed, the quantities before iteration $t$ in the imaginary setting have the same joint distribution as the quantities from iteration $s$ to iteration $t+s-1$ in the original algorithm), and $\mathcal{F}_{t-1}$ in the imaginary setting corresponds to $\mathcal{F}_{t+s-1}$ in the original algorithm. 
Now we apply Proposition~\ref{prop:smooth-random} to the imaginary setting, and we note that if we set $x_{\mathrm{fix}}$ to the original $x_0$, then the condition in Proposition~\ref{prop:smooth-random} holds (since by Remark~\ref{rem:decreasing}, $f(x_s)\leq f(x_0)$). Since quantities in iteration $t$ in the original algorithm correspond to quantities in iteration $t-s$ in the imaginary setting, Proposition~\ref{prop:smooth-random} tells us that if $T>s$, we have
\begin{align}
    \E[f(x_T)]-f(x^*)\leq \frac{2L' R^2\sum_{t=s}^{T-1}\E\left[\frac{1}{\E_t[C_t]}\right]}{(T-s)(T-s+1)},
\end{align}
so
\begin{align}
    \E[f(x_T)]-f(x^*)\leq \left(\frac{32}{q}+2\right)\frac{2L \frac{d}{q}R^2}{T-\left\lceil\frac{d}{q}\right\rceil+1}.
\end{align}
\end{proof}

\subsubsection{Proofs of Lemma~\ref{lem:variance-bound-smooth}, \ref{lem:expectation-bound-smooth} and \ref{lem:chebyshev}}
\label{sec:lemmas_for_thm3}

In this section, we first present \textbf{Lemma~\ref{lem:numeric}} for the proof of Lemma~\ref{lem:variance-bound-smooth}.

\begin{lemma}
\label{lem:numeric}
Suppose $d\geq 3$, then $\V[\xi_t^2]< \frac{2q}{(d-1)^2}$.
\end{lemma}
\begin{proof}
For convenience, denote $D:=d-1$ in the following. By Proposition~\ref{prop:equivalence}, the distribution of $\xi_t^2$ is the same as the distribution of $\sum_{i=1}^q z_i^2$, where $(z_1,\cdots,z_q)^\top\sim\U(\Sp^{D-1})$.
We note that the distribution of $z$ is the same as the distribution of $\frac{x}{\|x\|}$, where $x\sim\mathcal{N}(0,\mathbf{I})$. Therefore,
\begin{align}
    \E\left[\left(\sum_{i=1}^q z_i^2\right)^2\right]=\E\left[\left(\sum_{i=1}^q \frac{x_i^2}{\|x\|^2}\right)^2\right]=\E\left[\frac{\left(\sum_{i=1}^q x_i^2\right)^2}{\left(\sum_{i=1}^D x_i^2\right)^2}\right].
\end{align}
By Theorem~1 in \cite{heijmans1999does}, $\frac{\sum_{i=1}^q x_i^2}{\|x\|^2}$ and $\|x\|^2$ are independently distributed. Therefore, $\frac{\left(\sum_{i=1}^q x_i^2\right)^2}{\left(\sum_{i=1}^D x_i^2\right)^2}$ and $\left(\sum_{i=1}^D x_i^2\right)^2$ are independently distributed, which implies
\begin{align}
    \E\left[\frac{\left(\sum_{i=1}^q x_i^2\right)^2}{\left(\sum_{i=1}^D x_i^2\right)^2}\right]=\frac{\E\left[\left(\sum_{i=1}^q x_i^2\right)^2\right]}{\E\left[\left(\sum_{i=1}^D x_i^2\right)^2\right]}.
\end{align}
We note that $\sum_{i=1}^q x_i^2$ follows the chi-squared distribution with $q$ degrees of freedom. Therefore, $\E\left[\sum_{i=1}^q x_i^2\right]=q$, and $\V\left[\sum_{i=1}^q x_i^2\right]=2q$. Therefore, $\E\left[\left(\sum_{i=1}^q x_i^2\right)^2\right]=\E\left[\sum_{i=1}^q x_i^2\right]^2+\V\left[\sum_{i=1}^q x_i^2\right]=q(q+2)$. Hence
\begin{align}
    \V\left[\sum_{i=1}^q z_i^2\right]=\E\left[\left(\sum_{i=1}^q z_i^2\right)^2\right]-\E\left[\sum_{i=1}^q z_i^2\right]^2=\frac{q(q+2)}{D(D+2)}-\frac{q^2}{D^2}=\frac{2q(D-q)}{D^2(D+2)}<\frac{2q}{D^2}.
\end{align}
Since $D=d-1$, the proof is complete.
\end{proof}

Then, the detailed proof of \textbf{Lemma~\ref{lem:variance-bound-smooth}} is as follows.

\begin{proof}
By the law of total variance, using Proposition~\ref{prop:expectation} and Lemma~\ref{lem:numeric}, we have
\begin{align}
    \V[E_t]&=\E[\V[E_t|E_{t-1}]]+\V[\E[E_t|E_{t-1}]] \\
    &=\E[(1-a' E_{t-1})^2\V[\xi_t^2]]+\V[a'(1-\E[\xi_t^2])E_{t-1}] \\
    &= \V[\xi_t^2]\E[(1-a' E_{t-1})^2]+(a')^2(1-\E[\xi_t^2])^2\V[E_{t-1}] \\
    &= \V[\xi_t^2](\E[(1-a' E_{t-1})]^2+\V[1-a' E_{t-1}])+(a')^2(1-\E[\xi_t^2])^2\V[E_{t-1}] \\
    &= \V[\xi_t^2](\E[(1-a' E_{t-1})]^2+(a')^2\V[ E_{t-1}])+(a')^2(1-\E[\xi_t^2])^2\V[E_{t-1}] \\
    &= (a')^2(\V[\xi_t^2]+(1-\E[\xi_t^2])^2)\V[E_{t-1}]+\V[\xi_t^2]\E[(1-a' E_{t-1})]^2 \\
    &\leq (a')^2(\V[\xi_t^2]+(1-\E[\xi_t^2])^2)\V[E_{t-1}]+\V[\xi_t^2] \\
    &\leq (a')^2\left(\frac{2q}{(d-1)^2}+\left(1-\frac{q}{d-1}\right)^2\right)\V[E_{t-1}]+\frac{2q}{(d-1)^2}.
\end{align}
If $d\geq 4$, then $\frac{2q}{(d-1)^2}+(1-\frac{q}{d-1})^2=1-\frac{q}{(d-1)^2}(2(d-1)-q-2)\leq 1-\frac{q}{(d-1)^2}(d-q)\leq 1$. Therefore we have
\begin{align}
    \V[E_t]\leq (a')^2\V[E_{t-1}]+\frac{2q}{(d-1)^2}.
\end{align}
Letting $a:=(a')^2, b:=\frac{2q}{(d-1)^2}$, then $\V[E_t]\leq a \V[E_{t-1}]+b$ and $0\leq a<1$. We have $\V[E_t]-\frac{b}{1-a} \leq a(\V[E_{t-1}]-\frac{b}{1-a}) \leq a^2(\V[E_{t-2}]-\frac{b}{1-a})\leq \ldots \leq a^t(\V[E_0]-\frac{b}{1-a})$, hence $\V[E_t]\leq\frac{b}{1-a} - a^t (\frac{b}{1-a} - \V[E_0])=(1-a^t)\frac{b}{1-a}\leq\frac{b}{1-a}=\frac{1}{1-(a')^2}\frac{2q}{(d-1)^2}$.

Finally, since $\V[\E_t[E_t]]=\V[\E[E_t|E_{t-1}]]\leq \V[E_t]$, the proof is completed.
\end{proof}

The detailed proof of \textbf{Lemma~\ref{lem:expectation-bound-smooth}} is as follows.

\begin{proof}
Similar to the proof of Proposition~\ref{prop:bound-expectation}, letting $a:=a'(1-\frac{q}{d-1})$ and $b:=\frac{q}{d-1}$, then $\E[E_t]=(1-a^t)\frac{b}{1-a}$, and $\frac{1-a'}{1-a}\geq \frac{2}{3}$. Meanwhile, since $\frac{q}{d-1}\leq \frac{L}{\hat{L}}$, $a\leq (1-\frac{q}{d-1})^3$. Therefore, if $t\geq \frac{d-1}{q}$, we have
\begin{align}
    a^t\leq (1-\frac{q}{d-1})^{3\frac{d-1}{q}}\leq \exp(-\frac{q}{d-1})^{3\frac{d-1}{q}}=e^{-3}.
\end{align}
Since $\frac{1-a'}{1-a}\geq \frac{2}{3}$ and $a^t\leq e^{-3}$, we have
\begin{align}
    \E[E_t]=(1-a^t)\frac{b}{1-a}\geq \frac{2}{3}(1-e^{-3})\frac{1}{1-a'}\frac{q}{d-1}\geq \frac{1}{2}\frac{1}{1-a'}\frac{q}{d-1}.
\end{align}
\end{proof}

The detailed proof of \textbf{Lemma~\ref{lem:chebyshev}} is as follows.
\begin{proof}
By Chebyshev's Inequality, we have
\begin{align}
    \Pr(X<\frac{1}{2}\E[X])\leq \frac{\V[X]}{(\frac{1}{2}\E[X])^2}\leq\frac{4\sigma^2}{\mu^2}.
\end{align}
Hence
\begin{align}
    \E\left[\frac{1}{X}\right]\leq \frac{1}{B}\Pr\left(X<\frac{1}{2}\E[X]\right)+\frac{1}{\frac{1}{2}\E[X]}\leq \frac{1}{B}\left(\frac{4\sigma^2}{\mu^2}+\frac{2}{\mu}\right)=\frac{1}{\mu B}\left(\frac{4\sigma^2}{\mu}+2\right).
\end{align}
\end{proof}

\subsubsection{Proof of Theorem~\ref{thm:prgf-strong}}
\label{sec:B53}
As in Section~\ref{sec:theorems3-proof}, similar to Eq.~\eqref{eq:iter}, we define $\{E_t\}_{t=0}^{T-1}$ as follows: $E_0=0$, and
\begin{align}
\label{eq:iterE}
    E_t=a' E_{t-1} + (1-a' E_{t-1})\zeta_t^2,
\end{align}
where $\zeta_t^2:=\sum_{i=1}^q \zeta_{ti}^2$, $\zeta_{ti}^2:=\min\{\xi_{ti}^2, \frac{1}{d-1}\}$ and $\xi_{ti}$ is defined in Proposition~\ref{prop:1}. Here we use $\zeta_t^2$ instead of $\xi_t^2$ in Eq.~\eqref{eq:iterE} to obtain a tighter bound when using McDiarmid's inequality in the proof of Lemma~\ref{lem:prob-bound}.

Here we first give Lemma~\ref{lem:prob-bound} for the proof of Theorem~\ref{thm:prgf-strong}.
\begin{lemma}[Proof in Section~\ref{sec:lemmas_for_thm4}]
\label{lem:prob-bound}
$\Pr\left(\frac{1}{T}\sum_{t=0}^{T-1}\E_t[E_t]< 0.1\frac{q}{d}\frac{1}{1-a'}\right)\leq \exp(-0.02T)$.
\end{lemma}

Then, we provide the proof of Theorem~\ref{thm:prgf-strong} in the following.
\begin{theorem_main}[History-PRGF, smooth and strongly convex]
\label{thm_main:prgf-strong}
Under the same conditions as in Theorem~\ref{thm:strong} ($f$ is $\tau$-strongly convex), when using the History-PRGF estimator, assuming $d\geq 4$, $\frac{q}{d-1}\leq \frac{L}{\hat{L}} \leq 1$, $\frac{q}{d}\leq 0.2\frac{L}{\tau}$, and $T\geq 5\frac{d}{q}$, we have
\begin{align}
\label{eq_main:prgf-strong}
    \E[\delta_T]\leq 2\exp(-0.1\frac{q}{d}\frac{\tau}{L}T)\delta_0.
\end{align}
\end{theorem_main}
\begin{proof}
Since $E_0=0\leq C_0$, and if $E_{t-1}\leq C_{t-1}$, then
\begin{align}
    E_t&=a' E_{t-1} + (1-a' E_{t-1})\sum_{i=1}^q \zeta_{ti}^2 \\
    &\leq a' E_{t-1} + (1-a' E_{t-1})\sum_{i=1}^q \xi_{ti}^2 \\
    &= a' E_{t-1} + (1-a' E_{t-1})\xi_t^2 \\
    &\leq a' C_{t-1} + (1-a' C_{t-1})\xi_t^2 \\
    &\leq C_t,
\end{align}
in which the second inequality is due to that $\xi_t^2\leq 1$ and the third inequality is due to Eq.~\eqref{eq:iter}. Therefore by mathematical induction we have that $\forall t, E_t\leq C_t$.


Then, since $\forall t$, $E_t\leq C_t$, we have $\frac{1}{T}\sum_{t=0}^{T-1}\E_t[E_t]\leq \frac{1}{T}\sum_{t=0}^{T-1}\E_t[C_t]$. Therefore, by Lemma~\ref{lem:prob-bound} we have $\Pr\left(\frac{1}{T}\sum_{t=0}^{T-1}\E_t[C_t]< 0.1\frac{q}{d}\frac{1}{1-a'}\right)\leq \exp(-0.02T)$. Let $k_T=\exp\left(\frac{\tau}{L'}\sum_{t=0}^{T-1}\E_t[C_t]\right)$. Since $\frac{1}{1-a'}=\frac{L'}{L}$,
\begin{align}
    \Pr\left(k_T<\exp\left(0.1\frac{q}{d}\frac{\tau}{L}T\right)\right)\leq\exp(-0.02T).
\end{align}
Meanwhile, let $\delta_t:=f(x_t)-f(x^*)$, Theorem~\ref{thm:strong} tells us that
\begin{align}
    \E[\delta_T k_T]\leq\delta_0.
\end{align}
Noting that $\delta_0\geq \delta_T$, we have 
\begin{align}
    \delta_0&\geq \E[\delta_T k_T] \\
    &\geq \E[\delta_T k_T 1_{k_T\geq\exp\left(0.1\frac{q}{d}\frac{\tau}{L}T\right)}] \\
    &\geq \exp\left(0.1\frac{q}{d}\frac{\tau}{L}T\right)\E[\delta_T  1_{k_T\geq\exp\left(0.1\frac{q}{d}\frac{\tau}{L}T\right)}] \\
    &= \exp\left(0.1\frac{q}{d}\frac{\tau}{L}T\right)(\E[\delta_T]-\E[\delta_T  1_{k_T<\exp\left(0.1\frac{q}{d}\frac{\tau}{L}T\right)}]) \\
    &\geq \exp\left(0.1\frac{q}{d}\frac{\tau}{L}T\right)(\E[\delta_T]-\delta_0\E[1_{k_T<\exp\left(0.1\frac{q}{d}\frac{\tau}{L}T\right)}])\\
    &=\exp\left(0.1\frac{q}{d}\frac{\tau}{L}T\right)\left(\E[\delta_T]-\delta_0\Pr\left(k_T<\exp\left(0.1\frac{q}{d}\frac{\tau}{L}T\right)\right)\right)\\
    &\geq \exp\left(0.1\frac{q}{d}\frac{\tau}{L}T\right)(\E[\delta_T]-\delta_0 \exp(-0.02T)),
\end{align}
in which $1_B$ denotes the indicator function of the event $B$ ($1_B=1$ when $B$ happens and $1_B=0$ when $B$ does not happen). By rearranging we have
\begin{align}
    \E[\delta_T]\leq \exp\left(-0.1\frac{q}{d}\frac{\tau}{L}T\right)\delta_0+\exp(-0.02T)\delta_0.
\end{align}
When $\frac{q}{d}\leq 0.2\frac{L}{\tau}$, $0.1\frac{q}{d}\frac{\tau}{L}T\leq 0.02T$, and hence $\exp\left(-0.1\frac{q}{d}\frac{\tau}{L}T\right)\geq \exp(-0.02T)$. Therefore, $\E[\delta_T]\leq 2\exp\left(-0.1\frac{q}{d}\frac{\tau}{L}T\right)\delta_0$. 
The proof of the theorem is completed.
\end{proof}

\subsubsection{Proof of Lemma~\ref{lem:prob-bound}}
\label{sec:lemmas_for_thm4}

In this section, we first give two lemmas for the proof of Lemma~\ref{lem:prob-bound}.
\begin{lemma}
\label{lem:expectation-bound}
If $d\geq 4$, then $\E[\zeta_t^2]\geq0.3\frac{q}{d-1}$.
\end{lemma}
\begin{proof}
We note that the distribution of $u_i$ is the uniform distribution from the unit sphere in the $(d-1)$-dimensional subspace $A$. Since $\xi_{ti}:=\overline{\nabla f(x_t)_A}^\top u_i$, $\xi_{ti}$ is indeed the inner product between one fixed unit vector and one uniformly random sampled unit vector. Therefore, its distribution is the same as $z_1$, where $(z_1, z_2, \ldots, z_{d-1})$ are uniformly sampled from $\mathbb{S}^{d-2}$, i.e. the unit sphere in $\mathbb{R}^{d-1}$. Now it suffices to prove that $\E[\min\{z_1^2,\frac{1}{d-1}\}]\geq\frac{0.3}{d-1}$.

Let $p(\cdot)$ denote the probability density function of $z_1$. For convenience let $D:=d-1$. We have $p(0)=\frac{S_{D-2}}{S_{D-1}}$, where $S_{D-1}$ denotes the surface area of $\mathbb{S}_{D-1}$. Since $S_{D-1}=\frac{2\pi^\frac{D}{2}}{\Gamma(\frac{D}{2})}$ where $\Gamma(\cdot)$ is the Gamma function, and by \cite{mortici2010new} we have $\frac{\Gamma(\frac{D}{2})}{\Gamma(\frac{D-1}{2})}\leq \sqrt{\frac{D-1}{2}}$, we have $p(0)\leq \sqrt{\frac{D-1}{2\pi}}\leq \sqrt{\frac{d-1}{2\pi}}$.
Meanwhile, we have $p(x)=p(0)\cdot \frac{(\sqrt{1-x^2})^{D-2}}{\sqrt{1-x^2}}=p(0)\cdot (\sqrt{1-x^2})^{D-3}$.
If $d\geq 4$, then $D\geq 3$, and we have $\forall x\in[-1,1], p(0)\geq p(x)$. Therefore, 
\begin{align}
    \Pr\left(z_1^2\geq \frac{1}{d-1}\right)=1-\Pr\left(|z_1|< \frac{1}{\sqrt{d-1}}\right)\geq 1-\frac{2}{\sqrt{d-1}}p(0)=1-\sqrt{\frac{2}{\pi}}\geq 0.2.
\end{align}
Similarly we have
\begin{align}
    \Pr\left(z_1^2\geq \frac{0.25}{d-1}\right)=1-\Pr\left(|z_1|< \frac{0.5}{\sqrt{d-1}}\right)\geq 1-\frac{1}{\sqrt{d-1}}p(0)=1-\sqrt{\frac{1}{2\pi}}\geq 0.6.
\end{align}
Let ${z'_1}^2:=\min\{z_1^2,\frac{1}{d-1}\}$. Then $\Pr\left({z'_1}^2\geq \frac{1}{d-1}\right)\geq 0.2$ and $\Pr\left({z'_1}^2\geq \frac{0.25}{d-1}\right)\geq 0.6$. Then
\begin{align}
    \E[{z'_1}^2]&\geq \frac{1}{d-1}\Pr\left({z'_1}^2\geq \frac{1}{d-1}\right)+\frac{0.25}{d-1}\Pr\left(\frac{1}{d-1}\geq {z'_1}^2\geq \frac{0.25}{d-1}\right) \\
    &=\frac{0.75}{d-1}\Pr\left({z'_1}^2\geq \frac{1}{d-1}\right)+\frac{0.25}{d-1}\Pr\left({z'_1}^2\geq \frac{0.25}{d-1}\right) \\
    &\geq \frac{0.3}{d-1}.
\end{align}
Hence $\E[\zeta_{ti}^2]\geq \frac{0.3}{d-1}$. By the definition of $\zeta_t^2$ the lemma is proved.
\end{proof}

\begin{lemma}
\label{lem:total_expectation_bound}
If $\frac{q}{d-1}\leq \frac{L}{\hat{L}}$, $T\geq 5\frac{d}{q}$, then $\frac{1}{T}\sum_{t=0}^{T-1}\E[E_t]\geq 0.2\frac{\frac{q}{d-1}}{1-a'}$.
\end{lemma}
\begin{proof}
By Eq.~\eqref{eq:iterE} and Lemma~\ref{lem:expectation-bound}, we have $\E_t[E_t]\geq \left(1-0.3\frac{q}{d-1}\right)a' E_{t-1} + 0.3\frac{q}{d-1}$.
Taking expectation to both sides, we have
\begin{align}
    \E[E_t]\geq \left(1-0.3\frac{q}{d-1}\right)a' \E[E_{t-1}] + 0.3\frac{q}{d-1}.
\end{align}
Let $a:=\left(1-0.3\frac{q}{d-1}\right)a', b:=0.3\frac{q}{d-1}$, then $\E[E_t]\geq a \E[E_{t-1}]+b$ and $0\leq a<1$. We have $\E[E_t]-\frac{b}{1-a} \geq a(\E[E_{t-1}]-\frac{b}{1-a}) \geq a^2(\E[E_{t-2}]-\frac{b}{1-a})\geq \ldots \geq a^t(\E[E_0]-\frac{b}{1-a})$, hence $\E[E_t]\geq\frac{b}{1-a} - a^t (\frac{b}{1-a} - \E[E_0])=(1-a^t)\frac{b}{1-a}$. Hence we have
\begin{align}
\label{eq:lower-bound-of-Et}
    \frac{1}{T}\sum_{t=0}^{T-1} \E[E_t] \geq \frac{b}{1-a} \left(1-\frac{1-a^T}{(1-a)T}\right) \geq \frac{b}{1-a} \left(1-\frac{1}{(1-a)T}\right).
\end{align}
Since $1-a=1-(1-0.3\frac{q}{d-1})(1-\frac{L}{\hat{L}})^2\leq 1-(1-0.3\frac{L}{\hat{L}})(1-\frac{L}{\hat{L}})^2$, noting that
\begin{align}
    \frac{1-(1-\frac{L}{\hat{L}})^2}{1-(1-0.3\frac{L}{\hat{L}})(1-\frac{L}{\hat{L}})^2}=\frac{\frac{L}{\hat{L}}+\frac{L}{\hat{L}}(1-\frac{L}{\hat{L}})}{\frac{L}{\hat{L}}+\frac{L}{\hat{L}}(1-\frac{L}{\hat{L}})+0.3\frac{L}{\hat{L}}(1-\frac{L}{\hat{L}})^2}\geq \frac{2}{2.3},
\end{align}
we have $\frac{1-a'}{1-a}\geq \frac{2}{2.3}$. Letting $T\geq 5\frac{d}{q}$, then $T\geq 5\frac{1}{\frac{q}{d}}\geq 5\frac{1}{\frac{L}{\hat{L}}}\geq 5\frac{1}{1-\sqrt{a'}}\geq 5\frac{1}{1-a}$. By Eq.~\eqref{eq:lower-bound-of-Et} we have
\begin{align}
    \frac{1}{T}\sum_{t=0}^{T-1} \E[E_t] \geq \frac{2}{2.3}\frac{b}{1-a'}\frac{4}{5} = \frac{2.4}{11.5}\frac{\frac{q}{d-1}}{1-a'}\geq 0.2 \frac{\frac{q}{d-1}}{1-a'}.
\end{align}
\end{proof}

Finally, the detailed proof of \textbf{Lemma~\ref{lem:prob-bound}} is as follows.

\begin{proof}
Let $\overline{E}:=\frac{1}{T}\sum_{t=0}^{T-1}\E_t[E_t]$. We note that $\{\zeta_1^2, \zeta_2^2, \ldots, \zeta_{T-1}^2\}$ are independent, and $\overline{E}$ is a function of them. Now suppose that the value of $\zeta_t^2$ is changed by $\Delta \zeta_t^2$, while the value of $\{\zeta_1^2, \ldots, \zeta_{t-1}^2, \zeta_{t+1}^2, \ldots, \zeta_{T-1}^2\}$ are unchanged. Then
\begin{align}
\Delta E_s&=0, &0\leq s \leq t-1; \\
\Delta E_s&=(1-a' E_{t-1})\Delta \zeta_t^2\leq \Delta \zeta_t^2, &s=t; \\
\Delta E_s&=(1-\zeta_s^2)a'\Delta E_{s-1}\leq a'\Delta E_{s-1}, &s\geq t+1.
\end{align}
Therefore, for $s\geq t$, $\Delta E_s\leq (a')^{s-t}\Delta E_t\leq (a')^{s-t}\Delta \zeta_t^2$; for $s<t$, $\Delta E_s=0$. By Eq.~\eqref{eq:iterE}, $\E_s[E_s]=a'(1-\E[\zeta_s^2]) E_{s-1}+\E[\zeta_s^2]$, so $\Delta \E_s[E_s]\leq a' \Delta E_{s-1}\leq \Delta E_{s-1}$. Hence
\begin{align}
    \Delta \overline{E}=\frac{1}{T}\sum_{s=0}^{T-1}\Delta \E_s[E_s]\leq \frac{1}{T}\sum_{s=t+1}^{T-1}(a')^{s-1-t}\Delta \zeta_t^2\leq \frac{1}{T}\frac{1}{1-a'}\Delta \zeta_t^2.
\end{align}
Since $\zeta_{ti}^2:=\min\{\xi_{ti}^2, \frac{1}{d-1}\}$, $0\leq \zeta_t^2\leq \frac{q}{d-1}$. Therefore $\Delta \overline{E}\leq \frac{1}{T}\frac{1}{1-a'}\frac{q}{d-1}$. Therefore, by McDiarmid's inequality, we have
\begin{align}
    \Pr(\overline{E}<\E[\overline{E}]-\epsilon)\leq \exp\left(-\frac{2\epsilon^2}{T\left(\frac{1}{T}\frac{1}{1-a'}\frac{q}{d-1}\right)^2}\right)\leq \exp\left(-2T\left(\epsilon(1-a')\frac{d-1}{q}\right)^2\right).
\end{align}
Let $\epsilon=0.1\frac{\frac{q}{d-1}}{1-a'}$, we have $\Pr(\overline{E}<\E[\overline{E}]-0.1\frac{\frac{q}{d-1}}{1-a'})\leq \exp(-0.02T)$. By Lemma~\ref{lem:total_expectation_bound}, $\E[\overline{E}]\geq 0.2\frac{\frac{q}{d-1}}{1-a'}$. Noting that $\frac{q}{d}\leq \frac{q}{d-1}$, the proof is completed.
\end{proof}

\section{Supplemental materials for Section~\ref{sec:4}}
\label{sec:C}
\subsection{Proof of Theorem~\ref{thm:nag}}
\label{sec:C1}
\begin{theorem_main}
\label{thm_main:nag}
In Algorithm~\ref{alg:nag-extended}, if $\theta_t$ is $\mathcal{F}_{t-1}$-measurable, we have
\begin{align}
\label{eq_main:theorem-nag}
    \E\left[(f(x_T)-f(x^*))\left(1+\frac{\sqrt{\gamma_0}}{2}\sum_{t=0}^{T-1}\sqrt{\theta_t}\right)^2\right]\leq f(x_0)-f(x^*)+\frac{\gamma_0}{2}\|x_0-x^*\|^2.
\end{align}
\end{theorem_main}
To help understand the design of Algorithm~\ref{alg:nag-extended}, we present the proof sketch below, where the part which is the same as the original proof in \cite{nesterov2017random} is omitted.
\begin{proof}[Sketch of the proof]
Since $x_{t+1}= y_t - \frac{1}{\hat{L}} g_1(y_t)$ and $g_1(y_t)= \nabla f(y_t)^\top v_t \cdot v_t$, by Lemma~\ref{lem:single-progress},
\begin{align}
\label{eq:single-process-nag}
    \E_t[f(x_{t+1})]\leq f(y_t) - \frac{\E_t\left[\left(\nabla f(y_t)^\top v_t\right)^2\right]}{2L'} &\leq f(y_t) - \frac{\E_t\left[\left(\nabla f(y_t)^\top v_t\right)^2\right]}{2\hat{L}}.
\end{align}
Define $\rho_t:=\frac{\gamma_t}{2}\|m_t-x^*\|^2 + f(x_t)-f(x^*)$. The same as in original proof, we have
\begin{align}
\label{eq:before_expectation}
    \rho_{t+1} &= \frac{\gamma_{t+1}}{2}\|m_t-x^*\|^2 - \alpha_t g_2(y_t)^\top (m_t-x^*) + \frac{\theta_t}{2}\|g_2(y_t)\|^2 + f(x_{t+1}) - f(x^*).
\end{align}
Then we derive $\E_t[\rho_{t+1}]$. We mentioned in Remark~\ref{rem:dependence-history} that the notation $\E_t[\cdot]$ means the conditional expectation $\E[\cdot|\mathcal{F}_{t-1}]$, where $\mathcal{F}_{t-1}$ is a sub $\sigma$-algebra modelling the historical information, and we require that $\mathcal{F}_{t-1}$ includes all the randomness before iteration $t$. Therefore, $\gamma_t$ and $m_t$ are $\mathcal{F}_{t-1}$-measurable. The assumption in Theorem~\ref{thm_main:nag} requires that $\theta_t$ is $\mathcal{F}_{t-1}$-measurable. Since $\alpha_t$ is determined by $\gamma_t$ and $\theta_t$ (through a Borel function), $\alpha_t$ is also $\mathcal{F}_{t-1}$-measurable. We have
\begin{align}
    &\ \ \ \ \ \E_t[\rho_{t+1}] \nonumber \\
    &= \frac{\gamma_{t+1}}{2}\|m_t-x^*\|^2 - \alpha_t \E_t[g_2(y_t)]^\top (m_t-x^*) + \frac{\theta_t}{2}\E_t[\|g_2(y_t)\|^2] + \E_t[f(x_{t+1})] - f(x^*) \label{eq:key-in-pars-proof-measurable} \\
    &= \frac{\gamma_{t+1}}{2}\|m_t-x^*\|^2 - \alpha_t \nabla f(y_t)^\top (m_t-x^*) + \frac{\theta_t}{2}\E_t[\|g_2(y_t)\|^2] + \E_t[f(x_{t+1})] - f(x^*) \label{eq:key-in-pars-proof-2} \\
    &\leq \frac{\gamma_{t+1}}{2}\|m_t-x^*\|^2 - \alpha_t \nabla f(y_t)^\top (m_t-x^*) + \frac{\E_t\left[\left(\nabla f(y_t)^\top v_t\right)^2\right]}{2\hat{L}} + \E_t[f(x_{t+1})] - f(x^*) \label{eq:key-in-pars-proof-1} \\
    &\leq \frac{\gamma_{t+1}}{2}\|m_t-x^*\|^2 - \alpha_t \nabla f(y_t)^\top (m_t-x^*) + f(y_t) - f(x^*) \label{eq:key-in-pars-proof-3} \\
    &\leq (1-\alpha_t) \rho_t,
\end{align}
where the first equality is because $m_t$, $\alpha_t$ and $\theta_t$ are $\mathcal{F}_{t-1}$-measurable, the second equality is because $\E_t[g_2(y_t)]=\nabla f(y_t)$, the first inequality is because $\theta_t\leq \frac{\E_t\left[\left(\nabla f(y_t)^\top v_t\right)^2\right]}{\hat{L}\cdot\E_t[\|g_2(y_t)\|^2]}$, the second inequality is because of Eq.~\eqref{eq:single-process-nag}, and the last inequality is the same as in original proof. By the similar reasoning to the proof of Theorem~\ref{thm_main:strong}, we have $\E\left[\frac{\rho_T}{\prod_{t=0}^{T-1}(1-\alpha_t)}\right]\leq \rho_0$. By the original proof, $\prod_{t=0}^{T-1}(1-\alpha_t)\leq \frac{1}{\left(1+\frac{\sqrt{\gamma_0}}{2}\sum_{t=0}^{T-1}\sqrt{\theta_t}\right)^2}$, completing the proof.
\end{proof}
From the proof sketch, we see that
\begin{itemize}
    \item The requirement that $\E_t[g_2(y_t)]=\nabla f(y_t)$ is to ensure that Eq.~\eqref{eq:key-in-pars-proof-2} holds.
    \item The constraint on $\theta_t$ in Line~\ref{lne:3-alg2} of Algorithm~\ref{alg:nag-extended} is to ensure that Eq.~\eqref{eq:key-in-pars-proof-1} holds.
    \item The choice of $g_1(y_t)$ ($g_1(y_t)= \nabla f(y_t)^\top v_t \cdot v_t$) and update of $x_t$ ($x_{t+1}= y_t - \frac{1}{\hat{L}} g_1(y_t)$) is the same as in Algorithm~\ref{alg:greedy-descent}, i.e. the greedy descent framework. This is since Eq.~\eqref{eq:key-in-pars-proof-3} requires that $\E_t[f(x_{t+1})]$ decreases as much as possible from $f(y_t)$.
    \item From Eq.~\eqref{eq:before_expectation} to Eq.~\eqref{eq:key-in-pars-proof-measurable}, we require $\E_t[\alpha_t g_2(y_t)^\top (m_t-x)]=\alpha_t\E_t[g_2(y_t)]^\top (m_t-x)$ and $\E_t[\theta_t \|g_2(y_t)\|^2]=\theta_t\E_t[\|g_2(y_t)\|^2]$. Therefore, to make the two above identities hold, by the property of ``pulling out known factors'' in taking conditional expectation, we require that $m_t$, $\alpha_t$ and $\theta_t$ are $\mathcal{F}_{t-1}$-measurable. Since we make sure in Remark~\ref{rem:dependence-history} that $\mathcal{F}_{t-1}$ always includes all the randomness before iteration $t$, and $\alpha_t$ is determined by $\gamma_t$ and $\theta_t$, it suffices to let $\theta_t$ be $\mathcal{F}_{t-1}$-measurable. We note that ``being $\mathcal{F}_{t-1}$-measurable'' means ``being determined by the history, i.e. fixed given the history''. 
\end{itemize}

Now we present the complete proof of Theorem~\ref{thm_main:nag}.
\begin{proof}
Since $x_{t+1}= y_t - \frac{1}{\hat{L}} g_1(y_t)$ and $g_1(y_t)= \nabla f(y_t)^\top v_t \cdot v_t$, by Lemma~\ref{lem:single-progress},
\begin{align}
    \E_t[f(x_{t+1})]&\leq f(y_t) - \frac{\E_t\left[\left(\nabla f(y_t)^\top v_t\right)^2\right]}{2L'} \\
    &\leq f(y_t) - \frac{\E_t\left[\left(\nabla f(y_t)^\top v_t\right)^2\right]}{2\hat{L}}.
\end{align}
For an arbitrary fixed $x$, define $\rho_t(x):=\frac{\gamma_t}{2}\|m_t-x\|^2 + f(x_t)-f(x)$. Then
\begin{align}
    \rho_{t+1}(x) &= \frac{\gamma_{t+1}}{2}\|m_{t+1}-x\|^2 + f(x_{t+1}) - f(x) \\
    &= \frac{\gamma_{t+1}}{2}\|m_t-x\|^2 - \frac{\gamma_{t+1}\theta_t}{\alpha_t}g_2(y_t)^\top (m_t-x) + \frac{\gamma_{t+1}\theta_t^2}{2\alpha_t^2}\|g_2(y_t)\|^2 + f(x_{t+1}) - f(x) \\
    &= \frac{\gamma_{t+1}}{2}\|m_t-x\|^2 - \alpha_t g_2(y_t)^\top (m_t-x) + \frac{\theta_t}{2}\|g_2(y_t)\|^2 + f(x_{t+1}) - f(x).
\end{align}

We make sure in Remark~\ref{rem:dependence-history} that $\mathcal{F}_{t-1}$ always includes all the randomness before iteration $t$. Therefore, $\gamma_t$ and $m_t$ are $\mathcal{F}_{t-1}$-measurable. The assumption in Theorem~\ref{thm_main:nag} requires that $\theta_t$ is $\mathcal{F}_{t-1}$-measurable. Since $\alpha_t$ is determined by $\gamma_t$ and $\theta_t$ (through a Borel function), $\alpha_t$ is also $\mathcal{F}_{t-1}$-measurable. Since $m_t$, $\alpha_t$ and $\theta_t$ are $\mathcal{F}_{t-1}$-measurable, we have $\E_t[\alpha_t g_2(y_t)^\top (m_t-x)]=\alpha_t\E_t[g_2(y_t)]^\top (m_t-x)$ and $\E_t[\theta_t \|g_2(y_t)\|^2]=\theta_t\E_t[\|g_2(y_t)\|^2]$. Hence

\begin{align}
    \E_t[\rho_{t+1}(x)] &= \frac{\gamma_{t+1}}{2}\|m_t-x^*\|^2 - \alpha_t \E_t[g_2(y_t)]^\top (m_t-x^*) + \frac{\theta_t}{2}\E_t[\|g_2(y_t)\|^2] + \E_t[f(x_{t+1})] - f(x^*) \\
    &= \frac{\gamma_{t+1}}{2}\|m_t-x\|^2 - \alpha_t \nabla f(y_t)^\top (m_t-x) + \frac{\theta_t}{2}\E_t[\|g_2(y_t)\|^2] + \E_t[f(x_{t+1})] - f(x) \\
    &\leq \frac{\gamma_{t+1}}{2}\|m_t-x\|^2 - \alpha_t \nabla f(y_t)^\top (m_t-x) + \frac{\E_t\left[\left(\nabla f(y_t)^\top v_t\right)^2\right]}{2\hat{L}} + \E_t[f(x_{t+1})] - f(x) \\
    &\leq \frac{\gamma_{t+1}}{2}\|m_t-x\|^2 - \alpha_t \nabla f(y_t)^\top (m_t-x) + f(y_t) - f(x) \\
    &= \frac{\gamma_{t+1}}{2}\|m_t-x\|^2 - \nabla f(y_t)^\top (\alpha_t m_t-\alpha_t x) + f(y_t) - f(x) \\
    &= \frac{\gamma_{t+1}}{2}\|m_t-x\|^2 + \nabla f(y_t)^\top (-y_t+(1-\alpha_t)x_t+\alpha_t x) + f(y_t) - f(x) \\
    &\leq \frac{\gamma_{t+1}}{2}\|m_t-x\|^2 + f\left((1-\alpha_t)x_t+\alpha_t x\right) - f(x) \\
    &\leq \frac{\gamma_{t+1}}{2}\|m_t-x\|^2 + (1-\alpha_t)f(x_t)-(1-\alpha_t) f(x) \\
    &= (1-\alpha_t) \left(\frac{\gamma_t}{2}\|m_t-x\|^2+f(x_t)-f(x)\right) \\
    &= (1-\alpha_t) \rho_t(x).
\end{align}

Therefore, 
\begin{align*}
    \rho_0(x)&=\E[\rho_0(x)]\geq \E\left[\frac{\E_0[\rho_1(x)]}{1-\alpha_0}\right]=\E\left[\E_0\left[\frac{\rho_1(x)}{1-\alpha_0}\right]\right]=\E\left[\frac{\rho_1(x)}{1-\alpha_0}\right] \\
    &\geq \E\left[\frac{\E_1[\rho_2(x)]}{(1-\alpha_0)(1-\alpha_1)}\right]=\E\left[\E_1\left[\frac{\rho_2(x)}{(1-\alpha_0)(1-\alpha_1)}\right]\right]=\E\left[\frac{\rho_2(x)}{(1-\alpha_0)(1-\alpha_1)}\right] \\
    &\geq \ldots \\
    &\geq \E\left[\frac{\rho_T(x)}{\prod_{t=0}^{T-1}(1-\alpha_t)}\right].
\end{align*}

We have $\rho_T(x)\geq f(x_T)-f(x)$. To prove the theorem, let $x=x^*$. The remaining is to give an upper bound of $\prod_{t=0}^{T-1}(1-\alpha_t)$. Let $\psi_k:=\prod_{t=0}^{k-1}(1-\alpha_t)$ and $a_k:=\frac{1}{\sqrt{\psi_k}}$, we have
\begin{align}
    a_{k+1}-a_k &= \frac{1}{\sqrt{\psi_{k+1}}}-\frac{1}{\sqrt{\psi_k}} = \frac{\sqrt{\psi_k}-\sqrt{\psi_{k+1}}}{\sqrt{\psi_k\psi_{k+1}}} = \frac{\psi_k-\psi_{k+1}}{\sqrt{\psi_k\psi_{k+1}}(\sqrt{\psi_k}+\sqrt{\psi_{k+1}})} \\
    &\geq \frac{\psi_k-\psi_{k+1}}{\sqrt{\psi_k\psi_{k+1}}(2\sqrt{\psi_k})} \\
    &= \frac{\psi_k-(1-\alpha_k)\psi_k}{2\psi_k\sqrt{\psi_{k+1}}}=\frac{\alpha_k}{2\sqrt{\psi_{k+1}}}=\frac{\sqrt{\gamma_{k+1}\theta_k}}{2\sqrt{\psi_{k+1}}}=\frac{\sqrt{\theta_k}}{2}\sqrt{\frac{\gamma_{k+1}}{\psi_{k+1}}} \\
    &= \frac{\sqrt{\gamma_0\theta_k}}{2}.
\end{align}
Since $\psi_0=1$, $a_0=1$. Hence $a_T\geq 1+\frac{\sqrt{\gamma_0}}{2}\sum_{t=0}^{T-1}\sqrt{\theta_t}$. Therefore, $\psi_T\leq \frac{1}{\left(1+\frac{\sqrt{\gamma_0}}{2}\sum_{t=0}^{T-1}\sqrt{\theta_t}\right)^2}$. The proof is completed.
\end{proof}

\subsection{Construction of $g_2(y_t)$}
\label{sec:C2}
We first note that in PARS, the specification of $\mathcal{F}_{t-1}$ is similar to that in Example~\ref{exp:new-prgf}. That is, we suppose that $p_t$ is determined before sampling $\{u_1,u_2,\ldots,u_q\}$, but it could depend on extra randomness in iteration $t$. We let $\mathcal{F}_{t-1}$ also includes the extra randomness of $p_t$ in iteration $t$ (not including the randomness of $\{u_1,u_2,\ldots,u_q\}$) besides the randomness before iteration $t$. Meanwhile, we note that the assumption in Theorem~\ref{thm_main:nag} requires that $\theta_t$ is $\mathcal{F}_{t-1}$-measurable, and this is satisfied if the algorithm to find $\theta_t$ in Algorithm~\ref{alg:nag-extended} is deterministic given randomness in $\mathcal{F}_{t-1}$ (does not use $\{u_1,u_2,\ldots,u_q\}$ in iteration $t$). Since $\mathcal{F}_{t-1}$ includes randomness before iteration $t$, if $\theta_t$ is $\mathcal{F}_{t-1}$-measurable, we can show that $y_t$ is $\mathcal{F}_{t-1}$-measurable.

We also note that in Section~\ref{sec:4} and Appendix~\ref{sec:C}, we let $D_t:=\left(\overline{\nabla f(y_t)}^\top p_t\right)^2$, which is different from the previous definition $D_t:=\left(\overline{\nabla f(x_t)}^\top p_t\right)^2$ in Section~\ref{sec:3} and Appendix~\ref{sec:B}. This is because in ARS-based algorithms, we care about gradient estimation at $y_t$ instead of that at $x_t$.

In Algorithm~\ref{alg:nag-extended}, we need to construct $g_2(y_t)$ as an unbiased estimator of $\nabla f(y_t)$ satisfying $\E_t[g_2(y_t)]=\nabla f(y_t)$. Since Theorem~\ref{thm:nag} tells us that a larger $\theta_t$ could potentially accelerate convergence, by Line~\ref{lne:3-alg2} of Algorithm~\ref{alg:nag-extended}, we want to make $\E_t[\|g_2(y_t)\|^2]$ as small as possible. To save queries, we hope that we can reuse the queries $\nabla f(y_t)^\top p_t$ and $\{\nabla f(y_t)^\top u_i\}_{i=1}^q$ used in the process of constructing $g_1(y_t)$. 

Here we adopt the construction process in Section~\ref{sec:construction-prgf}, and leave the discussion of alternative ways in Section~\ref{sec:alternative}. We note that if we let $H$ be the $(d-1)$-dimensional subspace perpendicular to $p_t$, then
\begin{align}
\label{eq:decomp}
    \nabla f(y_t) = \nabla f(y_t)^\top p_t\cdot p_t+(\I-p_t p_t^\top)\nabla f(y_t) = \nabla f(y_t)^\top p_t\cdot p_t + \nabla f(y_t)_H.
\end{align}
Therefore, we need to obtain an unbiased estimator of $\nabla f(y_t)_H$. This is straightforward since we can utilize $\{u_i\}_{i=1}^q$ which is uniformly sampled from the $(d-1)$-dimensional space $H$.
\begin{proposition}
\label{eq:expectation-subspace}
For any $1\leq i\leq q$, $\E_t[\nabla f(y_t)^\top u_i\cdot u_i]=\frac{1}{d-1}\nabla f(y_t)_H$.
\end{proposition}
\begin{proof}
We have $\E_t[u_i u_i^\top]=\frac{\I-p_t p_t^\top}{d-1}$ (See Section A.2 in~\cite{cheng2019improving} for the proof.). Therefore,
\begin{align}
    \E_t[\nabla f(y_t)^\top u_i\cdot u_i]=\frac{\I-p_t p_t^\top}{d-1} \nabla f(y_t) = \frac{1}{d-1}\nabla f(y_t)_H.
\end{align}
\end{proof}
Therefore,
\begin{align}
\label{eq:g2}
    g_2(y_t)=\nabla f(y_t)^\top p_t\cdot p_t+\frac{d-1}{q}\sum_{i=1}^q \nabla f(y_t)^\top u_i \cdot u_i
\end{align}
satisfies that $\E_t[g_2(y_t)]=\nabla f(y_t)$. Then
\begin{align}
    \E_t[\|g_2(y_t)\|^2]&=\|\nabla f(y_t)\|^2 \E_t\left[\left\|\overline{\nabla f(y_t)}^\top p_t\cdot p_t+\frac{d-1}{q}\sum_{i=1}^q \overline{\nabla f(y_t)}^\top u_i \cdot u_i\right\|^2\right] \\
    &= \|\nabla f(y_t)\|^2 \left(\left(\overline{\nabla f(y_t)}^\top p_t\right)^2+\frac{(d-1)^2}{q^2}\sum_{i=1}^q \E_t\left[\left(\overline{\nabla f(y_t)}^\top u_i\right)^2\right]\right) \\
    &= \|\nabla f(y_t)\|^2 \left(D_t+\frac{d-1}{q}(1-D_t)\right), \label{eq:second-moment}
\end{align}
where the last equality is due to the fact that $\E_t[u_i u_i^\top]=\frac{\I-p_t p_t^\top}{d-1}$ (hence $\E_t\left[\left(\overline{\nabla f(y_t)}^\top u_i\right)^2\right]=\frac{1-D_t}{d-1}$). Meanwhile, if we adopt an RGF estimator as $g_2(y_t)$, then $\E_t[\|g_2(y_t)\|^2]=\frac{d}{q}\|\nabla f(y_t)\|^2$. Noting that $D_t+\frac{d-1}{q}(1-D_t)<\frac{d}{q}$, our proposed unbiased estimator results in a smaller $\E_t[\|g_2(y_t)\|^2]$ especially when $D_t$ is closed to $1$, since it utilizes the prior information.

Finally, using $g_2(y_t)$ in Eq.~\eqref{eq:g2}, when calculating the following expression which appears in Line~\ref{lne:3-alg2} of Algorithm~\ref{alg:nag-extended}, the term $\|\nabla f(y_t)\|^2$ would be cancelled:
\begin{align}
\label{eq:including-Dt}
    \frac{\E_t\left[\left(\nabla f(y_t)^\top v_t\right)^2\right]}{\hat{L}\cdot\E_t[\|g_2(y_t)\|^2]} = \frac{\E_t\left[\left(\overline{\nabla f(y_t)}^\top v_t\right)^2\right]}{\hat{L}\left(D_t+\frac{d-1}{q}(1-D_t)\right)} = \frac{D_t+\frac{q}{d-1}(1-D_t)}{\hat{L}\left(D_t+\frac{d-1}{q}(1-D_t)\right)},
\end{align}
where the last equality is due to Lemma~\ref{lem_main:expected-drift}.

\subsubsection{Alternative way to construct $g_2(y_t)$}
\label{sec:alternative}
Instead of using the orthogonalization process in Section~\ref{sec:construction-prgf}, when constructing $g_1(y_t)$ as the PRGF estimator, we can also first sample $q$ orthonormal vectors $\{u_i\}_{i=1}^q$ uniformly from $\Sp_{d-1}$, and then let $p_t$ be orthogonal to them with $\{u_i\}_{i=1}^q$ unchanged. Then we can construct $g_2(y_t)$ using this set of $\{\nabla f(y_t)^\top u_i\}_{i=1}^q$ and $\nabla f(y_t)^\top p_t$.

\begin{example}[RGF]
Since $\{u_i\}_{i=1}^q$ are uniformly sampled from $\Sp_{d-1}$, we can directly use them to construct an unbiased estimator of $\nabla f(y_t)$. We let $g_2(y_t)=\frac{d}{q}\sum_{i=1}^q \nabla f(y_t)^\top u_i\cdot u_i$. We show that it is an unbiased estimator of $\nabla f(y_t)$, and $\E_t[\|g_2(y_t)\|^2]= \frac{d}{q} \|\nabla f(y_t)\|^2$.
\begin{proof}
    In Section~\ref{sec:rgf} we show that $\E[u_i u_i^\top]=\frac{\I}{d}$. Therefore
    \begin{align*}
        \E_t[g_2(y_t)]&=\frac{d}{q}\sum_{i=1}^q \E_t[u_i u_i^\top]\nabla f(y_t) =\frac{d}{q}\sum_{i=1}^q \frac{1}{d}\nabla f(y_t)=\nabla f(y_t).
    \end{align*}
    \begin{align*}
        \E_t[\|g_2(y_t)\|^2]&=\frac{d^2}{q^2}\sum_{i=1}^q \E_t[(\nabla f(y_t)^\top u_i)^2]=\frac{d^2}{q^2}\sum_{i=1}^q \nabla f(y_t)^\top \E_t[u_i u_i^\top]\nabla f(y_t) \\ &=\frac{d^2}{q^2}\sum_{i=1}^q \frac{1}{d}\|\nabla f(y_t)\|^2=\frac{d}{q}\|\nabla f(y_t)\|^2.
    \end{align*}
\end{proof}
We see that $\E_t[\|g_2(y_t)\|^2]$ here is larger than Eq.~\eqref{eq:second-moment}.
\end{example}
\begin{example}[Variance reduced RGF]
To reduce the variance of RGF estimator, we could use $p_t$ to construct a control variate. Here we use $p_t$ to refer to the original $p_t^{ori}$ before orthogonalization so that it is fixed w.r.t. randomness of $\{u_1,\ldots,u_q\}$ (then it requires one additional query to obtain $\nabla f(y_t)^\top p_t^{ori}$). Specifically, we can let $g_2(y_t)=\frac{d}{q}\sum_{i=1}^q (\nabla f(y_t)^\top u_i\cdot u_i-(\nabla f(y_t)^\top p_t\cdot p_t)^\top u_i\cdot u_i)+\nabla f(y_t)^\top p_t\cdot p_t$. We show that it is unbiased, and $\E_t[\|g_2(y_t)\|^2]=\|\nabla f(y_t)\|^2 \left(D_t+\frac{d}{q}(1-D_t)\right)$.
\begin{proof}
    \begin{align*}
        \E_t[(\nabla f(y_t)^\top p_t\cdot p_t)^\top u_i\cdot u_i]=\E_t[u_i u_i^\top]\nabla f(y_t)^\top p_t\cdot p_t=\frac{1}{d}\nabla f(y_t)^\top p_t\cdot p_t.
    \end{align*}
    Therefore,
    \begin{align*}
        \E_t[g_2(y_t)]=\E_t\left[\frac{d}{q}\sum_{i=1}^q \nabla f(y_t)^\top u_i\cdot u_i\right]=\nabla f(y_t).
    \end{align*}
    Let $\nabla f(y_t)_H:=\nabla f(y_t)-\nabla f(y_t)^\top p_t\cdot p_t$. We define that $\V[x]$ for a stochastic vector $x$ is such that $\V[x]=\sum_i\V[x_i]$. Then for any stochastic vector $x$, $\E[\|x\|^2]=\|\E[x]\|^2+\V[x]$. We have $\V_t[g_2(y_t)]=\V_t\left[\frac{d}{q}\sum_{i=1}^q \nabla f(y_t)_H^\top u_i\cdot u_i\right]$.
    Let $g_2'(y_t):=\frac{d}{q}\sum_{i=1}^q \nabla f(y_t)_H^\top u_i\cdot u_i$. Then $\E_t[g_2'(y_t)]=\nabla f(y_t)_H$, $\E_t[\|g_2'(y_t)\|^2]=\frac{d}{q}\|\nabla f(y_t)_H\|^2$. Therefore, $\V_t[g_2(y_t)]=\E_t[\|g_2'(y_t)\|^2]-\|\E_t[g_2'(y_t)\|^2]=\left(\frac{d}{q}-1\right)\|\nabla f(y_t)_H\|^2=(1-D_t)\left(\frac{d}{q}-1\right)\|\nabla f(y_t)\|^2$. Hence,
    \begin{align*}
        \E_t[\|g_2(y_t)\|^2]&=\|\E_t[g_2(y_t)]\|^2+\V_t[g_2(y_t)]=\left(1+(1-D_t)\left(\frac{d}{q}-1\right)\right)\|\nabla f(y_t)\|^2 \\
        &=\left(D_t+\frac{d}{q}(1-D_t)\right)\|\nabla f(y_t)\|^2.
    \end{align*}
\end{proof}
We see that $\E_t[\|g_2(y_t)\|^2]$ here is comparable but slightly worse (slightly larger) than Eq.~\eqref{eq:second-moment}.
\end{example}
In summary, we still favor Eq.~\eqref{eq:g2} as the construction of $g_2(y_t)$ due to its simplicity and smallest value of $\E_t[\|g_2(y_t)\|^2]$ among the ones we propose.

\subsection{Estimation of $D_t$ and proof of convergence of PARS using this estimator}
\label{sec:C3}
\label{sec:estimation-Dt}
In zeroth-order optimization, $D_t$ is not accessible since $D_t=\left(\frac{\nabla f(y_t)^\top p_t}{\|\nabla f(y_t)\|}\right)^2$. For the term inside the square, while the numerator can be obtained from the oracle, we do not have access to the denominator. Therefore, our task is to estimate $\|\nabla f(y_t)\|^2$.

To save queries, it is ideal to reuse the oracle query results used to obtain $v_t$ and $g_2(y_t)$: $\nabla f(y_t)^\top p_t$ and $\{\nabla f(y_t)^\top u_i\}_{i\in\{1,2,\ldots,q\}}$. Again, we suppose $p_t,u_1,\cdots,u_q$ are obtained from the process in Section~\ref{sec:construction-prgf}. By Eq.~\eqref{eq:decomp}, we have
\begin{align}
    \|\nabla f(y_t)\|^2=(\nabla f(y_t)^\top p_t)^2+\|\nabla f(y_t)_H\|^2.
\end{align}
Since $\{u_i\}_{i=1}^q$ is uniformly sampled from the $(d-1)$-dimensional space $H$,
\begin{proposition}
For any $1\leq i\leq q$, $\E_t[(\nabla f(y_t)^\top u_i)^2]=\frac{1}{d-1}\|\nabla f(y_t)_H\|^2$.
\end{proposition}
\begin{proof}
By Proposition~\ref{eq:expectation-subspace}, $\E_t[\nabla f(y_t)^\top u_i\cdot u_i]=\frac{1}{d-1}\nabla f(y_t)_H$. Therefore,
\begin{align}
    \E_t[(\nabla f(y_t)^\top u_i)^2]&=\nabla f(y_t)^\top \E_t[\nabla f(y_t)^\top u_i\cdot u_i]=\frac{1}{d-1}\nabla f(y_t)^\top \nabla f(y_t)_H \\
    &=\frac{1}{d-1}\|\nabla f(y_t)_H\|^2.
\end{align}
\end{proof}
Thus, we adopt the following unbiased estimate:
\begin{align}
    \|\nabla f(y_t)_H\|^2\approx \frac{d-1}{q}\sum_{i=1}^q \left(\nabla f(y_t)^\top u_i\right)^2
\end{align}
By Johnson-Lindenstrauss Lemma (see Lemma 5.3.2 in \cite{vershynin2018high}), this approximation is rather accurate given a moderate value of $q$. Therefore, we have
\begin{align}
\label{eq:est-norm}
\|\nabla f(y_t)\|^2\approx \left(\nabla f(y_t)^\top p_t\right)^2+\frac{d-1}{q}\sum_{i=1}^q \left(\nabla f(y_t)^\top u_i\right)^2
\end{align}
and
\begin{align}
\label{eq:est-Dt}
    D_t=\frac{(\nabla f(y_t)^\top p_t)^2}{\|\nabla f(y_t)\|^2}\approx\frac{\left(\nabla f(y_t)^\top p_t\right)^2}{\left(\nabla f(y_t)^\top p_t\right)^2+\frac{d-1}{q}\sum_{i=1}^q \left(\nabla f(y_t)^\top u_i\right)^2}.
\end{align}

\subsubsection{PARS-Est algorithm with theoretical guarantee}
In fact, the estimator of $D_t$ concentrates well around the true value of $D_t$ given a moderate value of $q$. To reach an algorithm with theoretical guarantee, we could adopt a conservative estimate of $D_t$, such that the constraint of $\theta_t$ in Line~\ref{lne:3-alg2} of Algorithm~\ref{alg:nag-extended} is satisfied with high probability. We show the prior-guided implementation of Algorithm~\ref{alg:nag-extended} with estimation of $D_t$ in Algorithm~\ref{alg:nag-impl-theory}, call it PARS-Est, and show that it admits a theoretical guarantee.
\begin{algorithm}[!htbp]
\small
\caption{Prior-guided ARS with a conservative estimator of $D_t$ (PARS-Est)}
\label{alg:nag-impl-theory}
\begin{algorithmic}[1]
\Require $L$-smooth convex function $f$; initialization $x_0$; $\hat{L} \geq L$; Query count per iteration $q$; iteration number $T$; $\gamma_0>0$.
\Ensure $x_T$ as the approximate minimizer of $f$.
\State $m_0\leftarrow x_0$;
\For {$t = 0$ to $T-1$}
\State Obtain the prior $p_t$;
\State Find a $\theta_t$ such that $\theta_t\leq \theta_t'$ in which $\theta_t'$ is defined in the following steps: \label{lne:finding-begin}
\State Step 1: $y_t\leftarrow(1-\alpha_t)x_t+\alpha_t m_t$, where $\alpha_t\geq0$ is a positive root of the equation $\alpha_t^2=\theta_t (1-\alpha_t)\gamma_t$; $\gamma_{t+1}\leftarrow (1-\alpha_t)\gamma_t$;\label{lne:blue-begin}
\State Step 2: Sample an orthonormal set of $\{u_i\}_{i=1}^q$ in the subspace perpendicular to $p_t$ uniformly, as in Section~\ref{sec:construction-prgf}; \label{lne:first-sample}
\State Step 3: $\hat{D}_t\leftarrow\frac{\left(\nabla f(y_t)^\top p_t\right)^2}{\left(\nabla f(y_t)^\top p_t\right)^2+\frac{{\color{red} 2 }(d-1)}{q}\sum_{i=1}^q \left(\nabla f(y_t)^\top u_i\right)^2}$; $\theta_t'\leftarrow\frac{\hat{D}_t+\frac{q}{d-1}(1-\hat{D}_t)}{\hat{L}\left(\hat{D}_t+\frac{d-1}{q}(1-\hat{D}_t)\right)}$; \label{lne:finding-end}
\State Resample $\{u_i\}_{i=1}^q$ and calculate $v_t$ as in Section~\ref{sec:construction-prgf}; \label{lne:resample}
\State $g_1(y_t)\leftarrow \nabla f(y_t)^\top v_t \cdot v_t=\sum_{i=1}^q \nabla f(y_t)^\top u_i \cdot u_i+\nabla f(y_t)^\top p_t\cdot p_t$;
\State $g_2(y_t)\leftarrow\frac{d-1}{q}\sum_{i=1}^q \nabla f(y_t)^\top u_i \cdot u_i+\nabla f(y_t)^\top p_t\cdot p_t$;
\State $x_{t+1}\leftarrow y_t - \frac{1}{\hat{L}} g_1(y_t)$, $m_{t+1}\leftarrow m_t - \frac{\theta_t}{\alpha_t} g_2(y_t)$;
\EndFor
\Return $x_T$.
\end{algorithmic}
\end{algorithm}

\begin{theorem}
\label{thm:pars-est-theory}
Let
\begin{align}
    p=\Pr\left(\sum_{i=1}^q x_i^2< \frac{q}{2(d-1)}\right)
\end{align}
where $(x_1,x_2,...,x_{d-1})^\top\sim\U(\Sp_{d-2})$, i.e. follows a uniform distribution over the unit $(d-1)$-dimensional sphere. Then, in Algorithm~\ref{alg:nag-impl-theory}, for any $\delta\in(0,1)$, choosing a $q$ such that $p\leq \frac{\delta}{T}$, there exists an event $M$ such that $\Pr(M)\geq 1-\delta$, and
\begin{align}
\label{eq:convergence}
    \E\left[(f(x_T) - f(x^*))\left(1+\frac{\sqrt{\gamma_0}}{2}\sum_{t=0}^{T-1}\sqrt{\theta_t}\right)^2\middle|M\right]\leq f(x_0)-f(x^*)+\frac{\gamma_0}{2}\|x_0-x^*\|^2.
\end{align}
\end{theorem}
\begin{proof}
We first explain the definition of $\mathcal{F}_{t-1}$ in the proof (recall that $\E_t[\cdot]$ is $\E[\cdot|\mathcal{F}_{t-1}]$). Since in Theorem~\ref{thm_main:nag} we require $\theta_t$ to be $\mathcal{F}_{t-1}$-measurable, we let $\mathcal{F}_{t-1}$ also includes the randomness in Line~\ref{lne:first-sample} of Algorithm~\ref{alg:nag-impl-theory} in iteration $t$, besides randomness before iteration $t$ and randomness of $p_t$. We note that $\mathcal{F}_{t-1}$ does not include the randomness in Line~\ref{lne:resample}.

Let $M$ be the event that: For each $t\in\{0,1,...,T-1\}$, $\frac{d-1}{q}\sum_{i=1}^q \left(\nabla f(y_t)^\top u_i\right)^2\geq \frac{1}{2}\|\nabla f(y_t)_H\|^2$. When $M$ is true, we have that $\forall t$,
\begin{align}
    \hat{D}_t&=\frac{\left(\nabla f(y_t)^\top p_t\right)^2}{\left(\nabla f(y_t)^\top p_t\right)^2+\frac{2(d-1)}{q}\sum_{i=1}^q \left(\nabla f(y_t)^\top u_i\right)^2} \\
    &\leq \frac{\left(\nabla f(y_t)^\top p_t\right)^2}{\left(\nabla f(y_t)^\top p_t\right)^2+\|\nabla f(y_t)_H\|^2}\\
    &=\frac{\left(\nabla f(y_t)^\top p_t\right)^2}{\|\nabla f(y_t)\|^2}=D_t.
\end{align}
Therefore,
\begin{align}
    \theta_t &\leq\theta_t'=\frac{\hat{D}_t+\frac{q}{d-1}(1-\hat{D}_t)}{\hat{L}\left(\hat{D}_t+\frac{d-1}{q}(1-\hat{D}_t)\right)}\leq\frac{D_t+\frac{q}{d-1}(1-D_t)}{\hat{L}\left(D_t+\frac{d-1}{q}(1-D_t)\right)}=\frac{\E_t\left[\left(\nabla f(y_t)^\top v_t\right)^2\right]}{\hat{L}\cdot\E_t[\|g_2(y_t)\|^2]}. \label{eq:theta_t_alg_3}
\end{align}
Since $\mathcal{F}_{t-1}$ includes the randomness in Line~\ref{lne:first-sample} of Algorithm~\ref{alg:nag-impl-theory} in iteration $t$, $\E_t[\cdot]$ refer to only taking expectation w.r.t. the randomness of $v_t$ and $g_2(y_t)$ in iteration $t$, i.e. w.r.t. $\{u_1,\ldots,u_q\}$ in Line~\ref{lne:resample} of Algorithm~\ref{alg:nag-impl-theory}. Since $\{u_1,\ldots,u_q\}$ in Line~\ref{lne:resample} is independent of $\{u_1,\ldots,u_q\}$ in Line~\ref{lne:first-sample}, adding $\{u_1,\ldots,u_q\}$ in Line~\ref{lne:first-sample} to the history does not change the distribution of $\{u_1,\ldots,u_q\}$ in Line~\ref{lne:resample} given the history. Therefore according to the analysis in 
Section~\ref{sec:C2}, the last equality of Eq.~\eqref{eq:theta_t_alg_3} holds, and $\E_t[g_2(y_t)]=\nabla f(y_t)$. By Theorem~\ref{thm_main:nag}, Eq.~\eqref{eq:convergence} is proved.

Next we give a lower bound of $\Pr(M)$. Let us fix $t$. Then
\begin{align*}
    \Pr\left(\frac{d-1}{q}\sum_{i=1}^q \left(\nabla f(y_t)^\top u_i\right)^2< \frac{1}{2}\|\nabla f(y_t)_H\|^2\right)=p.
\end{align*}
Since for different $t$ the events inside the brackets are independent, by union bound we have $\Pr(M)\geq 1-pT$. Since $p\leq \frac{\delta}{T}$, the proof is completed.
\end{proof}
\begin{remark}
To save queries, one may think that when constructing $v_t$ and $g_2(y_t)$, we could omit the procedure of resampling $\{u_i\}_{i=1}^q$ in Line~\ref{lne:resample}, and reuse $\{u_i\}_{i=1}^q$ sampled in Line~\ref{lne:first-sample} to utilize the queries of relevant directional derivatives in Line~\ref{lne:finding-end}. Our theoretical analysis does not support this yet, as explained below.

If we reuse $\{u_i\}_{i=1}^q$ sampled in Line~\ref{lne:first-sample} to construct $v_t$ and $g_2(y_t)$, then both $\theta_t$ and $\{g_2(y_t), v_t\}$ depend on the same set of $\{u_i\}_{i=1}^q$. Since Theorem~\ref{thm_main:nag} requires $\theta_t$ to be $\mathcal{F}_{t-1}$-measurable, we have to make $\mathcal{F}_{t-1}$ include randomness of this set of $\{u_i\}_{i=1}^q$. Then both $g_2(y_t)$ and $v_t$ are fixed given the history $\mathcal{F}_{t-1}$, which is not desired (e.g. $\E_t[g_2(y_t)]=\nabla f(y_t)$ generally does not hold since $\E_t[g_2(y_t)]=g_2(y_t)$ now, making the proof of Theorem~\ref{thm_main:nag} fail).
\end{remark}
\begin{remark}
For given $d$ and $q$, $p$ can be calculated in closed form with the aid of softwares such as Mathematica. When $d=2000$, if $q=50$, then $p\approx 7.5\times 10^{-4}$. If $q=100$, then $p\approx 3.5\times 10^{-6}$. Hence $p$ is rather small so that a moderate value of $q$ is enough to let $p\leq\frac{\delta}{T}$.

In fact, $p$ can be bounded by $O(\exp(-cq))$ by Johnson-Lindenstrauss Lemma where $c$ is an absolute constant (see Lemma 5.3.2 in \cite{vershynin2018high}). Note that the bound is exponentially decayed w.r.t. $q$ and independent of $d$.
\end{remark}
\begin{remark}
We give an analysis of the influence of the additional factor {\color{red} 2} in Line~\ref{lne:finding-end} of Algorithm~\ref{alg:nag-impl-theory}. Let
\begin{align*}
    \hat{D}_{t2}&=\frac{\left(\nabla f(y_t)^\top p_t\right)^2}{\left(\nabla f(y_t)^\top p_t\right)^2+\frac{2(d-1)}{q}\sum_{i=1}^q \left(\nabla f(y_t)^\top u_i\right)^2}, \\
    \hat{D}_{t1}&=\frac{\left(\nabla f(y_t)^\top p_t\right)^2}{\left(\nabla f(y_t)^\top p_t\right)^2+\frac{d-1}{q}\sum_{i=1}^q \left(\nabla f(y_t)^\top u_i\right)^2}, \\
    \theta_{t2}&=\frac{\hat{D}_{t2}+\frac{q}{d-1}(1-\hat{D}_{t2})}{\hat{L}\left(\hat{D}_{t2}+\frac{d-1}{q}(1-\hat{D}_{t2})\right)}, \\
    \theta_{t1}&=\frac{\hat{D}_{t1}+\frac{q}{d-1}(1-\hat{D}_{t1})}{\hat{L}\left(\hat{D}_{t1}+\frac{d-1}{q}(1-\hat{D}_{t1})\right)}.
\end{align*}
Then $\hat{D}_{t1}\geq \hat{D}_{t2}$ and $1-\hat{D}_{t1}\leq 1-\hat{D}_{t2}$. We have
\begin{align}
    \frac{\theta_{t2}}{\theta_{t1}}&=\frac{\hat{D}_{t2}+\frac{q}{d-1}(1-\hat{D}_{t2})}{\hat{D}_{t1}+\frac{q}{d-1}(1-\hat{D}_{t1})}\cdot\frac{\hat{D}_{t1}+\frac{d-1}{q}(1-\hat{D}_{t1})}{\hat{D}_{t2}+\frac{d-1}{q}(1-\hat{D}_{t2})} \\
    &\geq \frac{\hat{D}_{t2}}{\hat{D}_{t1}}\cdot \frac{1-\hat{D}_{t1}}{1-\hat{D}_{t2}} \\
    &= \frac{\frac{\hat{D}_{t2}}{1-\hat{D}_{t2}}}{\frac{\hat{D}_{t1}}{1-\hat{D}_{t1}}} \\
    &= \frac{\frac{\left(\nabla f(y_t)^\top p_t\right)^2}{\frac{2(d-1)}{q}\sum_{i=1}^q \left(\nabla f(y_t)^\top u_i\right)^2}}{\frac{\left(\nabla f(y_t)^\top p_t\right)^2}{\frac{d-1}{q}\sum_{i=1}^q \left(\nabla f(y_t)^\top u_i\right)^2}} \\
    &= \frac{1}{2}.
\end{align}
Therefore, we have $\theta_{t2}\geq \frac{1}{2}\theta_{t1}$.

Meanwhile, since $\hat{D}_{t2}\geq 0$, we have $\theta_{t2}\geq \frac{q^2}{\hat{L}(d-1)^2}$. Hence $\theta_{t2}\geq\max\left\{\frac{1}{2}\theta_{t1}, \frac{q^2}{\hat{L}(d-1)^2}\right\}$.
\end{remark}

\subsection{Approximate solution of $\theta_t$ and implementation of PARS in practice (PARS-Impl)}
\label{sec:C4}
We note that in Line~\ref{lne:3-alg2} of Algorithm~\ref{alg:nag-extended}, it is not straightforward to obtain an ideal solution of $\theta_t$, since $y_t$ depends on $\theta_t$. Theoretically speaking, $\theta_t>0$ satisfying the inequality $\theta_t\leq \frac{\E_t\left[\left(\nabla f(y_t)^\top v_t\right)^2\right]}{\hat{L}\cdot\E_t[\|g_2(y_t)\|^2]}$ always exists, since by Eq.~\eqref{eq:including-Dt}, $\frac{\E_t\left[\left(\nabla f(y_t)^\top v_t\right)^2\right]}{\hat{L}\cdot\E_t[\|g_2(y_t)\|^2]}\geq \frac{q^2}{\hat{L}(d-1)^2}:=\theta$ always holds, so we can always let $\theta_t=\theta$. However, such estimate of $\theta_t$ is too conservative and does not benefit from a good prior (when $D_t$ is large). Therefore, one can guess a value of $D_t$, and then compute the value of $\theta_t$ by Eq.~\eqref{eq:including-Dt}, and then estimate the value of $D_t$ and verify that $\theta_t\leq \frac{\E_t\left[\left(\nabla f(y_t)^\top v_t\right)^2\right]}{\hat{L}\cdot\E_t[\|g_2(y_t)\|^2]}$ holds. If it does not hold, we need to try a smaller $\theta_t$ until the inequality is satisfied. For example, in Algorithm~\ref{alg:nag-impl-theory}, if we implement its Line~\ref{lne:finding-begin} to Line~\ref{lne:finding-end} with a guessing procedure,\footnote{For example, (1) compute $\theta_t'$ with $\theta_t\leftarrow 0$ by running Line~\ref{lne:blue-begin} to Line~\ref{lne:finding-end}; (2) 
we guess $\theta_t\leftarrow \kappa\theta_t'$ to compute a new $\theta_t'$ by rerunning Line~\ref{lne:blue-begin} to Line~\ref{lne:finding-end}, where $0<\kappa< 1$ is a discount factor to obtain a more conservative estimate of $\theta_t$; (3) if $\theta_t\leq \theta_t'$, then we have found $\theta_t$ as required; else, we go to step (2).} we could obtain an runnable algorithm with convergence guarantee. However, in practice such procedure could require multiple runs from Line~\ref{lne:blue-begin} to Line~\ref{lne:finding-end} in Algorithm~\ref{alg:nag-impl-theory}, which requires many additional queries; on the other hand, due to the additional factor $2$ in Line~\ref{lne:finding-end} of Algorithm~\ref{alg:nag-impl-theory}, we would always find a conservative estimate of $\theta_t$.

In this section, we introduce the algorithm we use to find an approximate solution to find $\theta_t$ in Line~\ref{lne:3-alg2} of Algorithm~\ref{alg:nag-extended}, which does not have theoretical guarantee but empirically performs well. The full algorithm PARS-Impl is shown in Algorithm~\ref{alg:pars-impl}. It stills follow the PARS framework (Algorithm~\ref{alg:nag-extended}), and our procedure to find $\theta_t$ is shown in Line~\ref{lne:approx-start} to Line~\ref{lne:approx-end}.\footnote{Line~\ref{lne:approx-start} and Line~\ref{lne:approx-end} require the query of $\nabla f(y_t^{(0)})^\top p_t$ and $\nabla f(y_t^{(1)})^\top p_t$ respectively, so each iteration of Algorithm~\ref{alg:pars-impl} requires $2$ additional queries to the directional derivative oracle, or requires $4$ additional queries to the function value oracle using finite difference approximation of the directional derivative.} Next we explain the procedure to find $\theta_t$ in detail.

\begin{algorithm}[!htbp]
\small
\caption{Prior-Guided Accelerated Random Search in implementation (PARS-Impl)}
\label{alg:pars-impl}
\begin{algorithmic}[1]
\Require $L$-smooth convex function $f$; initialization $x_0$; $\hat{L} \geq L$; Query count per iteration $q$ (cannot be too small); iteration number $T$; $\gamma_0>0$.
\Ensure $x_T$ as the approximate minimizer of $f$.
\State $m_0\leftarrow x_0$;
\State $\|\hat{\nabla} f_{-1}\|^2\leftarrow +\infty$;
\For {$t = 0$ to $T-1$}
\State Obtain the prior $p_t$;
\State $y_t^{(0)}\leftarrow x_t$; $\hat{D}_t\leftarrow \frac{(\nabla f(y_t^{(0)})^\top p_t)^2}{\|\hat{\nabla}f_{t-1}\|^2}$; $\theta_t\leftarrow \frac{\hat{D}_t+\frac{q}{d-1}(1-\hat{D}_t)}{\hat{L}\left(\hat{D}_t+\frac{d-1}{q}(1-\hat{D}_t)\right)}$; \label{lne:approx-start}
\State $y_t^{(1)}\leftarrow(1-\alpha_t)x_t+\alpha_t m_t$, where $\alpha_t\geq0$ is a positive root of the equation $\alpha_t^2=\theta_t (1-\alpha_t)\gamma_t$; \label{lne:approx-middle}
\State $\hat{D}_t\leftarrow \frac{(\nabla f(y_t^{(1)})^\top p_t)^2}{\|\hat{\nabla}f_{t-1}\|^2}$; $\theta_t\leftarrow \frac{\hat{D}_t+\frac{q}{d-1}(1-\hat{D}_t)}{\hat{L}\left(\hat{D}_t+\frac{d-1}{q}(1-\hat{D}_t)\right)}$; \label{lne:approx-end}
\State $y_t\leftarrow(1-\alpha_t)x_t+\alpha_t m_t$, where $\alpha_t\geq0$ is a positive root of the equation $\alpha_t^2=\theta_t (1-\alpha_t)\gamma_t$; $\gamma_{t+1}\leftarrow (1-\alpha_t)\gamma_t$;
\State Sample an orthonormal set of $\{u_i\}_{i=1}^q$ in the subspace perpendicular to $p_t$ uniformly, as in Section~\ref{sec:construction-prgf};
\State $g_1(y_t)\leftarrow \sum_{i=1}^q \nabla f(y_t)^\top u_i \cdot u_i+\nabla f(y_t)^\top p_t\cdot p_t$;
\State $g_2(y_t)\leftarrow\frac{d-1}{q}\sum_{i=1}^q \nabla f(y_t)^\top u_i \cdot u_i+\nabla f(y_t)^\top p_t\cdot p_t$;
\State $\|\hat{\nabla} f_t\|^2\leftarrow \left(\nabla f(y_t)^\top p_t\right)^2+\frac{d-1}{q}\sum_{i=1}^q \left(\nabla f(y_t)^\top u_i\right)^2$; \label{lne:est-norm}
\State $x_{t+1}\leftarrow y_t - \frac{1}{\hat{L}} g_1(y_t)$, $m_{t+1}\leftarrow m_t - \frac{\theta_t}{\alpha_t} g_2(y_t)$;
\EndFor
\Return $x_T$.
\end{algorithmic}
\end{algorithm}

Specifically, we try to find an approximated solution of $\theta_t$ satisfying the equation $\theta_t=\frac{\E_t\left[\left(\nabla f(y_t)^\top v_t\right)^2\right]}{\hat{L}\cdot\E_t[\|g_2(y_t)\|^2]}$ to find a $\theta_t$ as large as possible and approximately satisfies the inequality $\theta_t\leq\frac{\E_t\left[\left(\nabla f(y_t)^\top v_t\right)^2\right]}{\hat{L}\cdot\E_t[\|g_2(y_t)\|^2]}$. Since $\frac{\E_t\left[\left(\nabla f(y_t)^\top v_t\right)^2\right]}{\hat{L}\cdot\E_t[\|g_2(y_t)\|^2]}=\frac{D_t+\frac{q}{d-1}(1-D_t)}{\hat{L}\left(D_t+\frac{d-1}{q}(1-D_t)\right)}$, we try to solve the equation
\begin{align}
    \theta_t=g(\theta_t):=\frac{D_t+\frac{q}{d-1}(1-D_t)}{\hat{L}\left(D_t+\frac{d-1}{q}(1-D_t)\right)},
\end{align}
where $D_t=(\overline{\nabla f(y_t)}^\top p_t)^2$ and $y_t$ depends on $\theta_t$. This corresponds to finding the fixed-point of $g$, so we apply the fixed-point iteration method. Specifically, we first let $\theta_t=0$, then $y_t=x_t$, and let $\theta_t\leftarrow g(\theta_t)$ (the above corresponding to Line~\ref{lne:approx-start} of Algorithm~\ref{alg:pars-impl}); then we calculate $y_t$ again using the new value of $\theta_t$ (corresponding to Line~\ref{lne:approx-middle}), and let $\theta_t\leftarrow g(\theta_t)$ (corresponding to Line~\ref{lne:approx-end}). We find that two iterations are able to lead to satisfactory performance. Note that then two additional queries to the directional derivative oracle are required to obtain $\nabla f(y_t^{(0)})^\top p_t$ and $\nabla f(y_t^{(1)})^\top p_t$ used in Line~\ref{lne:approx-start} and Line~\ref{lne:approx-end}.

Since $D_t=(\overline{\nabla f(y_t)}^\top p_t)^2=\frac{(\nabla f(y_t)^\top p_t)^2}{\|\nabla f(y_t)\|^2}$, we need to estimate $\|\nabla f(y_t)\|^2$ as introduced in Section~\ref{sec:estimation-Dt}. However, $y_t^{(0)}$ and $y_t^{(1)}$ in Algorithm~\ref{alg:pars-impl} are different from both $y_t$ and $y_{t-1}$, so to estimate $\|\nabla f(y_t^{(0)})\|^2$ and $\|\nabla f(y_t^{(1)})\|^2$ as in Section~\ref{sec:estimation-Dt}, many additional queries are required (since the query results of the directional derivative at $y_{t-1}$ or $y_t$ cannot be reused). Therefore, we introduce one additional approximation: we use the estimate of $\|\nabla f(y_{t-1})\|^2$ as the approximation of $\|\nabla f(y_t^{(0)})\|^2$ and $\|\nabla f(y_t^{(1)})\|^2$. Since the gradient norm itself is relatively large (compared with e.g. directional derivatives) and in zeroth-order optimization, the single-step update is relatively small, we expect that $\|\nabla f(y_t^{(0)})\|^2$ and $\|\nabla f(y_t^{(1)})\|^2$ are closed to $\|\nabla f(y_{t-1})\|^2$. In Algorithm~\ref{alg:pars-impl}, Line~\ref{lne:est-norm} estimates $\|\nabla f(y_t)\|^2$ by Eq.~\eqref{eq:est-norm}, and the estimator is denoted $\|\hat{\nabla}f_t\|^2$. Given this, $\|\nabla f(y_t^{(0)})\|^2$ and $\|\nabla f(y_t^{(1)})\|^2$ are approximated by $\|\hat{\nabla}f_{t-1}\|^2$ for calculations of $\hat{D}_t$ in Line~\ref{lne:approx-start} and Line~\ref{lne:approx-end} as approximations of $\big(\overline{\nabla f(y_t^{(0)})}^\top p_t\big)^2$ and $\big(\overline{\nabla f(y_t^{(1)})}^\top p_t\big)^2$.

Finally we note that in the experiments, we find that when using Algorithm~\ref{alg:pars-impl}, the error brought by approximation of $\|\nabla f(y_t^{(0)})\|^2$ and $\|\nabla f(y_t^{(1)})\|^2$ sometimes makes the performance of the algorithm not robust, especially when $q$ is small (e.g. $q=10$), which could lead the algorithm to divergence. Therefore, we propose two tricks to suppress the influence of approximation error (we note that in practice, the second trick is more important, while the first trick is often not necessary given the application of the second trick):
\begin{itemize}
    \item To reduce the variance of $\|\hat{\nabla} f_t\|$ when $q$ is small, we let
    \begin{align}
        \|\hat{\nabla} f_t^{\mathrm{avg}}\|^2=\frac{1}{k}\sum_{s=t-k+1}^t \|\hat{\nabla} f_s\|^2,
    \end{align}
    and use $\|\hat{\nabla} f_{t-1}^{\mathrm{avg}}\|^2$ to replace $\|\hat{\nabla} f_{t-1}\|^2$ in Line~\ref{lne:approx-start} and Line~\ref{lne:approx-end}. In our experiments we choose $k=10$. Compared with $\|\hat{\nabla} f_{t-1}\|^2$, using $\|\hat{\nabla} f_{t-1}^{\mathrm{avg}}\|^2$ to estimate $\|\nabla f(y_t^{(0)})\|^2$ and $\|\nabla f(y_t^{(1)})\|^2$ could reduce the variance at the cost of increased bias. We note that the increased bias sometimes brings problems, so one should be careful when applying this trick.
    
    \item Although $D_t\leq 1$, It is possible that $\hat{D}_t$ in Line~\ref{lne:approx-start} and Line~\ref{lne:approx-end} is larger than $1$, which could lead to a negative $\theta_t$. Therefore, a clipping of $\hat{D}_t$ is required. In our experiments, we observe that a $\hat{D}_t$ which is less than but very close to $1$ (when caused by the accidental large approximation error) could also lead to instability of optimization, perhaps because that it leads to a too large value of $\theta_t$ used to determine $y_t$ and to update $m_t$. Therefore, we let $\hat{D}_t\leftarrow\min\{\hat{D}_t, B_{\mathrm{ub}}\}$  in Line~\ref{lne:approx-start} and Line~\ref{lne:approx-end} before calculating $\theta_t$, where $0< B_{\mathrm{ub}}\leq 1$ is fixed. In our experiments we set $B_{\mathrm{ub}}$ to $0.6$.
\end{itemize}

We leave a more systematic study of the approximation error as future work.

\subsection{Implementation of History-PARS in practice (History-PARS-Impl)}
\label{sec:C5}
In PARS, when using a specific prior instead of the prior from a general source, we can utilize some properties of the prior. When using the historical prior ($p_t=v_{t-1}$), we find that $D_t$ is usually similar to $D_{t-1}$, and intuitively it happens when the smoothness of the objective function does not change quickly along the optimization trajectory. Therefore, the best value of $\theta_t$ should also be similar to the best value of $\theta_{t-1}$. Based on this observation, we can directly use $\theta_{t-1}$ as the value of $\theta_t$ in iteration $t$, and the value of $\theta_{t-1}$ is obtained with $y_{t-1}$ in iteration $t-1$. Following this thread, we present our implementation of History-PARS, i.e. History-PARS-Impl, in Algorithm~\ref{alg:hist-pars-impl}.

\begin{algorithm}[!htbp]
\small
\caption{History-PARS in implementation (History-PARS-Impl)}
\label{alg:hist-pars-impl}
\begin{algorithmic}[1]
\Require $L$-smooth convex function $f$; initialization $x_0$; $\hat{L} \geq L$; Query count per iteration $q$ (cannot be too small); iteration number $T$; $\gamma_0>0$.
\Ensure $x_T$ as the approximate minimizer of $f$.
\State $m_0\leftarrow x_0$;
\State $\theta_{-1}\leftarrow$ a very small positive number close to $0$;
\State $v_{-1}\sim\U(\Sp_{d-1})$;
\For {$t = 0$ to $T-1$}
\State $y_t\leftarrow(1-\alpha_t)x_t+\alpha_t m_t$, where $\alpha_t\geq0$ is a positive root of the equation $\alpha_t^2=\theta_{t-1} (1-\alpha_t)\gamma_t$; $\gamma_{t+1}\leftarrow (1-\alpha_t)\gamma_t$;
\State Sample an orthonormal set $\{u_i\}_{i=1}^q$ in the subspace perpendicular to $v_{t-1}$, as in Section~\ref{sec:construction-prgf} with $p_t=v_{t-1}$;
\State $g_1(y_t)\leftarrow \sum_{i=1}^q \nabla f(y_t)^\top u_i \cdot u_i+\nabla f(y_t)^\top v_{t-1}\cdot v_{t-1}$; $v_t\leftarrow \overline{g_1(y_t)}$;
\State $g_2(y_t)\leftarrow\frac{d-1}{q}\sum_{i=1}^q \nabla f(y_t)^\top u_i \cdot u_i+\nabla f(y_t)^\top v_{t-1}\cdot v_{t-1}$;
\State $\theta_t\leftarrow\frac{D_t+\frac{q}{d-1}(1-D_t)}{\hat{L}\left(D_t+\frac{d-1}{q}(1-D_t)\right)}$, where $D_t$ is estimated using Eq.~\eqref{eq:est-Dt} with $p_t=v_{t-1}$;
\State $x_{t+1}\leftarrow y_t - \frac{1}{\hat{L}} g_1(y_t)$, $m_{t+1}\leftarrow m_t - \frac{\theta_{t-1}}{\alpha_t} g_2(y_t)$;
\EndFor
\Return $x_T$.
\end{algorithmic}
\end{algorithm}
\subsection{Full version of Algorithm~\ref{alg:nag-extended} considering the strong convexity parameter and its convergence theorem}
\label{sec:C6}
\label{sec:ars-strong}
\begin{algorithm}[!t]
\small
\caption{Extended accelerated random search framework for $\tau\geq 0$}
\label{alg:nag-strong}
\begin{algorithmic}[1]
\Require $L$-smooth and $\tau$-strongly convex function $f$; initialization $x_0$; $\hat{L} \geq L$; $\hat{\tau}$ such that $0\leq\hat{\tau} \leq \tau$; iteration number $T$; a positive $\gamma_0\geq \hat{\tau}$.
\Ensure $x_T$ as the approximate minimizer of $f$.
\State $m_0\leftarrow x_0$;
\For {$t = 0$ to $T-1$}
\State Find a $\theta_t>0$ such that $\theta_t\leq \frac{\E_t\left[\left(\nabla f(y_t)^\top v_t\right)^2\right]}{\hat{L}\cdot\E_t[\|g_2(y_t)\|^2]}$ where $\theta_t$, $y_t$ and $g_2(y_t)$ are defined in following steps:
\State Step 1: $y_t\leftarrow(1-\beta_t)x_t+\beta_t m_t$, where $\beta_t := \frac{\alpha_t\gamma_t}{\gamma_t+\alpha_t\hat{\tau}}$, $\alpha_t$ is a positive root of the equation $\alpha_t^2=\theta_t ((1-\alpha_t)\gamma_t+\alpha_t\hat{\tau})$; $\gamma_{t+1}\leftarrow (1-\alpha_t)\gamma_t+\alpha_t\hat{\tau}$;
\State Step 2: Let $v_t$ be a random vector s.t. $\|v_t\|=1$; $g_1(y_t)\leftarrow \nabla f(y_t)^\top v_t \cdot v_t$;
\State Step 3: Let $g_2(y_t)$ be an unbiased estimator of $\nabla f(y_t)$, i.e. $\E_t[g_2(y_t)]=\nabla f(y_t)$;
\State $\lambda_t\leftarrow\frac{\alpha_t}{\gamma_{t+1}}\hat{\tau}$;
\State $x_{t+1}\leftarrow y_t - \frac{1}{\hat{L}} g_1(y_t)$, $m_{t+1}\leftarrow (1-\lambda_t)m_t+\lambda_t y_t - \frac{\theta_t}{\alpha_t} g_2(y_t)$;
\EndFor
\Return $x_T$.
\end{algorithmic}
\end{algorithm}

In fact, the ARS algorithm proposed in \cite{nesterov2017random} requires knowledge of the strong convexity parameter $\tau$ of the objective function, and the original algorithm depends on $\tau$. The ARS algorithm has a convergence rate for general smooth convex functions, and also have another potentially better convergence rate if $\tau>0$. In previous sections, for simplicity, we suppose $\tau=0$ and illustrate the corresponding extension in Algorithm~\ref{alg:nag-extended}. In fact, for the general case $\tau\geq 0$, the original ARS can also be extended to allow for incorporation of prior information. We present the extension to ARS with $\tau\geq 0$ in Algorithm~\ref{alg:nag-strong}. Note that our modification is similar to that in Algorithm~\ref{alg:nag-extended}. For Algorithm~\ref{alg:nag-strong}, we can also provide its convergence guarantee as shown in Theorem~\ref{thm:nag-strong}. Note that after considering the strong convexity parameter in the algorithm, we have an additional convergence guarantee, i.e. Eq.~\eqref{eq:theorem-nag-strong-new}. In the corresponding PARS algorithm, we have $\theta_t\geq \frac{q^2}{\hat{L}d^2}$, so the convergence rate of PARS is not worse than that of ARS and admits improvement given a good prior.

\begin{theorem}
\label{thm:nag-strong}
In Algorithm~\ref{alg:nag-strong}, if $\theta_t$ is $\mathcal{F}_{t-1}$-measurable, we have
\begin{align}
\label{eq:theorem-nag-strong}
    \E\left[(f(x_T) - f(x^*))\left(1+\frac{\sqrt{\gamma_0}}{2}\sum_{t=0}^{T-1}\sqrt{\theta_t}\right)^2\right]\leq f(x_0)-f(x^*)+\frac{\gamma_0}{2}\|x_0-x^*\|^2.
\end{align}
and
\begin{align}
\label{eq:theorem-nag-strong-new}
    \E\left[(f(x_T) - f(x^*))\exp\left(\sqrt{\hat{\tau}}\sum_{t=0}^{T-1}\sqrt{\theta_t}\right)\right]\leq f(x_0)-f(x^*)+\frac{\gamma_0}{2}\|x_0-x^*\|^2.
\end{align}
\end{theorem}
\begin{proof}
Let $L_e:=\frac{\hat{L}}{2\hat{L}-L}\cdot\hat{L}$. We still have Eq.~\eqref{eq:value_diff}, so
\begin{align}
    \E_t[f(x_{t+1})]&\leq f(y_t) - \frac{\E_t\left[\left(\nabla f(y_t)^\top v_t\right)^2\right]}{2L_e} \\
    &\leq f(y_t) - \frac{\E_t\left[\left(\nabla f(y_t)^\top v_t\right)^2\right]}{2\hat{L}}.
\end{align}
For an arbitrary fixed $x$, define $\rho_t(x):=\frac{\gamma_t}{2}\|m_t-x\|^2 + f(x_t)-f(x)$. Let $r_t:=(1-\lambda_t)m_t+\lambda_t y_t$. We first prove a lemma.

Since $(1-\beta_t)x_t+\beta_t m_t=y_t=(1-\beta_t)y_t+\beta_t y_t$, we have $m_t-y_t=\frac{1-\beta_t}{\beta_t}(y_t-x_t)$. So
\begin{align}
    r_t&=(1-\lambda_t) m_t+\lambda_t y_t=y_t+(1-\lambda_t)(m_t-y_t)=y_t+(1-\lambda_t)\frac{1-\beta_t}{\beta_t}(y_t-x_t).
\end{align}
By $\beta_t=\frac{\alpha_t\gamma_t}{\gamma_t+\alpha_t\hat{\tau}}$, $\gamma_{t+1}=(1-\alpha_t)\gamma_t+\alpha_t\hat{\tau}$ and $\lambda_t=\frac{\alpha_t}{\gamma_{t+1}}\hat{\tau}$, after eliminating $\gamma_t$ and $\gamma_{t+1}$, we have $(1-\lambda_t)\frac{1-\beta_t}{\beta_t}=\frac{1-\alpha_t}{\alpha_t}$.
Hence $r_t=y_t+\frac{1-\alpha_t}{\alpha_t}(y_t-x_t)$, which means \begin{align}
    y_t=(1-\alpha_t)x_t+\alpha_t r_t.
\end{align}
Now we start the main proof.
\begin{align}
    \rho_{t+1}(x) &= \frac{\gamma_{t+1}}{2}\|m_{t+1}-x\|^2 + f(x_{t+1}) - f(x) \\
    &= \frac{\gamma_{t+1}}{2}\|r_t-x\|^2 - \frac{\gamma_{t+1}\theta_t}{\alpha_t}g_2(y_t)^\top (r_t-x) + \frac{\gamma_{t+1}\theta_t^2}{2\alpha_t^2}\|g_2(y_t)\|^2 + f(x_{t+1}) - f(x) \\
    &= \frac{\gamma_{t+1}}{2}\|r_t-x\|^2 - \alpha_t g_2(y_t)^\top (r_t-x) + \frac{\theta_t}{2}\|g_2(y_t)\|^2 + f(x_{t+1}) - f(x).
\end{align}

We make sure in Remark~\ref{rem:dependence-history} that $\mathcal{F}_{t-1}$ always includes all the randomness before iteration $t$. Therefore, $\gamma_t$, $m_t$ and $x_t$ are $\mathcal{F}_{t-1}$-measurable. The assumption in Theorem~\ref{thm_main:nag} requires that $\theta_t$ is $\mathcal{F}_{t-1}$-measurable. Since $\alpha_t$, $\beta_t$, $y_t$ and $r_t$ are determined by $\gamma_t$, $x_t$, $m_t$ and $\theta_t$ (through Borel functions), they are also $\mathcal{F}_{t-1}$-measurable. Since $\theta_t$, $\alpha_t$ and $r_t$ are $\mathcal{F}_{t-1}$-measurable, we have $\E_t[\alpha_t g_2(y_t)^\top (r_t-x)]=\alpha_t\E_t[g_2(y_t)]^\top (r_t-x)$ and $\E_t[\theta_t \|g_2(y_t)\|^2]=\theta_t\E_t[\|g_2(y_t)\|^2]$. Hence
\begin{align}
    \E_t[\rho_{t+1}(x)] &= \frac{\gamma_{t+1}}{2}\|r_t-x\|^2 - \alpha_t \E_t[g_2(y_t)]^\top (r_t-x) + \frac{\theta_t}{2}\E_t[\|g_2(y_t)\|^2] + \E_t[f(x_{t+1})] - f(x) \\
    &= \frac{\gamma_{t+1}}{2}\|r_t-x\|^2 - \alpha_t \nabla f(y_t)^\top (r_t-x) + \frac{\theta_t}{2}\E_t[\|g_2(y_t)\|^2] + \E_t[f(x_{t+1})] - f(x) \\
    &\leq \frac{\gamma_{t+1}}{2}\|r_t-x\|^2 - \alpha_t \nabla f(y_t)^\top (r_t-x) + \frac{\E_t\left[\left(\nabla f(y_t)^\top v_t\right)^2\right]}{2\hat{L}} + \E_t[f(x_{t+1})] - f(x) \\
    &\leq \frac{\gamma_{t+1}}{2}\|r_t-x\|^2 - \alpha_t \nabla f(y_t)^\top (r_t-x) + f(y_t) - f(x) \\
    &= \frac{\gamma_{t+1}}{2}\|r_t-x\|^2 - \nabla f(y_t)^\top (\alpha_t r_t-\alpha_t x) + f(y_t) - f(x) \\
    &= \frac{\gamma_{t+1}}{2}\|r_t-x\|^2 + \nabla f(y_t)^\top (-y_t+(1-\alpha_t)x_t+\alpha_t x) + f(y_t) - f(x) \\
    &= \frac{\gamma_{t+1}}{2}\|r_t-x\|^2 + \alpha_t\left(f(y_t)+\nabla f(y_t)^\top (x-y_t)\right) \\
    &\quad+ (1-\alpha_t)\left(f(y_t)+\nabla f(y_t)^\top (x_t-y_t)\right) - f(x) \\
    &\leq \frac{\gamma_{t+1}}{2}\|r_t-x\|^2 + (1-\alpha_t)f(x_t)-(1-\alpha_t) f(x)- \frac{\alpha_t\tau}{2}\|x-y_t\|^2.
\end{align}

We also have
\begin{align}
    \frac{\gamma_{t+1}}{2}\|r_t-x\|^2 &= \frac{\gamma_{t+1}}{2}\|(1-\lambda_t)m_t+\lambda_t y_t-x\|^2 \\
    &= \frac{\gamma_{t+1}}{2}\|(1-\lambda_t)(m_t-x)+\lambda_t(y_t-x)\|^2 \\
    &\leq \frac{\gamma_{t+1}(1-\lambda_t)}{2}\|m_t-x\|^2+\frac{\gamma_{t+1}\lambda_t}{2}\|y_t-x\|^2 \\
    &= \frac{\gamma_{t+1}(1-\lambda_t)}{2}\|m_t-x\|^2+\frac{\alpha_t\hat{\tau}}{2}\|y_t-x\|^2 \\
    &= (1-\alpha_t)\frac{\gamma_t}{2}\|m_t-x\|^2+\frac{\alpha_t\hat{\tau}}{2}\|x-y_t\|^2,
\end{align}
where the inequality is due to Jensen's inequality applied to the convex function $\|\cdot\|^2$, and the third equality is obtained after substituting $\lambda_t\gamma_{t+1}=\alpha_t\hat{\tau}$ by the definition of $\lambda_t$. Since $\gamma_{t+1}=(1-\alpha_t)\gamma_t+\alpha_t\hat{\tau}=(1-\alpha_t)\gamma_t+\lambda_t\gamma_{t+1}$, we have $\gamma_{t+1}(1-\lambda_t)=(1-\alpha_t)\gamma_t$, which leads to the last equality.

Hence
\begin{align}
    \E_t[\rho_{t+1}(x)] &\leq \frac{\gamma_{t+1}}{2}\|r_t-x\|^2 + (1-\alpha_t)f(x_t)-(1-\alpha_t) f(x)- \frac{\alpha_t\tau}{2}\|x-y_t\|^2 \\
    &= (1-\alpha_t)\rho_t(x)+\frac{\alpha_t(\hat{\tau}-\tau)}{2}\|x-y_t\|^2 \\
    &\leq (1-\alpha_t)\rho_t(x).
\end{align}

Therefore, 
\begin{align*}
    \rho_0(x)&=\E[\rho_0(x)]\geq \E\left[\frac{\E_0[\rho_1(x)]}{1-\alpha_0}\right]=\E\left[\E_0\left[\frac{\rho_1(x)}{1-\alpha_0}\right]\right]=\E\left[\frac{\rho_1(x)}{1-\alpha_0}\right] \\
    &\geq \E\left[\frac{\E_1[\rho_2(x)]}{(1-\alpha_0)(1-\alpha_1)}\right]=\E\left[\E_1\left[\frac{\rho_2(x)}{(1-\alpha_0)(1-\alpha_1)}\right]\right]=\E\left[\frac{\rho_2(x)}{(1-\alpha_0)(1-\alpha_1)}\right] \\
    &\geq \ldots \\
    &\geq \E\left[\frac{\rho_T(x)}{\prod_{t=0}^{T-1}(1-\alpha_t)}\right].
\end{align*}

We have $\rho_T(x)\geq f(x_T)-f(x)$. To prove the theorem, let $x=x^*$. The remaining is to give an upper bound of $\prod_{t=0}^{T-1}(1-\alpha_t)$. Let $\psi_k:=\prod_{t=0}^{k-1}(1-\alpha_t)$ and $a_k:=\frac{1}{\sqrt{\psi_k}}$, we have
\begin{align}
    a_{k+1}-a_k &= \frac{1}{\sqrt{\psi_{k+1}}}-\frac{1}{\sqrt{\psi_k}} = \frac{\sqrt{\psi_k}-\sqrt{\psi_{k+1}}}{\sqrt{\psi_k\psi_{k+1}}} = \frac{\psi_k-\psi_{k+1}}{\sqrt{\psi_k\psi_{k+1}}(\sqrt{\psi_k}+\sqrt{\psi_{k+1}})} \\
    &\geq \frac{\psi_k-\psi_{k+1}}{\sqrt{\psi_k\psi_{k+1}}(\sqrt{\psi_k})} \\
    &= \frac{\psi_k-(1-\alpha_k)\psi_k}{2\psi_k\sqrt{\psi_{k+1}}}=\frac{\alpha_k}{2\sqrt{\psi_{k+1}}}=\frac{\sqrt{\gamma_{k+1}\theta_k}}{2\sqrt{\psi_{k+1}}}=\frac{\sqrt{\theta_k}}{2}\sqrt{\frac{\gamma_{k+1}}{\psi_{k+1}}} \\
    &\geq \frac{\sqrt{\gamma_0\theta_k}}{2}.
\end{align}
The last step is because $\gamma_{t+1}\geq (1-\alpha_t)\gamma_t$, so $\frac{\gamma_{k+1}}{\gamma_0}\geq \prod_{t=0}^k (1-\alpha_t)=\psi_{k+1}$. Since $\psi_0=1$, $a_0=1$. Hence $a_T\geq 1+\frac{\sqrt{\gamma_0}}{2}\sum_{t=0}^{T-1}\sqrt{\theta_t}$. Therefore,
\begin{align}
    \psi_T\leq \frac{1}{\left(1+\frac{\sqrt{\gamma_0}}{2}\sum_{t=0}^{T-1}\sqrt{\theta_t}\right)^2}.
\end{align}

Meanwhile, since $\gamma_0\geq \hat{\tau}$ and $\gamma_{t+1}=(1-\alpha_t)\gamma_t+\alpha_t \hat{\tau}$, we have that $\forall t, \gamma_t\geq \hat{\tau}$. Then $\alpha_t^2=\theta_t((1-\alpha_t)\gamma_t+\alpha_t\hat{\tau})\geq \theta_t\hat{\tau}$, then we have that $\alpha_t\geq \sqrt{\hat{\tau}\theta_t}$. Therefore,
\begin{align}
    \psi_T\leq \prod_{t=0}^{T-1} \left(1-\sqrt{\hat{\tau}\theta_t}\right)\leq \exp\left(-\sqrt{\hat{\tau}}\sum_{t=0}^{T-1}\sqrt{\theta_t}\right).
\end{align}

The proof is completed.
\end{proof}

\section{Supplemental materials for Section~\ref{sec:5}}
\label{sec:D}
\subsection{More experimental settings in Section~\ref{sec:5.1}}
\label{sec:D1}
In experiments in this section, we set the step size $\mu$ used in finite differences (Eq.~\eqref{eq:forward-difference}) to $10^{-6}$, and the parameter $\gamma_0$ in ARS-based methods to $\hat{L}$.
\subsubsection{Experimental settings for Fig.~\ref{fig:biased}}
\label{sec:D11}
\paragraph{Prior}
We adopt the setting in Section~4.1 of \cite{maheswaranathan2019guided} to mimic the case that the prior is a biased version of the true gradient. Specifically, we let $p_t=\overline{\overline{\nabla f(x_t)}+(b+n_t)}$, where $b$ is a fixed vector and $n_t$ is a random vector uniformly sampled each iteration, $\|b\|=1$ and $\|n_t\|=1.5$.

\paragraph{Test functions}
Our test functions are as follows. We choose $f_1$ as the ``worst-case smooth convex function'' used to construct the lower bound complexity of first-order optimization, as used in \cite{nesterov2017random}:
\begin{align}
    f_1(x)= \frac{1}{2}(x^{(1)})^2+\frac{1}{2}\sum_{i=1}^{d-1}(x^{(i+1)}-x^{(i)})^2+\frac{1}{2}(x^{(d)})^2-x^{(1)},  \text{ where }x_0=\mathbf{0}.
\end{align}
We choose $f_2$ as a simple smooth and strongly convex function with a worst-case initialization:
\begin{align}
    f_2(x)=\sum_{i=1}^d \left(\frac{i}{d}\cdot (x^{(i)})^2\right),\text{ where }x_0^{(1)}=d, x_0^{(i)}=0 \text{ for }i\geq 2.
\end{align}
We choose $f_3$ as the Rosenbrock function ($f_8$ in \cite{hansen2009real}) which is a well-known non-convex function used to test the performance of optimization problems:
\begin{align}
    f_3(x)=\sum_{i=1}^{d-1} \left(100\left((x^{(i)})^2-x^{(i+1)}\right)^2+(x^{(i)}-1)^2\right),  \text{ where }x_0=\mathbf{0}.
\end{align}
We note that ARS, PARS-Naive and PARS could depend on a strong convexity parameter (see Section~\ref{sec:ars-strong}) when applied to a strongly convex function. Therefore, for $f_2$ we set this parameter to the ground truth value. For $f_1$ and $f_3$ we set it to zero, i.e. we use Algorithm~\ref{alg:nag-extended}.
\subsubsection{Experimental settings for Fig.~\ref{fig:history}}
In this part we set $d=500$ and set $q$ such that each iteration of each algorithm costs $11$ queries. Since when using the historical prior, we aim to build algorithms agnostic to parameters of the objective function, we set the strong convexity parameter in ARS-based methods to $0$ even though we know that e.g. $f_2$ is strongly convex. Correspondingly, we adopt adaptive restart of function scheme \cite{o2015adaptive} to reach the ideal performance. We introduce our implementation here. In each iteration (suppose that currently it is iteration $t$) of Algorithm~\ref{alg:hist-pars-impl}, we check whether $f(y_t)\leq f(y_{t-1})$. If not, we set $m_{t+1}\leftarrow x_{t+1}$ and $\gamma_{t+1}\leftarrow \gamma_0$ as the restart.

\subsection{More experimental settings in Section~\ref{sec:5.2}}
\label{sec:D2}
We perform targeted attacks under the $\ell_2$ norm with the perturbation bound set to $3.514$ ($=32/255\times\sqrt{784}$) if each pixel value has the range $[0,1]$. The objective function to maximize for attacking image $x$ is the C\&W loss \cite{carlini2017towards}, i.e. $f(x)=Z(x)_t-\max_{i\neq t}Z(x)_i$, where $t$ is the target class and $Z(x)$ is the logits given the input $x$. The network architecture is from the PyTorch example (\url{https://github.com/pytorch/examples/tree/master/mnist}).

We set the step size $\mu$ used in finite differences (Eq.~\eqref{eq:forward-difference}) to $10^{-4}$, and the parameter $\gamma_0$ in ARS-based methods to $\hat{L}$. To deal with the constraints in optimization, in each iteration we perform projection after the update to ensure that the constraints are satisfied. In historical-prior-guided methods, to prevent the prior from pointing to the infeasible region (where the constraints are not satisfied), we let the prior $p_t$ be $\overline{x_t-x_{t-1}}$ for History-PRGF and $\overline{x_t-y_{t-1}}$ for History-PARS. In unconstrained optimization, this is equivalent to the original choice of $p_t$ ($p_t=\overline{g_{t-1}}$ for History-PRGF and $p_t=\overline{g_1(y_{t-1})}$ for History-PARS) up to sign. But in constrained optimization, since $x_t$ is further projected to the feasible region after the update from $x_{t-1}$ or $y_{t-1}$, they are not equivalent.

We note that the number of queries for each image does not count queries (one query per iteration) to check whether the attack has succeeded.

\section{Potential negative societal impacts}

As a theoretical work, we think this paper can provide valuable insights on understanding existing algorithms and may inspire new algorithms for zeroth-order optimization, while having no significant potential negative societal impacts. One may pay attention to its application to query-based black-box adversarial attacks.

\end{document}